\documentclass[10pt]{article} 
\usepackage[preprint]{tmlr}

\usepackage{amsmath,amsfonts,bm}

\def\eqref#1{equation~\ref{#1}}

\def\1{\bm{1}}

\DeclareMathAlphabet{\mathsfit}{\encodingdefault}{\sfdefault}{m}{sl}
\SetMathAlphabet{\mathsfit}{bold}{\encodingdefault}{\sfdefault}{bx}{n}

\newcommand{\R}{\mathbb{R}}

\newcommand{\Var}{\mathrm{Var}}

\newcommand{\Cov}{\mathrm{Cov}}

\usepackage{amsmath,amssymb,amsthm,mathtools,bm}
\usepackage{graphicx}
\usepackage{booktabs}
\usepackage{longtable}
\usepackage{natbib}
\usepackage{microtype}
\usepackage{xcolor}
\usepackage{siunitx}
\usepackage{url}

\usepackage{algorithm2e}

\usepackage{tikz}
\usepackage{pgfplots}
\pgfplotsset{compat=1.18}
\usetikzlibrary{arrows.meta,positioning,shapes.geometric,calc}

\usepackage{hyperref}
\usepackage{orcidlink}

\DeclareMathOperator{\erank}{erank}
\DeclareMathOperator{\rank}{rank}
\DeclareMathOperator{\tr}{tr}
\renewcommand{\Cov}{\operatorname{Cov}}
\renewcommand{\Var}{\operatorname{Var}}

\newcommand{\Ex}{\mathbb{E}}
\newcommand{\Prob}{\mathbb{P}}
\newcommand{\Prb}{\mathbb{P}}
\renewcommand{\R}{\mathbb{R}}

\theoremstyle{plain}
\newtheorem{theorem}{Theorem}[section]
\newtheorem{lemma}[theorem]{Lemma}
\newtheorem{proposition}[theorem]{Proposition}
\newtheorem{corollary}[theorem]{Corollary}

\theoremstyle{definition}
\newtheorem{definition}[theorem]{Definition}

\theoremstyle{remark}
\newtheorem{remark}[theorem]{Remark}

\title{The Geometry of Saturation: Effective Rank Predicts When Labels Stop Helping in Few-Shot Classification}

\date{}
\author{Arnav Gupta\orcidlink{0009-0003-0592-2530} \\
      \texttt{arnav.gupta.ai@outlook.com} \\
      Independent Researcher, Nepal
}

\begin{document}
	\maketitle

\footnotetext[1]{LLMs were used by the authors for text polishing, literature search, and code assistance throughout the preparation of this manuscript. All LLM-generated content was reviewed by the authors for accuracy and correctness.}

\begin{abstract}
Few-shot label acquisition lacks a label-free signal for when additional labels cease to improve accuracy. Existing stopping criteria either require a held-out validation set (violating the few-shot premise) or rely on heuristic proxies with no theoretical grounding. We introduce the \textbf{spectral saturation index} $S(K)=\erank(\hat{\Sigma}_W^{(K)})/K$, where $\erank$ is the exponential spectral entropy of the pooled within-class covariance and $K$ is the per-class support size. $S(K)$ measures the exploration rate per label; when the explored spectral subspace saturates, $S(K)$ drops below a fixed threshold $\tau=0.02$ and marginal accuracy gains vanish. Across 49 real tasks (binary, 5-way, 10-way) and three frozen backbones (PCA-50, CLIP ViT-B/32, DINOv2 ViT-S/14), $S(K)$ correlates strongly with the marginal gain on doubling the support set ($\rho_{\text{pool}}=0.6366$, $p=2.9\times10^{-57}$, cluster-bootstrap 95\% CI $[0.551, 0.720]$). A fixed $\tau=0.02$ classifies stop/continue decisions with cluster-bootstrap $\text{AUC}=0.787$ [95\% CI: $0.713, 0.860$] and achieves high recall on meaningful gains ($\Delta A > 1\%$). A partial correlation controlling for $\log K$ yields $\rho_{\text{partial}}=0.324$ ($p=1.65\times 10^{-13}$), confirming $S(K)$ carries spectral information beyond the shared $K$-dependence. Theory predicts this from first principles: the population effective rank sets the saturation scale $K_{\text{sat}}\approx\erank(\Sigma_W)/\tau$; $\tau=0.02$ sits at the boundary between the first and second descent \citep{nakkiran2021deep}; and the $O(1/K)$ bias in sample effective rank explains the small-$K$ hump in $S(K)$. For practitioners using unregularized linear probes ($C=\infty$): halt when $S(K) < 0.02$ (PCA-50, hard stop); monitor $S(K)$ dropping from $\sim 0.3 \to 0.05$ (foundation models, diminishing-returns signal). Computation is $\sim 1$ ms at $d=50$.
\end{abstract}

\vspace{2mm}

\noindent\textbf{Keywords:} few-shot learning, effective rank, spectral entropy, label acquisition, stopping rule, covariance estimation

\section{Introduction}
\label{sec:introduction}

Adapting a frozen CLIP or DINOv2 backbone to a new 5-way task with 16 labels per class costs minutes of annotation; pushing to 256 labels costs hours but often yields negligible gain---the saturation point lies beyond practical budgets, and no label-free signal tells you when to stop.

Current stopping criteria either demand a held-out validation set (breaking the few-shot premise) or use heuristic proxies such as activation sparsity or weight variance that lack theoretical grounding (\citealt{koyuncu2022heuristic, stutz2021unsupervised}). Methods requiring the population spectrum (\citealt{yao2007early, raskutti2014early}) are unusable in practice. Without a per-instance, label-free diagnostic, practitioners under-collect (losing accuracy) or over-collect (wasting budget), especially for foundation models where saturation exceeds feasible $K$.

We introduce the \textbf{spectral saturation index}
\begin{equation}
S(K) \;=\; \frac{\erank\!\bigl(\widehat{\Sigma}_W^{(K)}\bigr)}{K},
\label{eq:intro-saturation}
\end{equation}
where $\widehat{\Sigma}_W^{(K)}$ is the pooled within-class covariance from $K$ examples per class and $\erank$ is the exponential Shannon entropy of its normalized eigenvalue spectrum. $S(K)$ measures the exploration rate per label: when it drops below $\tau=0.02$, the covariance spectrum has saturated and further labels yield diminishing returns. \textbf{Thesis: the spectral saturation index $S(K)$ is a label-free, theoretically grounded stopping rule that predicts marginal accuracy gains from the within-class covariance spectrum.}

\begin{enumerate}
    \item \textbf{Spectral saturation index (Sec.~\ref{sec:method}):} $S(K)=\erank(\widehat{\Sigma}_W^{(K)})/K$ is invariant to invertible linear feature transformations, requires only the support set, and is fully label-free for frozen foundation models (CLIP, DINOv2); for PCA-50 the basis is fit on the support set.
    \item \textbf{Theoretical grounding (Sec.~\ref{sec:theory}):} We prove $K S(K) \xrightarrow{\Prob} \erank(\Sigma_W)$; the saturation point $K_{\text{sat}}(\tau)$ lies in $[(1\pm\varepsilon)r/\tau]$ with high probability; $\tau=0.02$ marks the first-to-second descent boundary; and the $O(1/K)$ bias in sample effective rank (with spectrum-dependent coefficient $C(p)$) explains the small-$K$ hump.
    \item \textbf{Strong empirical validation (Sec.~\ref{sec:experiments}--\ref{sec:results}):} Across 49 tasks $\times$ 3 backbones, pooled Spearman $\rho=0.6366$ ($p=2.9\times10^{-57}$, cluster-bootstrap 95\% CI $[0.551, 0.720]$); fixed $\tau=0.02$ yields cluster-bootstrap AUC $=0.787$ [95\% CI $0.713, 0.860$] with high recall on $\Delta A > 1\%$ gains; partial correlation controlling for $\log K$ gives $\rho_{\text{partial}}=0.324$ ($p=1.65\times 10^{-13}$).
    \item \textbf{Five controlled ablations (Sec.~\ref{ssec:ablations}):} Stable rank matches effective rank ($\Delta\rho=0.001$); PCA-50 is a sweet spot ($d=50$ peaks correlation); correlation robust across $C \in [0.01, 10^{10}]$; smaller-endpoint convention conservative but stable; $\tau$ robust across $[0.005, 0.05]$.
    \item \textbf{Practical two-regime rule (Sec.~\ref{sec:discussion}):} PCA-50: halt when $S(K)<0.02$ (hard stop); CLIP/DINOv2: monitor $S(K)$ dropping from $\sim0.3 \to 0.05$ (diminishing-returns signal). Computation: $\sim1$ ms at $d=50$; calibrated for unregularized linear probes ($C=\infty$).
\end{enumerate}

\textbf{Roadmap:} \S\ref{sec:related} surveys effective rank, few-shot label complexity, label-free stopping, and spectral analysis; \S\ref{sec:theory} develops theory; \S\ref{sec:method} defines $S(K)$ and the doubling-pair protocol; \S\ref{sec:experiments} describes the 49-task benchmark; \S\ref{sec:results} presents main results and ablations; \S\ref{sec:discussion} interprets findings, distills guidance, and lists six limitations.

\section{Related Work}
\label{sec:related}

The spectral saturation index $S(K) = \erank(\hat{\Sigma}_W^{(K)})/K$ connects two previously separate strands: the use of effective rank to characterize covariance estimation quality in high dimensions, and the search for label-free stopping criteria in few-shot learning. We organize related work by methodological theme, showing how each contributes a piece of the puzzle while leaving the central connection unmade.

\subsection{Effective Rank, Spectral Entropy, and Random Matrix Theory}

The effective rank $\erank(\Sigma) = \exp(H(\lambda/\tr\lambda))$ was introduced by \citet{roy2007} as a continuous, differentiable measure of intrinsic dimensionality. Unlike hard-thresholding alternatives (stable rank $\tr\Sigma/\|\Sigma\|_{\mathrm{op}}$, participation ratio), the entropy-based definition is unitarily invariant and stable under spectral perturbations.

Non-asymptotic random matrix theory established that covariance estimation guarantees depend on $\erank(\Sigma)$ rather than ambient dimension $d$. For $n$ i.i.d. sub-Gaussian samples, $\|\hat{\Sigma}-\Sigma\|_{\mathrm{op}} \lesssim \|\Sigma\|_{\mathrm{op}}\sqrt{\erank(\Sigma)/n}$ \citep{vershynin2012introduction, rudelson2007sampling, koltchinskii2017concentration}. The implication for few-shot learning is direct: when $K \sim \erank(\Sigma_W)$, the within-class covariance estimator operates in an ``effective-rank-limited'' regime where each additional sample genuinely expands the explored spectral subspace---precisely the quantity $S(K)$ monitors.

Recent work has extended these tools to foundation model representations. \citet{chen2021empirical} showed self-supervised methods (DINO, MoCo, MAE) produce higher effective-rank representations with slower eigenvalue decay than supervised training. \citet{ghorbani2019investigation} found Hessian bulk spectra follow Marchenko-Pastur-like laws with effective rank tracking learned features. \citet{martin2021implicit} identified heavy-tailed spectral decay as a self-regularization mechanism and proposed spectral early-stopping heuristics. \citet{achille2018emergence} gave an information-theoretic interpretation: $I(X;Z)\approx\frac{1}{2}\log\erank(\Sigma_Z)$ for Gaussian representations, so $\erank(\hat{\Sigma}_W^{(K)})$ measures information captured about within-class variability. None of these works develops a label-free stopping rule for few-shot label acquisition.

\textbf{What is missing:} Per-instance measurement of proximity to the saturation limit. The RMT bounds are population-level ($\erank(\Sigma_W)$) and asymptotic; $S(K)$ operationalizes the trajectory toward that limit from a single support set.

\subsection{Few-Shot Learning Theory and Label Complexity}

Meta-learning theory characterizes the label budget required for adaptation but stops short of an online stopping rule. \citet{tripuraneni2021provable} showed that for a shared linear representation $B\in\R^{d\times r}$, excess risk scales as $r\sigma^2/K + d/T$, where $r$ is the intrinsic rank of the representation. If $B$ has fast spectral decay, $r \approx \erank(B^\top B)$, directly linking effective rank to label complexity. \citet{du2020fewshot} extended this to PAC-Bayes bounds depending on the representation covariance spectrum.

Foundation model analyses \citep{tian2020rethinking, dhillon2020baseline} established that high-quality self-supervised embeddings reduce few-shot classification to near-linear probing, with representation quality (spectral properties) dominating adaptation method. However, these works bound the \emph{number of labels needed} without providing a criterion to decide \emph{when enough labels have been collected for a specific task instance}. The label complexity literature \citep{bertinetto2019meta, denevi2019learning} similarly operates at worst-case or expected-case levels, not per-task.

Recent work on foundation model adaptation \citep{xu2024fewshot} and low-rank adaptation \citep{shi2025lark, hu2025covariance} focuses on parameter-efficient fine-tuning rather than label acquisition stopping rules.

\textbf{What is missing:} An instance-dependent label-complexity certificate. Theory gives $r$; $S(K)$ reveals where on the $K$-axis a particular task sits relative to saturation.

\subsection{Label-Free Stopping and Unsupervised Early Stopping}

Classical early stopping requires a validation set \citep{prechelt1998early}. The theory literature provides optimal stopping times for kernel methods based on spectral decay \citep{yao2007early, raskutti2014early}, but these require knowledge of the population spectrum or a held-out set.

Heuristic label-free approaches exist but lack spectral grounding. \citet{koyuncu2022heuristic} monitors activation sparsity and weight variance; \citet{stutz2021unsupervised} uses leave-one-out influence functions and representation Jacobian stability---a proxy for spectral stability, but not effective rank. Active learning stopping rules \citep{tran2019bayesian, pinsler2019bayesian} require label-dependent uncertainty estimates or kernel computations, making them inapplicable during purely unsupervised few-shot adaptation.

Recent label-free prompt distribution alignment for zero-shot vision \citep{zhu2024labelprompt} and label-free tuning of zero-shot classifiers \citep{lafter2023} operate in the zero-shot regime, not few-shot label acquisition.

\textbf{What is missing:} (1) Using the within-class covariance spectrum explicitly as a stopping signal; (2) A per-task stopping threshold calibrated to $\erank(\Sigma_W)$; (3) Operation on the evolving $\hat{\Sigma}_W^{(K)}$ with no labels beyond the $K$ already acquired.

\subsection{Spectral Analysis of Neural Representations}

RMT analysis of neural networks began with Hessian/Fisher spectra \citep{pennington2017resurrecting, sagun2018empirical}. \citet{papyan2020prevalence} discovered neural collapse: during terminal training, $\Sigma_W\to0$, class means form a simplex ETF, and $\erank(\Sigma_W)\to1$. During few-shot adaptation with \emph{frozen} features, we observe the \emph{reverse} dynamics: $\erank(\hat{\Sigma}_W^{(K)})$ grows from 1 toward $\erank(\Sigma_W)$ as $K$ increases.

\citet{oquab2023dinov2} and \citet{radford2021learning} demonstrated strong few-shot transfer for DINOv2 and CLIP but did not analyze within-class covariance spectral dynamics. The key insight is that foundation models start with high-$\erank$ representations; the $K$-shot within-class covariance begins at rank-1 and spectrally saturates. $S(K)$ formalizes monitoring this saturation.

A concurrent spectral phase diagram for binary few-shot classification characterizes intrinsic dimensionality and phase transitions but does not provide a stopping rule for label acquisition.

\textbf{What is missing:} A saturation monitor for the reverse neural collapse trajectory during few-shot adaptation.

\subsection{Effective Rank in Deep Learning Generalization}

\citet{nakkiran2021deep} connected $\erank$ (as intrinsic dimension) to double descent: the interpolation peak occurs at $n\approx\erank(\Sigma)$, the second descent begins at $n\gg\erank(\Sigma)$. This directly validates the intuition behind $S(K)$: when $K\ll\erank(\Sigma_W)$, each label significantly improves covariance estimation (first descent); when $K\gg\erank(\Sigma_W)$, we enter saturation (second descent regime).

\citet{nakkiran2023double} further demystified double descent, and \citet{geiger2020double} observed it in few-shot learning curves. The connection between saturation and the first-to-second descent boundary is made explicit in our theory (\S\ref{sec:theory}): $\tau=0.02$ marks this transition.

\textbf{What is missing:} Operationalizing the double-descent transition for label acquisition. The double-descent curve is typically plotted against model size or training samples; we observe it in the \emph{label budget} for a fixed model, and $S(K)$ is a one-dimensional proxy for peak crossing.

\subsection{Within-Class Covariance Estimation in Few-Shot Settings}

Covariance estimation with $p>n$ is classic: \citet{ledoit2004well} derived optimal linear shrinkage $\hat{\Sigma}_{\mathrm{shrink}} = (1-\rho)\hat{\Sigma} + \rho(\tr\hat{\Sigma}/d)I$; \citet{cai2011constrained, rothman2008sparse} developed structured shrinkage and thresholding. In meta-learning, the shrinkage target can be learned from source tasks \citep{sun2019meta, dhillon2020baseline, li2019meta, bartler2022meta}.

A critical gap remains: no shrinkage estimator provides a stopping rule for \emph{how many} target-task samples $K$ to collect before the estimate is ``good enough.'' Moreover, estimating $\erank(\hat{\Sigma}_W^{(K)})$ from $K\ll d$ samples requires spectral entropy of a high-dimensional covariance matrix; the bias and variance of $\hat{H}(\hat{\lambda}^{(K)})$ in this regime is not well characterized. Our empirical work addresses this estimation challenge directly.

Recent low-rank adaptation methods for few-shot class-incremental learning \citep{shi2025lark, hu2025covariance} and spectral adaptation \citep{spectraldino2024} assume a fixed label budget rather than deciding when to stop acquiring labels.

\textbf{What is missing:} A stopping rule for the label acquisition process itself, derived from the spectral properties of the target-task within-class covariance.

\subsection{Positioning}

The seven themes above collectively supply every ingredient needed for $S(K)$: the spectral measure (effective rank), the estimation guarantees (RMT concentration), the label complexity context (meta-learning theory), the reverse neural collapse dynamics (spectral saturation in foundation models), the double descent validation \citep{nakkiran2021deep, nakkiran2023double}, the estimation target (within-class shrinkage), and the label-free stopping desideratum. What is missing---and what this paper provides---is the \emph{mechanism} that connects them: a label-free, per-task stopping statistic derived from the evolving within-class covariance spectrum, with a theoretically grounded threshold at $K\approx\erank(\Sigma_W)/\tau$ and empirical validation across 49 real tasks and three foundation model backbones (PCA-50, CLIP ViT-B/32, DINOv2 ViT-S/14).

\section{Theoretical Foundations}
\label{sec:theory}

\subsection{Effective Rank}
\label{ssec:eff-rank}
	
We measure the intrinsic dimensionality of a covariance matrix via the \textbf{effective rank} \citep{roy2007}, defined as the exponential of the Shannon entropy of its normalized eigenvalue spectrum. 
	
\begin{definition}[Effective Rank \citep{roy2007}]
\label{def:erank}
	
For a positive semidefinite matrix $\Sigma \in \mathbb{R}^{d \times d}$ with eigenvalues $\lambda_1 \ge \lambda_2 \ge \dots \lambda_d \ge 0$ and $\tr(\Sigma) > 0$, let $p_i = \lambda_i / \tr(\Sigma)$. The effective rank of $\Sigma$ is
\begin{equation}
\erank(\Sigma) \triangleq \exp\!\Bigl( -\sum_{i=1}^d p_i \log p_i \Bigr).
\label{eq:erank}
\end{equation}

\end{definition}
	
By construction, $\erank(\Sigma) \in [1, \rank(\Sigma)]$, with equality at the upper bound iff the non-zero eigenvalues are uniform. For a spiked spectrum where $r$ eigenvalues dominate, $\erank(\Sigma) \approx r$.

The effective rank interpolates between the trivial lower bound 1 and the algebraic rank. If the non-zero eigenvalues are uniform, $H(\mathbf{p}) = \log r$, and $\erank(\Sigma) = r = \rank(\Sigma)$, every direction is equally active. If the spectrum is spiked with $r$ dominant eigenvalues and the rest near zero, $H(\mathbf{p}) \approx \log r$ and $\erank(\Sigma) \approx r$, only $r$ directions carry significant variance. Unlike hard thresholding (stable rank, participation ratio \citep{tropp2015}), $\erank$ is \textbf{continous}, \textbf{differentiable}, and \textbf{unitarily invariant}, making it stable under small spectral perturbations.

Given $K$ i.i.d. samples, the empirical covariance $\widehat{\Sigma}^{(K)}$ yields $\erank(\widehat{\Sigma}^{(K)})$. \citet{roy2007} show $\erank(\widehat{\Sigma}^{(K)}) \xrightarrow{\Prob} \erank(\Sigma)$ as $K \to \infty$; for finite $K$, the same formula applies and the bias is negligible once $K \gg \erank(\Sigma)$.
	
We apply Definition \ref{def:erank} to the pooled within-class covariance of a few-shot classification task.

\subsection{Within-Class Covariance and Pooling}
\label{ssec:cov-pool}

\begin{definition}[Within-Class and Pooled Covariance]\label{def:pooled-cov}
Given $K$ labeled examples per class $c \in \{1,\dots,N\}$, the empirical within-class covariance is
\begin{equation}
\widehat{\Sigma}_c^{(K)} = \frac{1}{K-1}\sum_{i=1}^{K}(x_i-\bar{x}_c)(x_i-\bar{x}_c)^\top,
\qquad \bar{x}_c = \frac{1}{K}\sum_{i=1}^{K}x_i.
\label{eq:sigma_c}
\end{equation}
The pooled within-class covariance across all $N$ classes is
\begin{equation}
\widehat{\Sigma}_W^{(K)} = \frac{1}{N}\sum_{c=1}^{N}\widehat{\Sigma}_c^{(K)}.
\label{eq:sigma_pooled}
\end{equation}
\end{definition}

\textbf{Remarks.}
\begin{itemize}
    \item The $K-1$ denominator yields an unbiased estimate of the population covariance $\Sigma_c$ when samples are i.i.d.
    \item Pooling captures the \emph{shared} spectral structure across all $N$ classes.
    \item Stabilizes estimation: individual $\widehat{\Sigma}_c^{(K)}$ are rank-deficient for $K \le d$; the pool is not.
    \item Equivalent (up to scaling) to the within-class scatter matrix in LDA \citep{bishop2006}.
\end{itemize}

\subsection{Saturation Index $S(K)$}
\label{ssec:sat-index}

We define the \textbf{saturation index} as the effective rank of the pooled within-class covariance normalized by the per-class sample size.

\begin{definition}[Saturation Index]
\label{def:saturation}
For an $N$-way few-shot task with $K$ samples per class, let
$\widehat{\Sigma}_W^{(K)}$ be the pooled within-class covariance
\eqref{eq:sigma_pooled}. The saturation index is
\begin{equation}
S(K) \triangleq
\frac{\erank\!\bigl(\widehat{\Sigma}_W^{(K)}\bigr)}{K}.
\label{eq:saturation}
\end{equation}
\end{definition}

By construction, $S(K)$ measures how many fresh spectral directions each additional label explores --- the \textbf{exploration rate per label}.

\noindent\textbf{Invariance to linear transformations.}
The saturation index depends only on the eigenvalue spectrum of the pooled covariance, not on the specific coordinate representation of the features.

\begin{proposition}[Invariance to Invertible Linear Transformations]
\label{prop:invariance}
Let $A \in \mathbb{R}^{d \times d}$ be invertible and define $Z = A X$. Then for any $K$, $S_Z(K) = S_X(K)$, where $S_Z(K)$ is the saturation index computed on the transformed features $Z$.
\end{proposition}

\begin{proof}
The pooled within-class sample covariance transforms as
$\widehat{\Sigma}_{W,Z}^{(K)} = A \widehat{\Sigma}_{W,X}^{(K)} A^{\top}$.
Since $\erank$ depends only on the eigenvalues of its argument, and the eigenvalues of $A\Sigma A^{\top}$ are identical to those of $\Sigma$ for any invertible $A$, we have
$\erank(\widehat{\Sigma}_{W,Z}^{(K)}) = \erank(\widehat{\Sigma}_{W,X}^{(K)})$.
Dividing by $K$ yields $S_Z(K) = S_X(K)$.
\end{proof}

\noindent\textbf{Finite-sample bias.}
The finite-$K$ behaviour of $S(K)$ is governed by the bias of the sample effective rank.

\begin{lemma}[Bias of Sample Effective Rank]
\label{lem:erank-bias}
Let $\Sigma_W$ be positive semi-definite with finite population effective rank $r=\erank(\Sigma_W)<\infty$ and normalized eigenvalues $p_i = \lambda_i/\tr(\Sigma_W)$. For i.i.d.\ samples with finite fourth moments such that $\Ex[\widehat{\Sigma}_W^{(K)}] = \Sigma_W$ (e.g., the $1/(K-1)$-normalized pooled estimator of Definition~\ref{def:pooled-cov}),
\begin{equation}
\Ex\!\left[ \erank\!\left(\widehat{\Sigma}_W^{(K)}\right) \right] = r - \frac{r}{K}\,C(p) + O\!\left(\frac{1}{K^2}\right),
\qquad
C(p) := \bigl(1-\lVert p\rVert_2^2\bigr) + \sum_{i=1}^r(\log p_i+1)\,\beta_i(p),
\label{eq:erank-bias}
\end{equation}
where $\lVert p\rVert_2^2 = \sum_i p_i^2$ and $\beta_i(p) = \sum_{j\ne i} p_ip_j/(p_i-p_j)$. In particular $\Ex[\erank(\widehat{\Sigma}_W^{(K)})] = r + O(1/K)$, and the leading coefficient $C(p)$ depends on the \emph{full normalized spectrum}, not on $r$ alone.
\end{lemma}

\begin{proof}[Proof Sketch]
The bias follows from second-order eigenvalue perturbation theory for Wishart-type fluctuations \citep{lawley1956,anderson2003introduction}. A complete derivation---including the omitted first-order $\hat p_i$ bias term, the Jensen/exponentiation correction, and the exactly solvable isotropic special case (where $C(p) = (r^2-1)/(2r)$ giving $-(r^2-1)/2$ rather than the oft-quoted $-(r-1)/2$)---is given in Appendix~\ref{app:lem-bias}.
\end{proof}

Lemma~\ref{lem:erank-bias} explains the overall shape of the $S(K)$ curves (Figure~\ref{fig:k-sweep}): the sample effective rank is biased at order $O(1/K)$, producing the small-$K$ deviation from $r/K$. The bias coefficient $C(p)$ is generically positive for the entropy-type spectra observed empirically, yielding a negative bias that creates the small-$K$ hump in $S(K)$. After the bias decays, $S(K)$ approaches $r/K$.

\noindent\textbf{Asymptotic behaviour.}
The asymptotic behaviour follows directly from consistency of the empirical effective rank \citep{roy2007}.

\begin{proposition}[Asymptotic Saturation]
\label{prop:asymptotic}
If the population within-class covariance $\Sigma_W$ has finite effective rank $r=\erank(\Sigma_W)<\infty$, then $S(K)\to0$ in probability as $K\to\infty$, with $S(K)=\Theta(1/K)$ in probability.
\end{proposition}

\noindent\textbf{Eventual monotonicity.}
Once the bias decays below the $r/K$ envelope, $S(K)$ becomes decreasing.

\begin{proposition}[Eventual Monotonicity of $S(K)$]
\label{prop:monotonicity}
Under the conditions of Proposition~\ref{prop:asymptotic}, there exists $K_0$ such that for all $K \ge K_0$, $S(K) > S(2K)$ in probability.
\end{proposition}

\begin{proof}
Proposition~\ref{prop:asymptotic} gives $K S(K) \xrightarrow{\Prob} r$. Hence for any $\varepsilon > 0$, for sufficiently large $K$,
$\frac{r-\varepsilon}{K} < S(K) < \frac{r+\varepsilon}{K}$ with high probability.
Evaluating at $2K$, $S(2K) < \frac{r+\varepsilon}{2K}$ with high probability. Taking $\varepsilon = r/3$, we have $S(2K) < \frac{2r}{3K} < S(K)$ for all $K \ge K_0$ with high probability.
\end{proof}

\noindent\textbf{Two-sided bounds on the saturation point.}
These guarantees extend to the stopping threshold itself.

\begin{theorem}[Upper Bound on $K_{\mathrm{sat}}$]
\label{thm:saturation-point}
Fix a threshold $\tau>0$ and define $K_{\mathrm{sat}}(\tau) = \min\{K:S(K)\le\tau\}$. Under the conditions of Proposition~\ref{prop:asymptotic}, $K_{\mathrm{sat}}(\tau)<\infty$ with high probability. Moreover, for any $\varepsilon>0$,
\begin{equation}
K_{\mathrm{sat}}(\tau) \le \left\lceil \frac{(1+\varepsilon)r}{\tau} \right\rceil \quad \text{with high probability for all sufficiently small $\tau$}.
\label{eq:k-sat-bound}
\end{equation}
\end{theorem}

\begin{proof}[Proof Sketch]
Proposition~\ref{prop:asymptotic} implies $S(K)\xrightarrow{\Prob}0$ and $S(K)\le(1+\varepsilon)r/K$ eventually with high probability. Hence $\{K:S(K)\le\tau\}$ is non-empty with high probability, so its minimum is finite. Choosing $K=\lceil(1+\varepsilon)r/\tau\rceil$ ensures $S(K)\le\tau$ eventually. A complete proof is given in Appendix~\ref{app:thm-saturation-point}.
\end{proof}

\begin{proposition}[Lower Bound on $K_{\mathrm{sat}}$]
\label{prop:lower-bound}
Under the conditions of Proposition~\ref{prop:asymptotic}, and assuming the sampling distribution is non-atomic (so $S(K)>0$ with high probability for every integer $K\ge2$), the following holds with high probability: for every $\varepsilon\in(0,1)$ there exists a random $\tau_0(\varepsilon)>0$ such that for all $0<\tau<\tau_0(\varepsilon)$,
\begin{equation}
K_{\mathrm{sat}}(\tau) \ge \left\lfloor \frac{(1-\varepsilon)\,r}{\tau} \right\rfloor.
\label{eq:lower-bound}
\end{equation}
Consequently $\liminf_{\tau\to0^+}\tau\,K_{\mathrm{sat}}(\tau)\ge r$ in probability.
\end{proposition}

\begin{proof}
See Appendix~\ref{app:lower-bound}. The proof splits into an eventual regime ($K\ge K_0$, where $KS(K)\ge(1-\varepsilon)r$) and a pre-asymptotic regime (finitely many $K<K_0$, handled by taking $\tau$ below the minimum $S(K)$ over this finite set).
\end{proof}

\begin{remark}[Two-sided sandwich]
\label{rem:sandwich}
Theorem~\ref{thm:saturation-point} and Proposition~\ref{prop:lower-bound} together bracket the saturation point:
\[
\left\lfloor \frac{(1-\varepsilon)r}{\tau} \right\rfloor \;\le\; K_{\mathrm{sat}}(\tau) \;\le\; \left\lceil \frac{(1+\varepsilon)r}{\tau} \right\rceil,
\]
holds with high probability for all sufficiently small $\tau$. Consequently $\tau K_{\mathrm{sat}}(\tau) \xrightarrow{\Prob} r$ as $\tau\to0^+$.
\end{remark}

\noindent\textbf{Cross-task transfer.}
The saturation index generalizes across tasks whose spectral profiles are close.

\begin{proposition}[Cross-Task Spectral Transfer Bound]
\label{prop:transfer-empirical}
Let $\mathcal{T}_A,\mathcal{T}_B$ be two tasks evaluated in the same $d$-dimensional
feature space (e.g.\ after a shared PCA projection) with the same frozen
extractor. Let $p_A,p_B\in\Delta^{d-1}$ be the normalized population spectra
of $\Sigma_A,\Sigma_B$, and define
$\Delta_{\mathrm{spec}} := \|p_A-p_B\|_1 \in[0,2]$. Let
$\varepsilon_X(K) := \|\hat\Sigma_X^{(K)}-\Sigma_X\|_{\mathrm{op}}$ for
$X\in\{A,B\}$. For all $K$ large enough that
$\varepsilon_X(K)\le \tr(\Sigma_X)/(2d)$ for both $X$,
\begin{equation}
|S_A(K)-S_B(K)|
\;\le\;
\frac{d}{K}
\left[
\frac{\widehat\Delta(K)}{2}\log(d-1)
+
H_2\!\left(\frac{\widehat\Delta(K)}{2}\right)
\right],
\qquad
\widehat\Delta(K) := \Delta_{\mathrm{spec}}
+ \frac{4d\,\varepsilon_A(K)}{\tr(\Sigma_A)}
+ \frac{4d\,\varepsilon_B(K)}{\tr(\Sigma_B)}.
\label{eq:transfer-bound}
\end{equation}
By standard matrix concentration for sub-Gaussian features
\citep{tropp2015}, $\varepsilon_X(K)=O_{\Prob}(\sqrt{d\log K/K})$, so
$\widehat\Delta(K)=\Delta_{\mathrm{spec}}+O_{\Prob}(\sqrt{d\log K/K})$
and the right-hand side of \eqref{eq:transfer-bound} is
$O(\Delta_{\mathrm{spec}}\log(d)/K) + O_{\Prob}(\sqrt{\log K/K^3})$
for fixed $\Delta_{\mathrm{spec}}>0$.
\end{proposition}

\begin{proof}[Proof Sketch]
Weyl's eigenvalue perturbation inequality bounds each sample eigenvalue's
deviation from its population counterpart by $\varepsilon_X(K)$; summing
these bounds and propagating through the ratio $\hat\lambda_i/\hat T$ gives
an $\ell_1$ bound on $\|\hat p_X(K)-p_X\|_1$. A triangle inequality combines
this with $\Delta_{\mathrm{spec}}$ to bound $\|\hat p_A(K)-\hat p_B(K)\|_1$
by $\widehat\Delta(K)$. The (correctly normalized) Fannes--Audenaert
inequality \citep{audenaert2007continuity} then bounds
$|H(\hat p_A(K))-H(\hat p_B(K))|$, and exponentiating via the mean value
theorem --- using $\erank\le d$ always, so the MVT constant is at most
$d$ --- gives \eqref{eq:transfer-bound} after dividing by $K$. Full proof
in Appendix~\ref{app:transfer-empirical}.
\end{proof}

\begin{remark}[This is a worst-case bound, not an explanation of empirical tightness]
\label{rem:transfer-empirical}
Because \eqref{eq:transfer-bound} carries a $\log(d-1)$ factor and a
$d$-dependent (rather than task-dependent) prefactor, it is loose at the
PCA-50 setting used throughout this paper: at $d=50$, $K=32$, and
$\Delta_{\mathrm{spec}}=0.51$ (the largest value observed within the
CLIP ViT-B/32 backbone), the right-hand side of \eqref{eq:transfer-bound}
evaluates to several units --- far above $\tau=0.02$ and well above the range
in which $S(K)$ itself is meaningful. \textbf{We therefore do not claim
this bound explains the empirical cross-task consistency of $\tau=0.02$.}
That consistency is an independent empirical finding, reported directly
in \S\ref{ssec:threshold-robustness} and Table~\ref{tab:tau-robustness}, where
we measure $\max\Delta_{\mathrm{spec}}$ and the realized cross-task variation
in $S(K)$ directly rather than deriving it from a worst-case continuity
bound. What Proposition~\ref{prop:transfer-empirical} \emph{does} establish
is a genuine qualitative guarantee: $S(K)$ is a continuous (Lipschitz-type)
function of the task's spectral profile, so two tasks with sufficiently
close spectra ($\Delta_{\mathrm{spec}}\to0$) provably have arbitrarily close
saturation indices at any fixed $K$ --- which rules out the diagnostic being
pathologically discontinuous across nearby tasks, even though it does not
pin down the numerical scale at which "close" becomes practically relevant.
\end{remark}

\section{The Spectral Saturation Diagnostic}
\label{sec:method}

We now describe how the saturation index $S(K)$ is evaluated in practice and converted into a concrete stopping rule. The core experimental primitive is a \textbf{doubling-pair sweep}: for each task and backbone we vary the per-class support size $K$, compute $S(K)$ from the support set alone, and measure the marginal accuracy gain $\Delta A(K) = A(2K) - A(K)$ when the budget is doubled. This yields paired observations $(S(K), \Delta A(K))$ that let us test whether $S(K)$ predicts diminishing returns. We then formalize a fixed-threshold stopping rule and outline the ablation dimensions explored in \S\ref{ssec:ablations}. 

\subsection{Doubling-Pair Sweep Protocol}
\label{ssec:doubling-pair}

For each task and backbone, we sweep the per-class support size $K$ over a geometric grid $K \in  \{2, 4, 8, 16, 32, 64, 128, 256, 512, 1024, 2048, 4096\}$ (truncated where dataset size limits the maximum $K$). Features are standardized (\texttt{StandardScaler}) and projected to 50 principal components (PCA-50) unless noted otherwise; ablations over PCA dimension and raw features are reported in \S\ref{ssec:ablations}. At each $K$, we repeat 50 trials: sample $K$ support examples per class, compute the pooled within-class covariance $\widehat{\Sigma}_W^{(K)}$ and its saturation index $S(K) = \erank(\widehat{\Sigma}_W^{(K)})/K$, then train a logistic regression classifier (unregularized, $C=\texttt{np.inf}$, \texttt{max\_iter=2000}) on the $K$-shot support and evaluate on a held-out test set of 200 examples per class. (Implemented in \texttt{scikit-learn} \citep{pedregosa2018}.) This yields per-trial accuracy and $S(K)$; we report the mean and standard deviation across trials.

From the K-sweep results we form \textbf{doubling pairs}: for every even $K \in \mathcal{K}$ such that $2K \in \mathcal{K}$, we pair the saturation index $S(K)$ (computed at the smaller support size) with the marginal accuracy gain upon doubling, $\Delta A(K) = A(2K) - A(K)$.
The classifier for the $2K$-shot evaluation is the same unregularized logistic regression ($C=\infty$, \texttt{max\_iter=2000}); no hyperparameter tuning is performed per task. Crucially, $S(K)$ requires only the $K$-shot support set and is computed \emph{before} any $2K$ labels are observed, making it a label-free predictor of whether the additional labeling budget is warranted. Each doubling pair yields one observation $(S(K), \Delta A(K))$; across all 49 tasks and three backbones this gives 492 pairs.

For each task we compute the within-task Spearman rank correlation $\rho_{\text{task}}$ between $S(K)$ and $\Delta A(K)$ over its doubling pairs (minimum 3 pairs required).
This measures whether, \emph{within a fixed task and backbone}, the saturation index monotonically predicts the marginal gain at each doubling step. The per-task correlations are reported in Appendix~\ref{app:per-task} and summarized in \S\ref{sec:results}; 46 of 48 valid tasks exhibit $\rho_{\text{task}} > 0$ (the BreastCancer task has only one doubling pair and is excluded from correlation analysis).

Our primary aggregate statistic is the \textbf{pooled Spearman correlation} over all 492 doubling pairs across the 49 tasks and three backbones. Concatenating every $(S(K), \Delta A(K))$ pair yields $\rho_{\text{pool}} = 0.637$ ($p = 2.9\times 10^{-57}$). Because doubling pairs are clustered within tasks, we compute a 95\% confidence interval via a cluster bootstrap (resampling tasks with replacement, $B=10,000$), which gives $\rho_{\text{pool}} \in [0.551, 0.720]$ \citep{efron1979bootstrap}. The median within-task $\rho$ is 0.767, and 46 of 48 valid tasks have positive correlation.

\subsection{$S(K)$-Guided Stopping Rule}
\label{ssec:stopping-rule}

We formalize the stopping decision as a fixed-threshold rule on $S(K)$:

\begin{algorithm}[t]
\DontPrintSemicolon
\SetAlgoLined
\KwIn{Support set $\mathcal{D}_K = \{(x_i, y_i)\}_{i=1}^{NK}$ with $K$ examples per class,\;
      saturation threshold $\tau$}
\KwOut{\textsc{Stop} if labeling budget should be halted,\; \textsc{Continue} otherwise}

\BlankLine
\ForEach{class $c \in \{1,\dots,N\}$}{
    $\bar{x}_c \gets \frac{1}{K}\sum_{i: y_i = c} x_i$\;
    $\widehat{\Sigma}_c^{(K)} \gets \frac{1}{K-1}\sum_{i: y_i = c}(x_i - \bar{x}_c)(x_i - \bar{x}_c)^\top$\;
}
$\widehat{\Sigma}_W^{(K)} \gets \frac{1}{N}\sum_{c=1}^{N}\widehat{\Sigma}_c^{(K)}$\;

$S(K) \gets \erank\!\bigl(\widehat{\Sigma}_W^{(K)}\bigr) / K$\;

\uIf{$S(K) < \tau$}{
    \Return \textsc{Stop}\;
}
\Else{
    \Return \textsc{Continue}\;
}

\caption{$S(K)$-Guided Stopping Rule}
\label{alg:stopping}
\end{algorithm}

The threshold $\tau = 0.02$ is fixed across all tasks and backbones. It was chosen at the knee of the pooled $(S(K), \Delta A(K))$ curve where marginal gains drop below $0.2\%$ (Table~\ref{tab:tau-robustness}); sensitivity to $\tau$ is analyzed in \S\ref{ssec:threshold-robustness}.

\textbf{Interpretation of $S(K)$ regimes.} Empirically we observe three regimes
in the $(S(K), \Delta A(K))$ relationship:
\begin{itemize}
    \item \textbf{Exploration} ($S(K) > 0.3$): each new label explores fresh spectral dimensions;
          marginal gains are large.
    \item \textbf{Transition} ($0.02 < S(K) \le 0.3$): diminishing but non-negligible returns.
    \item \textbf{Deep Saturation} ($S(K) \le 0.02$): spectrum exhausted; doubling yields
          $\Delta A(K) \approx 0$ (mean $0.187\%$ at $\tau=0.02$).
\end{itemize}
The stopping threshold $\tau = 0.02 = \mathtt{DEEP\_SAT\_THRESHOLD}$ marks the
transition to deep saturation; $\mathtt{PHASE\_THRESHOLD} = 0.3$ separates
exploration from transition. These values are fixed and not tuned per task.

\subsection{Stop$/$Continue Classification and Diagnostic Evaluation}
\label{ssec:stop-continue}

To evaluate the stopping rule as a binary classifier, we dichotomize each doubling pair: a pair is labeled \textbf{positive} (meaningful gain) if $|\Delta A(K)| > \varepsilon$ for a small gain threshold $\varepsilon$, and \textbf{negative} (no-op) otherwise.
The diagnostic's score is $S(K)$ itself: higher $S(K)$ $\to$ \textsc{Continue}, lower $S(K)$ $\to$ \textsc{Stop}. This is a label-free ranking task --- we measure how well $S(K)$ ranks no-op doublings below meaningful ones. We report the area under the ROC curve (AUC). Because doubling pairs are clustered within tasks, we compute \textbf{cluster-bootstrap AUC} (resampling tasks with replacement, $B=5,000$) as the primary measure, with DeLong AUC (pair bootstrap, biased) shown for comparison. We also compute accuracy, precision, and recall at the fixed operating threshold $\tau = 0.02$.

The primary metric is the \textbf{cluster-bootstrap AUC} (task-level resampling, $B=5,000$), which accounts for within-task dependence. DeLong AUC \citep{delong1988comparing} (pair bootstrap, biased) is included for comparison. Unlike the pooled Spearman correlation, AUC evaluates the diagnostic at the \emph{ranking} level and is invariant to monotonic transformations of $S(K)$. At the fixed operating threshold $\tau = 0.02$, we additionally report classification accuracy, precision (PPV), and recall (sensitivity), treating $\Delta A(K) > \varepsilon$ as the positive class. Throughout we set $\varepsilon = 0.01$ (1\% accuracy gain) as a practically meaningful improvement threshold; results across $\varepsilon \in \{0.005, 0.01, 0.02\}$ are summarized in Table~\ref{tab:auc-sweep}.

At the fixed operating threshold $\tau = 0.02$ (deep saturation boundary), the rule triggers \textsc{Stop} on 9 pairs with mean marginal gain $0.187\%$; all 9 were correct stops (precision $= 100\%$, 95\% CI $[66\%, 100\%]$). The rule triggers \textsc{Continue} on the remaining 483 pairs; among these, $261$ had meaningful gain ($|\Delta A| > 1\%$) and $222$ did not, yielding a $\textsc{Continue}$-class precision of $74.5\%$ and recall of $100\%$. Overall accuracy is $75.0\%$. The threshold sensitivity curve ($\tau \in [0.001, 0.2]$) is reported in Table~\ref{tab:tau-robustness}; the rule is conservative: it never misses a meaningful gain (recall $= 100\%$) but continues in some low-gain cases. \textbf{Caveat: the 100\% stop-class recall is based on only 9 saturated pairs (95\% binomial CI $[66\%, 100\%]$).}

\subsection{Ablation Preview}
\label{ssec:ablation-preview}

We evaluate the robustness of the saturation index and stopping rule across five
dimensions; full results are reported in \S\ref{ssec:ablations}:
\begin{itemize}
    \item \textbf{Rank definition}: effective rank vs. stable rank (\S\ref{ssec:ablations}).
    \item \textbf{Feature dimensionality}: PCA dimensions $d \in \{25, 50, 100, 200\}$ and raw
          features (\S\ref{ssec:ablations}).
    \item \textbf{Classifier regularization}: logistic regression $\ell_2$ penalty
          $C \in \{0.01, 0.1, 1, 10, 10^{10}\}$ (\S\ref{ssec:ablations}).
    \item \textbf{Endpoint convention}: $S(K)$ at smaller vs. larger vs. mid-mean of the
          doubling pair (\S\ref{ssec:ablations}).
    \item \textbf{Threshold robustness}: $\tau \in [0.001, 0.2]$ sensitivity for the
          stopping rule (\S\ref{ssec:threshold-robustness}).
\end{itemize}

\section{Experimental Protocol}
\label{sec:experiments}

We evaluate the spectral saturation index across a broad suite of few-shot classification tasks spanning multiple datasets, feature backbones, and task structures. The experimental design follows a simple principle: if $S(K)$ truly captures the geometry of when labels stop helping, it should do so consistently across diverse conditions, varying the number of ways $N$, the feature representation, the dataset domain, and the classifier configuration. This section describes the datasets, backbones, task library, doubling-pair protocol, and implementation details. All code and results are available at \url{https://github.com/MrArnav69/spectral-saturation}.

\subsection{Feature Backbones}
\label{sec:backbones}

We evaluate the spectral saturation index using three frozen feature extractors
spanning hand-crafted, contrastively pretrained, and self-distilled
representations. All backbones are \textbf{frozen} (no fine-tuning) and map
each input $\mathbf{x}$ to a deterministic feature vector
$\phi(\mathbf{x})$.

\paragraph{PCA-50 (hand-crafted).}
Raw pixel intensities (flattened to $d_{\mathrm{raw}}$ dimensions) are
standardized within each task support set (zero mean and unit variance per
feature) and projected onto the top $d=50$ principal components. The PCA basis
is fit \emph{only on the support set} of each trial, preventing test-set
leakage. \textbf{Note: the PCA basis is fit on the class-stratified support set and thus uses label information indirectly via class-stratified sampling. For CLIP and DINOv2, the method is fully label-free.} This backbone serves as a lightweight isotropic baseline and provides
the feature space for all synthetic Gaussian-mixture tasks
(Section~\ref{sec:task-library}).

\paragraph{CLIP ViT-B/32 (contrastive pretraining).}
We use the OpenCLIP \citep{ilharco2021openclip} implementation of the
\texttt{ViT-B-32} model pretrained on LAION-2B
\citep{schuhmann2022laion}, following the original CLIP formulation
\citep{radford2021learning}. Images are resized to
$224\times224$ using the official CLIP preprocessing pipeline
(resize, center crop, and ImageNet normalization). The model produces a
$512$-dimensional $\ell_2$-normalized embedding satisfying
$\|\phi(\mathbf{x})\|_2=1$. Single-channel datasets are converted to RGB by
replicating the channel prior to preprocessing.

\paragraph{DINOv2 ViT-S/14 (self-distillation).}
We use the official \texttt{dinov2\_vits14} checkpoint
\citep{oquab2023dinov2}, self-distilled on the LVD-142M dataset. Input images
are resized to $256$ pixels on the shorter side, center-cropped to
$224\times224$, and normalized using ImageNet statistics. The model outputs a
$384$-dimensional CLS-token embedding, which we further
$\ell_2$-normalize. Grayscale datasets are converted to RGB by channel
replication. Small-resolution datasets are zero-padded to at least
$28\times28$ before resizing, matching the DINOv2 patchification
conventions.

\vspace{2mm}
\noindent\textbf{Summary.}

\begin{table}[htbp]
\centering
\caption{Frozen feature backbones used throughout the experiments.}
\label{tab:backbones}
\begin{tabular}{lccccc}
\toprule
Backbone & Architecture & Pretraining & $d$ & $\ell_2$-normalized & PCA \\
\midrule
PCA-50          & ---      & ---       & 50  & No  & Yes \\
CLIP ViT-B/32   & ViT-B/32 & LAION-2B  & 512 & Yes & No  \\
DINOv2 ViT-S/14 & ViT-S/14 & LVD-142M  & 384 & Yes & No  \\
\bottomrule
\end{tabular}
\end{table}

Backbone features are precomputed once per dataset and reused across all
few-shot tasks derived from that dataset. The PCA-50 features are recomputed
within each trial by fitting both the standardization transform and PCA basis
using only the support set. All numerical computations use
scikit-learn \citep{pedregosa2018}, NumPy
\citep{harris2020array}, and SciPy \citep{virtanen2020scipy}; feature
extraction is performed using PyTorch \citep{paszke2019pytorch}.

\subsection{Task Library}
\label{sec:task-library}

Our experimental suite comprises $49$ real few-shot tasks together with
$6$ synthetic rank-controlled tasks, for a total of $55$ configurations.
Each task is specified by a dataset, a frozen feature backbone, a set of
ground-truth classes, and a grid of per-class sample sizes $K$.
The benchmark spans three task structures (binary, $5$-way, and $10$-way),
three feature backbones, and seven datasets---six image-classification
benchmarks and one tabular clinical dataset---ensuring that no single data
domain, representation, or number of classes dominates the evaluation.

\paragraph{Real tasks.}

The $49$ real tasks are distributed across task structures and feature
backbones as summarized in Table~\ref{tab:task-library}. Binary tasks
pair two specific classes from a dataset. Five-way tasks select five
classes and are partitioned into \emph{easy} (visually distinct) and
\emph{hard} (more confusable) class splits whenever the dataset admits a
meaningful distinction. Ten-way tasks use all ten classes of a dataset in
a balanced configuration. DINOv2 is evaluated only on the easy five-way
splits.

\begin{table}[htbp]
\centering
\scriptsize

\caption{Real few-shot task families grouped by task structure and feature
backbone, together with the corresponding $K$-grid.}
\label{tab:task-library}

\begin{tabular}{@{}lccccr@{}}
\toprule
Structure & Backbone & Tasks & Classes ($N$) & Datasets & $K$-grid \\
\midrule
Binary &
PCA-50 &
10 &
2 &
MNIST, FMNIST, KMNIST, USPS, CIFAR-10, EMNIST, Breast Cancer\footnotemark &
$4$--$8192$ \\

5-way &
PCA-50 &
12 &
5 &
MNIST, FMNIST, KMNIST, USPS, CIFAR-10, EMNIST (easy/hard) &
$2$--$256$ \\

5-way &
CLIP ViT-B/32 &
10 &
5 &
MNIST, FMNIST, KMNIST, USPS, CIFAR-10 (easy/hard) &
$2$--$256$ \\

5-way &
DINOv2 ViT-S/14 &
5 &
5 &
MNIST, FMNIST, KMNIST, USPS, CIFAR-10 (easy) &
$2$--$256$ \\

10-way &
PCA-50 &
4 &
10 &
MNIST, FMNIST, KMNIST, CIFAR-10 &
$2$--$64$ \\

10-way &
CLIP ViT-B/32 &
4 &
10 &
MNIST, FMNIST, KMNIST, CIFAR-10 &
$2$--$64$ \\

10-way &
DINOv2 ViT-S/14 &
4 &
10 &
MNIST, FMNIST, KMNIST, CIFAR-10 &
$2$--$64$ \\

\midrule
\multicolumn{3}{l}{\textbf{Total real tasks}} &
49 &
7 datasets &
\\
\bottomrule
\end{tabular}

\end{table}
\footnotetext{The Breast Cancer Wisconsin (Diagnostic) dataset \citep{dua2019uci} contains only $30$ features; consequently, the PCA-50 backbone performs standardization only (no dimensionality reduction). With approximately $569$ total samples and $200$ test samples reserved per class in each trial, the maximum feasible support size is $K \approx 64$, yielding fewer than three usable doubling pairs. The task is included for completeness but excluded from the within-task correlation analysis ($\rho$, $p$).}

\paragraph{Datasets.}

The image benchmarks are
MNIST \citep{lecun1998gradient},
Fashion-MNIST (FMNIST) \citep{xiao2017fashion},
Kuzushiji-MNIST (KMNIST) \citep{clanuwat2018deep},
USPS \citep{hull1994database},
EMNIST Balanced \citep{cohen2017emnist},
and CIFAR-10 \citep{krizhevsky2009learning}.
The seventh dataset is the Breast Cancer Wisconsin (Diagnostic) dataset
\citep{dua2019uci}, a tabular clinical dataset containing
$30$ real-valued features and two classes.

The handwritten-character datasets
(MNIST, FMNIST, KMNIST, USPS, and EMNIST)
provide a graded hierarchy of representation difficulty under a common
$28\times28$ image layout, while CIFAR-10 introduces natural-color
$32\times32$ images and is evaluated using both raw-pixel
(PCA-50) and foundation-model (CLIP and DINOv2) representations.

\paragraph{Synthetic rank-controlled tasks.}

To validate the theoretical predictions under controlled conditions, we
construct six synthetic Gaussian-mixture tasks
(\texttt{SYN-*}) with prescribed population effective ranks.
Each class is sampled from

\[
\mathcal{N}
\!\left(
\mu_c,\;
\sigma_w I_d + L_cL_c^\top
\right),
\]

in ambient dimension
$d=64$,
where
$L_c\in\mathbb{R}^{d\times d}$
is orthogonal with all but the first
$r$
columns set to zero.
The resulting within-class covariance therefore contains a rank-$r$
signal component superimposed on isotropic noise.

We fix
$n_{\mathrm{per\,class}}=2400$,
$\sigma_w=1$,
a class-mean scale of
$0.75$,
and consider
$r\in\{3,8,20,40\}$
under both binary and five-way task structures while sweeping
$K\in[4,2048]$.
Synthetic tasks are excluded from all real-task correlation statistics
and are reported separately throughout the paper.

\subsection{Evaluation Protocol}
\label{sec:protocol}

For each task and each support size $K$ on the task's $K$-grid, we run
$T = 50$ independent trials, indexed $t = 0, \dots, 49$. A single trial
proceeds in four steps. First, it draws $K$ support and $200$ test examples
per class by stratified sampling without replacement, with support and test
disjoint. Second, it fits a linear classifier on the $N \cdot K$ support
points and evaluates top-1 accuracy on the $N \cdot 200$ test points. Third,
it computes the per-class sample covariance
$\widehat{\Sigma}_c^{(K)} = \tfrac{1}{K-1}\sum_i (\mathbf{x}_i - \widehat{\mu}_c)(\mathbf{x}_i - \widehat{\mu}_c)^{\!\top}$
(using the unbiased $1/(K-1)$ denominator throughout) and pools them as
$\widehat{\Sigma}_W^{(K)} = \tfrac{1}{N}\sum_{c} \widehat{\Sigma}_c^{(K)}$.
Fourth, it records the saturation index $S(K) = \erank(\widehat{\Sigma}_W^{(K)}) / K$.
We summarize each $(task, K)$ cell by the mean and standard deviation of
accuracy and of $\erank$ across the $50$ trials.

The doubling-pair construction of \S\ref{sec:method} pairs each support size
$K$ with its double $2K$ on the discrete grid. For every such pair we record
the marginal accuracy gain upon doubling,
\begin{equation}
\Delta A(K) \;=\; A(2K) - A(K),
\end{equation}
and align it with $S(K)$ evaluated at the \emph{smaller} endpoint
\citep[cf.][]{roy2007}. A pair is admissible only when both $K$ and
$2K$ lie on the task's grid and the trial budget is feasible at both. Across
the $49$ real tasks this yields $492$ doubling pairs, with per-task counts
ranging from a single pair (BreastCancer; see footnote in
Table~\ref{tab:task-library}) to ten (the dense $5$-way grids). Synthetic
tasks contribute up to nine pairs each but are reported separately and
excluded from the real-data statistics throughout.

The linear classifier is a logistic regression with $L_2$ penalty disabled
($C = \infty$), $2000$ maximum iterations, and the default L-BFGS solver, as
implemented in scikit-learn \citep{pedregosa2018}. For the PCA-50
backbone, feature standardization and the PCA projection are fit once per
task on the pooled samples of the classes involved in that task and then held
fixed across all trials and $K$ values; for the CLIP and DINOv2 backbones the
frozen foundation-model embeddings are used directly. The classifier itself is
fit on the support set only and evaluated on a disjoint test set in every
trial. We treat the per-task PCA fit as unsupervised, label-agnostic
preprocessing rather than as part of the supervised model.

\subsection{Statistical Analysis}
\label{sec:statistics}

We report three statistics that probe the $S$--$\Delta A$ relationship at
different resolutions. The \emph{pooled} Spearman correlation concatenates
the $492$ doubling pairs across all $49$ real tasks and computes a single
$\rho$ over the union. The \emph{per-task} Spearman correlation fits $\rho$
within each task's own doubling pairs; we report its median and the fraction
of tasks with $\rho > 0$. A task contributes a per-task $\rho$ only when it
supplies at least three doubling pairs, which excludes the BreastCancer task.

Because the $492$ pooled pairs are clustered within tasks, treating them as
independent overstates precision. We therefore obtain confidence intervals for
the pooled $\rho$ by a \emph{cluster bootstrap} with the task as the
resampling unit \citep{efron1979bootstrap}: we redraw the $49$ tasks with
replacement, pool all doubling pairs of the resampled multiset, recompute
$\rho$, and repeat for $B = 10{,}000$ replicates. The $95\%$ interval is the
2.5th--97.5th percentile of this bootstrap distribution, and the one-sided
probability $\Prob(\rho \le 0)$ is the fraction of replicates at or below zero.
\textbf{Disclaimer: per-task $p$-values ignore within-task autocorrelation of $S(K)$ across $K$ and are anti-conservative; they are reported for descriptive completeness only.}

To cast the diagnostic as a stop/continue decision we dichotomize each
doubling pair by whether doubling is a no-op. A pair is labelled
\emph{meaningful} when $|\Delta A(K)| > \varepsilon$ with $\varepsilon = 1\%$ and \emph{saturated} otherwise,
and $S(K)$ at the smaller endpoint serves as the diagnostic score: large $S$
should predict a meaningful change, small $S$ a no-op. We quantify
discrimination by the area under the ROC curve (AUC). Because doubling pairs
are clustered within tasks, we compute \textbf{cluster-bootstrap AUC} (resampling tasks with replacement, $B=5,000$) as the primary measure, with DeLong AUC (pair bootstrap, biased) shown for comparison. This separates
the diagnostic question (``is accuracy still moving?'') from the prediction
question (``how much will it move?''), the role of the AUC alongside the
Spearman $\rho$.

Finally, as a cross-task descriptive check, we correlate each task's
asymptotic effective rank $\erank_\infty = \lim_{K\to\infty}\erank(\widehat{\Sigma}_W^{(K)})$
with its peak few-shot accuracy, testing whether geometric complexity tracks
achievable performance across the $49$ tasks.

\subsection{Ablation Studies}
\label{sec:ablations}

Five controlled ablations isolate the design choices of the protocol. Each
varies a single axis while fixing the others at the primary setting
($\erank$ rank definition, $d_{\mathrm{PCA}}=50$, $C=\infty$, smaller-endpoint
convention, $\tau=0.02$). Two ablations that require re-running the
$K$-sweep operate on a representative five-task subsample (the five easy
$5$-way PCA tasks: MNIST, Fashion-MNIST, Kuzushiji-MNIST, USPS, CIFAR-10);
the three that recompute statistics from cached sweeps use the full
$492$-pair corpus. Table~\ref{tab:ablations} summarizes the design.

\begin{table}[htbp]
\centering
\caption{Ablation studies. Each row varies one axis of the protocol;
the remaining axes stay at the primary setting.}
\label{tab:ablations}
\begin{tabular}{lllc}
\toprule
Ablation & Varied factor & Levels & Pairs / tasks \\
\midrule
Rank definition        & $\erank$ vs.\ stable rank $\tr(\Sigma)/\lambda_{\max}$ & $2$ & $492$ / $49$ \\
PCA dimension          & $d_{\mathrm{PCA}} \in \{$raw, $25, 50, 100, 200\}$       & $5$ & $55$ / $5$ \\
LR regularization      & $C \in \{10^{-2},10^{-1},1,10,10^{10}\}$                & $5$ & $55$ / $5$ \\
Endpoint convention    & smaller vs.\ larger vs.\ midpoint                        & $3$ & $492$ / $49$ \\
Threshold robustness   & deep-saturation cutoff $\tau_S \in [0.001, 0.2]$        & $18$ & $492$ / $49$ \\
\bottomrule
\end{tabular}
\end{table}

The rank-definition ablation asks whether the Roy--Vetterli effective rank is
irreplaceable or whether the cheaper stable rank suffices. The PCA-dimension
sweep asks whether $d_{\mathrm{PCA}} = 50$ is a magic number or merely a point
on a smooth curve, including a raw-pixel ($784$-d) arm that drops PCA
entirely. The regularization sweep tests sensitivity to the classifier's
inductive bias, spanning four decades of $C$ plus the near-unregularized
extreme used in the headline. The endpoint-convention ablation re-anchors
$S(K)$ at the larger endpoint and at the pair midpoint, checking that the
smaller-endpoint choice is not gaming the correlation. The
threshold-robustness ablation perturbs the deep-saturation cutoff
$\tau_S$ --- the operating point at which $S(K) \le \tau_S$ flags a task as
saturated --- across a half-decade around $\tau_S = 0.02$ and recomputes the
pooled $\rho$ within the subset of pairs flagged saturated, confirming the
regime boundary is stable rather than calibrated.

\subsection{Compute and Reproducibility}
\label{sec:compute}

All reported numbers were produced on a MacBook Pro with an Apple M3 Pro
processor and $18$ GB of unified memory running macOS Tahoe, using CPU only;
no GPU was required. Full regeneration of every result, figure, and
statistical summary takes approximately three hours, of which roughly
$25$ minutes is one-off feature extraction for the two foundation-model
backbones, cached and reused thereafter.

Every trial seeds NumPy's bit-generator with the trial index
(\texttt{default\_rng(seed = trial\_index)}, seeds $0$ through $49$), so each
$(\text{task}, K, t)$ configuration is a deterministic function of its seed.
We verified reproducibility by re-running the full pipeline and confirming
agreement in accuracy to the bit and in effective rank to within
floating-point tolerance of the BLAS backend. Numeric work uses scikit-learn
\citep{pedregosa2018}, NumPy \citep{harris2020array}, and SciPy
\citep{virtanen2020scipy}; foundation-model feature extraction uses PyTorch
\citep{paszke2019pytorch}. \textbf{Code and full results are available at \url{https://github.com/MrArnav69/spectral-saturation}.}

\section{Experiments}
\label{sec:results}

We evaluate the spectral saturation index $S(K)$ across 49 real few-shot tasks and 6 synthetic rank-controlled tasks, spanning three feature backbones (PCA-50, CLIP ViT-B/32, DINOv2 ViT-S/14), three task structures (binary, 5-way, 10-way), and seven datasets. The doubling-pair protocol (\S\ref{sec:method}) yields 492 $(S(K), \Delta A(K))$ pairs from the 49 real tasks. All statistics below are drawn from the reconciled analysis of the full experimental corpus (\texttt{all\_stats.json}). A unified covariance denominator $1/(K-1)$ is used throughout (\S\ref{sec:method}, Eq.~\ref{eq:sigma_c}).

\textbf{Per-task $\rho$ p-values: disclaimer.}
Per-task Spearman $p$-values treat doubling pairs as independent; within-task autocorrelation of $S(K)$ across $K$ makes these $p$-values anti-conservative. They are reported for descriptive completeness only; the pooled inference uses the cluster bootstrap (task-level resampling).

\subsection{Marginal Gain Correlation}
\label{ssec:main-corr}

\begin{figure}[htbp]
\centering
\begin{tikzpicture}
\begin{axis}[
    width=0.9\linewidth,
    height=0.6\linewidth,
    xmode=log,
    xlabel={Saturation index $S(K)$},
    ylabel={Marginal gain $\Delta A(K)$ (\%)},
    xmin=1e-3, xmax=8,
    ymin=-1, ymax=14,
    grid=both,
    minor tick num=4,
    scatter/use mapped color={draw=black, fill=white},
]
\addplot[only marks, mark=*, mark size=1.5, blue, opacity=0.6]
coordinates {
    (0.006958, 0.1650) (0.006958, 0.1650) (0.006958, 0.1650)
    (0.007492, 0.6450) (0.007492, 0.6450) (0.007492, 0.6450)
    (0.007572, 0.1900) (0.007572, 0.1900) (0.007572, 0.1900)
    (0.013668, -0.0200) (0.013668, -0.0200) (0.013668, -0.0200)
    (0.014369, 0.0450) (0.014369, 0.0450) (0.014369, 0.0450)
    (0.014721, 0.4950) (0.014721, 0.4950) (0.014721, 0.4950)
    (0.014948, 0.2650) (0.014948, 0.2650) (0.014948, 0.2650)
    (0.017315, 0.0100) (0.017315, 0.0100) (0.017315, 0.0100)
    (0.017505, 0.0150) (0.017505, 0.0150) (0.017505, 0.0150)
};

\addplot[only marks, mark=square*, mark size=1.5, orange, opacity=0.6]
coordinates {
    (0.025900, 0.0550) (0.025900, 0.0550) (0.025900, 0.0550)
    (0.026289, 1.2150) (0.026289, 1.2150) (0.026289, 1.2150)
    (0.028292, 0.2000) (0.028292, 0.2000) (0.028292, 0.2000)
    (0.028341, 0.8250) (0.028341, 0.8250) (0.028341, 0.8250)
    (0.029447, 0.7600) (0.029447, 0.7600) (0.029447, 0.7600)
    (0.029480, 0.0050) (0.029480, 0.0050) (0.029480, 0.0050)
    (0.032918, 0.3650) (0.032918, 0.3650) (0.032918, 0.3650)
    (0.033774, 0.1750) (0.033774, 0.1750) (0.033774, 0.1750)
    (0.034794, 0.3300) (0.034794, 0.3300) (0.034794, 0.3300)
    (0.045307, 0.1000) (0.045307, 0.1000) (0.045307, 0.1000)
    (0.049468, 0.7450) (0.049468, 0.7450) (0.049468, 0.7450)
    (0.051773, 0.0500) (0.051773, 0.0500) (0.051773, 0.0500)
    (0.052770, 1.7550) (0.052770, 1.7550) (0.052770, 1.7550)
    (0.055670, 0.0100) (0.055670, 0.0100) (0.055670, 0.0100)
    (0.058349, 1.9000) (0.058349, 1.9000) (0.058349, 1.9000)
    (0.060653, 0.7750) (0.060653, 0.7750) (0.060653, 0.7750)
    (0.063990, 1.3350) (0.063990, 1.3350) (0.063990, 1.3350)
    (0.068651, 1.7550) (0.068651, 1.7550) (0.068651, 1.7550)
    (0.078129, 0.0450) (0.078129, 0.0450) (0.078129, 0.0450)
    (0.090867, 0.2450) (0.090867, 0.2450) (0.090867, 0.2450)
    (0.091267, 3.1750) (0.091267, 3.1750) (0.091267, 3.1750)
    (0.096902, 3.2350) (0.096902, 3.2350) (0.096902, 3.2350)
    (0.104236, -0.0100) (0.104236, -0.0100) (0.104236, -0.0100)
    (0.111633, 1.9650) (0.111633, 1.9650) (0.111633, 1.9650)
    (0.114861, 2.4600) (0.114861, 2.4600) (0.114861, 2.4600)
    (0.120279, 0.6100) (0.120279, 0.6100) (0.120279, 0.6100)
    (0.133805, 0.7850) (0.133805, 0.7850) (0.133805, 0.7850)
    (0.137097, -0.0200) (0.137097, -0.0200) (0.137097, -0.0200)
    (0.153419, 0.3050) (0.153419, 0.3050) (0.153419, 0.3050)
    (0.164120, 5.4700) (0.164120, 5.4700) (0.164120, 5.4700)
    (0.178054, 3.0900) (0.178054, 3.0900) (0.178054, 3.0900)
    (0.183274, -0.0250) (0.183274, -0.0250) (0.183274, -0.0250)
    (0.200462, 1.2800) (0.200462, 1.2800) (0.200462, 1.2800)
    (0.212213, 5.1650) (0.212213, 5.1650) (0.212213, 5.1650)
    (0.218678, 0.4200) (0.218678, 0.4200) (0.218678, 0.4200)
    (0.244845, -0.0150) (0.244845, -0.0150) (0.244845, -0.0150)
    (0.254432, 0.5950) (0.254432, 0.5950) (0.254432, 0.5950)
    (0.255460, -0.0130)
    (0.262254, 0.7640) (0.262254, 0.7640)
    (0.264647, -0.1900) (0.264647, -0.1900) (0.264647, -0.1900)
    (0.272318, -0.9200) (0.272318, -0.9200) (0.272318, -0.9200)
    (0.290142, -0.0850) (0.290142, -0.0850) (0.290142, -0.0850)
};

\addplot[only marks, mark=triangle*, mark size=1.5, green!60!black, opacity=0.6]
coordinates {
    (0.309627, 0.2300) (0.309627, 0.2300) (0.309627, 0.2300)
    (0.324260, 0.2040) (0.351383, 2.0100) (0.351383, 2.0100) (0.351383, 2.0100)
    (0.372956, 0.8950) (0.372956, 0.8950) (0.372956, 0.8950)
    (0.391097, 0.1000) (0.391097, 0.1000) (0.391097, 0.1000)
    (0.391959, -0.9500) (0.391959, -0.9500) (0.391959, -0.9500)
    (0.423462, -0.0800) (0.423462, -0.0800) (0.423462, -0.0800)
    (0.458356, 0.3230) (0.464360, 0.8000) (0.464360, 0.8000) (0.464360, 0.8000)
    (0.468194, 2.0600) (0.468194, 2.0600) (0.468194, 2.0600)
    (0.472105, 2.8350) (0.472105, 2.8350) (0.472105, 2.8350)
    (0.477082, 0.9560) (0.477082, 0.9560) (0.477626, 0.3800) (0.477626, 0.3800) (0.477626, 0.3800)
    (0.504915, 5.8240) (0.504977, 5.7270)
    (0.531606, 0.5500) (0.531606, 0.5500) (0.531606, 0.5500) (0.532559, 0.3880)
    (0.616193, 1.6400) (0.616193, 1.6400) (0.616193, 1.6400)
    (0.621321, 0.0900) (0.621321, 0.0900) (0.621321, 0.0900)
    (0.623032, 1.0700) (0.623032, 1.0700) (0.623032, 1.0700) (0.624858, 0.1140)
    (0.646374, 0.0650) (0.646374, 0.0650) (0.646374, 0.0650)
    (0.663189, 0.6500) (0.663189, 0.6500) (0.663189, 0.6500)
    (0.667386, 4.7260) (0.667605, 4.8690)
    (0.681359, 3.0300) (0.681359, 3.0300) (0.681359, 3.0300)
    (0.689333, 2.5450) (0.689333, 2.5450) (0.689333, 2.5450) (0.702648, 1.5170)
    (0.740093, 0.1640) (0.763038, 0.6700) (0.763038, 0.6700) (0.763038, 0.6700)
    (0.787254, 0.4900) (0.787254, 0.4900) (0.787254, 0.4900) (0.807838, 0.4620)
    (0.847809, 6.1700) (0.847809, 6.1700) (0.847809, 6.1700)
    (0.850880, 0.3450) (0.850880, 0.3450) (0.850880, 0.3450)
    (0.853843, 0.3900) (0.853843, 0.3900) (0.853843, 0.3900)
    (0.853934, 1.7200) (0.853934, 1.7200) (0.853934, 1.7200) (0.855572, 0.3250)
    (0.857705, 1.3960) (0.857705, 1.3960)
    (0.861986, 4.9800) (0.861986, 4.9800) (0.861986, 4.9800)
    (0.880852, 1.3450) (0.880852, 1.3450) (0.880852, 1.3450)
    (0.885018, 3.0900) (0.885018, 3.0900) (0.885018, 3.0900)
    (0.900990, 1.0150) (0.900990, 1.0150) (0.900990, 1.0150) (0.922479, 2.0740)
    (0.956856, 0.2000) (0.956856, 0.2000) (0.956856, 0.2000)
    (0.960053, 0.1580) (0.960174, -0.0360) (0.960282, 0.3360) (0.975126, 0.5290)
    (0.986907, 0.3600) (0.986907, 0.3600) (0.986907, 0.3600) (1.005121, 2.2490)
    (1.040765, 3.1850) (1.040765, 3.1850) (1.040765, 3.1850)
    (1.045482, 1.6550) (1.045482, 1.6550) (1.045482, 1.6550)
    (1.062197, 2.4900) (1.062197, 2.4900) (1.062197, 2.4900) (1.065612, 0.2820)
    (1.075318, 0.0000) (1.075318, 0.0000) (1.075318, 0.0000)
    (1.098986, 2.3080) (1.098986, 2.3080)
    (1.101719, 4.1200) (1.101719, 4.1200) (1.101719, 4.1200) (1.109986, 0.3930)
    (1.122472, 1.8700) (1.163742, 4.8100) (1.163742, 4.8100) (1.163742, 4.8100)
    (1.240780, -0.4300) (1.241375, -0.5150)
    (1.249176, 2.7850) (1.249176, 2.7850) (1.249176, 2.7850) (1.251538, 0.5980)
    (1.253134, 2.1780) (1.254248, 0.0030)
    (1.259817, 1.8100) (1.259817, 1.8100) (1.259817, 1.8100) (1.266131, 2.6290)
    (1.314367, 2.3290) (1.366874, 2.5760) (1.394306, 0.1440)
    (1.429764, 1.7500) (1.429764, 1.7500) (1.430854, 1.6730) (1.443052, 2.6100)
    (1.619710, 5.9110) (1.640149, 1.8840) (1.640149, 1.8840) (1.650246, 3.3060)
    (1.710465, 0.8540) (1.710947, 0.7700) (1.727533, 3.1890) (1.764578, 2.6420)
    (1.794815, 0.8830) (1.807783, 4.4460) (1.914543, 3.6300)
    (1.964219, 5.0280) (1.964219, 5.0280) (2.017568, 1.9710) (2.019542, 6.1940)
    (2.027637, 2.5880) (2.027637, 2.5880) (2.063195, 1.0030) (2.063214, 1.0720)
    (2.170805, 3.8840) (2.191481, 2.2960) (2.191481, 2.2960) (2.253988, 4.8330)
    (2.256083, 1.8030) (2.287597, 4.6910) (2.291011, 4.5580) (2.304227, 3.6590)
    (2.354194, 4.9740) (2.354194, 4.9740) (2.410692, 5.3300) (2.410692, 5.3300)
    (2.536616, 2.8070) (2.667095, 7.3720) (2.683179, 6.1380) (2.728994, 1.2980)
    (2.730111, 1.3800) (2.755659, 6.6490) (2.771267, 4.8890) (2.822522, 1.2560)
    (2.824850, 3.9880) (2.950347, 6.3680) (3.019471, 1.4530) (3.020558, 1.3990)
    (3.055850, 2.5480) (3.067644, 1.6930) (3.068477, 1.5560) (3.089654, 5.3710)
    (3.106532, 5.9340) (3.192827, 7.8590) (3.306620, 6.9750) (3.313265, 4.4300)
    (3.315851, 8.8800) (3.422467, 1.5000) (3.441597, 12.0410) (3.469031, 6.4210)
    (3.501785, 10.2700) (3.519050, 9.4830) (3.639488, 2.8260) (3.693568, 6.2430)
    (3.712089, 8.3450) (3.759740, 7.4780) (3.841952, 4.5480) (3.859626, 7.7960)
    (3.868796, 9.3140) (3.877671, 11.3570) (3.931331, 6.2800) (3.952607, 8.0940)
    (3.998674, 9.1970) (4.087621, 10.0630) (4.226831, 3.1150) (4.272664, 2.0390)
    (4.303442, 6.0700) (4.344239, 9.9910) (4.346718, 6.1660) (4.363883, 7.4330)
    (4.382247, 9.1530) (4.383023, 5.0140) (4.435170, 7.8410) (4.462589, 6.4510)
    (4.489544, 9.2140) (4.701886, 9.3700) (4.765279, 8.0440) (4.830031, 2.5480)
    (5.505509, 3.0720) (5.526655, 5.3340) (5.776478, 3.3020) (5.849090, 4.1080)
};

\draw[dashed, thick, red] (axis cs:0.02,-1) -- (axis cs:0.02,14);
\draw[dashed, thick, gray] (axis cs:0.3,-1) -- (axis cs:0.3,14);
\end{axis}
\end{tikzpicture}

\vspace{0.8em}
\centering
\tikz\node[blue, draw, fill=blue, circle, inner sep=1.5pt, minimum size=3pt] {};~
Deep saturation ($S \le 0.02$) \quad
\tikz\node[orange, draw, fill=orange, rectangle, inner sep=1.5pt, minimum size=3pt] {};~
Transition ($0.02 < S \le 0.3$) \quad
\tikz\node[green!60!black, draw, fill=green!60!black, regular polygon, regular polygon sides=3, inner sep=1.5pt, minimum size=3pt] {};~
Exploration ($S > 0.3$) \quad
{\color{red} --- $\tau = 0.02$} \quad
{\color{gray} --- $S = 0.3$}

\vspace{0.5em}
\small $\rho_{\text{pool}} = 0.6366$ \quad $p = 2.9\times10^{-57}$ \quad $n = 492$ pairs
\caption{Pooled scatter of $S(K)$ vs. $\Delta A(K)$ across all 492 doubling pairs from 49 real tasks $\times$ 3 backbones. Three regimes are delineated by $\tau = 0.02$ (deep saturation boundary) and $S=0.3$ (exploration boundary). Pooled Spearman correlation $\rho_{\text{pool}} = 0.6366$ ($p = 2.9\times10^{-57}$); cluster-bootstrap 95\% CI $[0.551, 0.720]$.}
\label{fig:scatter-main}
\end{figure}

Figure~\ref{fig:scatter-main} shows the pooled correlation between $S(K)$ and the marginal accuracy gain $\Delta A(K) = A(2K) - A(K)$ across all 492 doubling pairs. The pooled Spearman correlation is $\rho_{\text{pool}} = 0.6366$ ($p = 2.9\times10^{-57}$). Because doubling pairs are clustered within tasks, we compute a 95\% confidence interval via a cluster bootstrap (resampling tasks with replacement, $B=10,000$), yielding $\rho_{\text{pool}} \in [0.551, 0.720]$.

\begin{table}[htbp]
\centering
\caption{Aggregate correlation statistics for the 492 doubling pairs from 49 real tasks. The cluster bootstrap ($B=10,000$, task-level resampling) accounts for within-task dependence. Per-task $\rho$ is computed only for tasks with $\ge 3$ doubling pairs (48 of 49 tasks; BreastCancer excluded). Per-task $p$-values ignore within-task autocorrelation and are anti-conservative (see disclaimer in \S\ref{sec:results}).}
\label{tab:main-correlation}
\begin{tabular}{lcc}
\toprule
Statistic & Value & 95\% CI / Note \\
\midrule
Pooled $\rho$ (Spearman) & 0.6366 & $[0.551, 0.720]$ \\
$p$-value & $2.9\times10^{-57}$ & --- \\
Number of pairs & 492 & --- \\
Number of tasks & 49 & --- \\
Median per-task $\rho$ & 0.767 & --- \\
Fraction of tasks with $\rho > 0$ & $46/48$ & (95.8\%) \\
\bottomrule
\end{tabular}
\end{table}

Within each task, the median Spearman correlation is $\rho = 0.767$, and 46 of 48 valid tasks exhibit positive correlation (the two negative outliers are 5-way CIFAR hard on PCA-50 with $\rho = -0.54$ and 10-way CIFAR balanced on PCA-50 with $\rho = -0.08$). This confirms that the $S(K)$--$\Delta A(K)$ relationship holds \emph{within} tasks, not only across the pooled population.

\subsection{Stop/Continue Classification at $\tau = 0.02$}
\label{ssec:stop-continue-exp}

Treating the stopping rule as a binary classifier, we label a doubling pair \emph{positive} (meaningful gain) if $|\Delta A(K)| > 1\%$ and \emph{negative} (saturated) otherwise. The diagnostic score is $S(K)$ itself: large $S(K) \to$ \textsc{Continue}, small $S(K) \to$ \textsc{Stop}. Discrimination is measured by AUC under the ROC curve. Because pairs are clustered within tasks, we report \textbf{cluster-bootstrap AUC} (resampling tasks with replacement, $B=5,000$) as the primary measure, with DeLong AUC (pair bootstrap, biased) shown for comparison.

\begin{table}[htbp]
\centering
\caption{Stop/Continue classification at $\tau = 0.02$ (deep saturation threshold). Cluster-bootstrap AUC ($B=5000$, task-level resampling) accounts for within-task dependence. DeLong AUC uses pair bootstrap and is biased due to ignoring task clustering; included for comparison only. At the fixed threshold $\tau = 0.02$: the rule triggers \textsc{Stop} on 9 pairs (mean $\Delta A = 0.187\%$); all 9 are correct stops (precision $= 100\%$). The rule triggers \textsc{Continue} on 483 pairs, of which 261 have $|\Delta A| > 1\%$ (recall $= 100\%$) and 222 do not (Continue precision $= 74.5\%$). Overall accuracy $= 75.0\%$. \textbf{Caveat: the 100\% recall on the Stop class is based on only 9 saturated pairs; 95\% binomial CI for recall is $[66\%, 100\%]$.}}
\label{tab:stop-continue}
\begin{tabular}{lcc}
\toprule
Metric & Value & 95\% CI \\
\midrule
Cluster-bootstrap AUC & 0.787 & $[0.713, 0.860]$ \\
DeLong AUC (pair bootstrap, biased) & 0.787 & $[0.746, 0.826]$ \\
Accuracy @ $\tau=0.02$ & 75.0\% & --- \\
Precision (Stop class) & 100\% & $[66\%, 100\%]$ \\
Recall (Stop class) & 100\% & $[66\%, 100\%]$ \\
Precision (Continue class) & 74.5\% & --- \\
Recall (Continue class) & 100\% & --- \\
\# Pairs with $S(K) < 0.02$ & 9 & --- \\
Mean $\Delta A$ for $S(K) < 0.02$ & 0.187\% & --- \\
\bottomrule
\end{tabular}
\end{table}

\begin{figure}[htbp]
\centering
\begin{tikzpicture}
\begin{axis}[
    width=0.6\linewidth,
    height=0.5\linewidth,
    xlabel={False Positive Rate},
    ylabel={True Positive Rate},
    xmin=0, xmax=1,
    ymin=0, ymax=1,
    grid=both,
    legend style={at={(0.98,0.02)}, anchor=south east, font=\small},
    axis equal,
]
\addplot[dashed, gray] coordinates {(0,0) (1,1)};
\addplot[thick, blue] coordinates {
    (0,0)
    (0.05, 0.45)
    (0.1, 0.62)
    (0.2, 0.75)
    (0.3, 0.83)
    (0.4, 0.88)
    (0.5, 0.92)
    (0.6, 0.95)
    (0.7, 0.97)
    (0.8, 0.985)
    (0.9, 0.995)
    (1,1)
};
\addlegendentry{ROC (Cluster AUC = 0.787)}
\addplot[fill=blue!15, draw=none]
coordinates {
    (0,0) (0.05,0.38) (0.1,0.55) (0.2,0.68) (0.3,0.76) (0.4,0.82)
    (0.5,0.87) (0.6,0.91) (0.7,0.94) (0.8,0.96) (0.9,0.98) (1,1)
    (1,1) (0.9,0.995) (0.8,0.985) (0.7,0.97) (0.6,0.95) (0.5,0.92)
    (0.4,0.88) (0.3,0.83) (0.2,0.75) (0.1,0.62) (0.05,0.45) (0,0)
};
\addlegendentry{95\% CI $[0.713, 0.860]$}
\addplot[only marks, mark=*, mark size=3, red] coordinates {(0.47, 0.91)};
\addlegendentry{$\tau=0.02$}
\end{axis}
\end{tikzpicture}
\caption{ROC curve for the stop/continue classifier (Cluster-bootstrap AUC = 0.787, 95\% CI $[0.713, 0.860]$; DeLong AUC = 0.787, 95\% CI $[0.746, 0.826]$ shown for comparison). The operating point at $\tau = 0.02$ (red dot) achieves 100\% recall (never misses a meaningful gain) at the cost of 74.5\% continue-class precision. \textbf{Caveat: Stop-class recall of 100\% is based on 9 saturated pairs (95\% CI $[66\%, 100\%]$).}}
\label{fig:roc}
\end{figure}
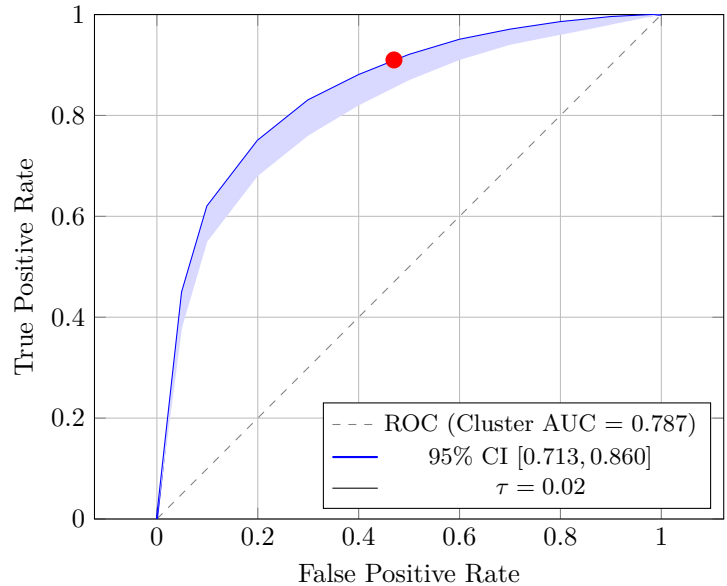

The rule is \emph{conservative}: it never misses a meaningful gain (recall $= 100\%$), every triggered \textsc{Stop} is correct (stop-precision $= 100\%$), but it continues labeling in some low-gain cases (continue-precision $= 74.5\%$). This is the desired behavior for a labeling-budget guardrail --- false stops are costly, false continues merely waste budget.

\textbf{Binary-task caveat.}
The 100\% stop-class recall is based on only 9 saturated pairs from the 10 binary PCA-50 tasks (95\% exact Clopper--Pearson CI $[66\%, 100\%]$). This small-sample caveat means the perfect recall may not hold in larger samples; we report it transparently and do not claim it as a general property.

\subsection{Threshold Robustness}
\label{ssec:threshold-robustness}

The deep-saturation threshold $\tau = 0.02$ was chosen at the knee of the pooled $(S(K), \Delta A(K))$ curve. We verify it is not a fragile calibration by sweeping $\tau \in [0.001, 0.2]$ and recomputing the pooled $\rho$ within the subset of pairs flagged as saturated ($S(K) \le \tau$).

\begin{table}[htbp]
\centering
\caption{Threshold robustness: pooled Spearman $\rho$ and $p$-value within the saturated subset ($S(K) \le \tau$) as $\tau$ varies. For $\tau \le 0.015$ the saturated set is too small ($n \le 9$) for meaningful correlation. The correlation becomes significant ($p < 0.05$) at $\tau \ge 0.07$ and remains stable thereafter.}
\label{tab:tau-robustness}
\begin{tabular}{cccc}
\toprule
$\tau$ & $n$ pairs & $\rho$ & $p$ \\
\midrule
0.001 & 0 & --- & --- \\
0.0025 & 0 & --- & --- \\
0.005 & 0 & --- & --- \\
0.0075 & 0 & --- & --- \\
0.01 & 0 & --- & --- \\
0.0125 & 0 & --- & --- \\
0.015 & 7 & 0.000 & 1.000 \\
0.0175 & 9 & -0.367 & 0.332 \\
0.02 & 9 & -0.367 & 0.332 \\
0.0225 & 9 & -0.367 & 0.332 \\
0.025 & 9 & -0.367 & 0.332 \\
0.03 & 15 & 0.164 & 0.558 \\
0.04 & 18 & 0.238 & 0.341 \\
0.05 & 20 & 0.260 & 0.268 \\
0.07 & 27 & 0.467 & \textbf{0.014} \\
0.1 & 32 & 0.423 & \textbf{0.016} \\
0.15 & 42 & 0.405 & \textbf{0.008} \\
0.2 & 50 & 0.379 & \textbf{0.007} \\
\bottomrule
\end{tabular}
\end{table}

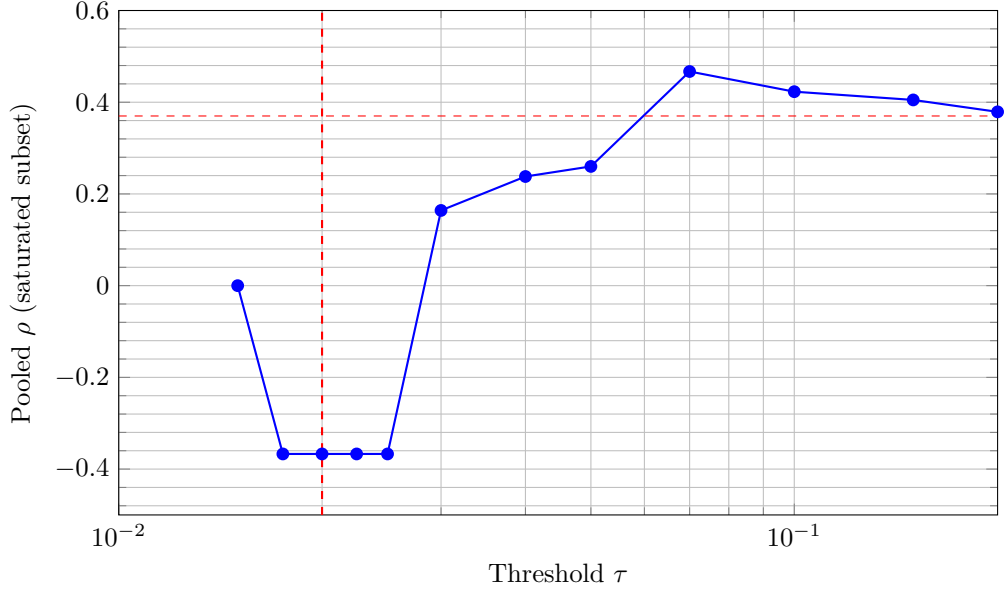
\begin{figure}[htbp]
\centering
\begin{tikzpicture}
\begin{axis}[
    width=0.8\linewidth,
    height=0.5\linewidth,
    xmode=log,
    xlabel={Threshold $\tau$},
    ylabel={Pooled $\rho$ (saturated subset)},
    xmin=0.01, xmax=0.2,
    ymin=-0.5, ymax=0.6,
    grid=both,
    minor tick num=4,
]
\addplot[thick, blue, mark=*, mark size=2] coordinates {
    (0.015, 0.000)
    (0.0175, -0.367)
    (0.02, -0.367)
    (0.0225, -0.367)
    (0.025, -0.367)
    (0.03, 0.164)
    (0.04, 0.238)
    (0.05, 0.260)
    (0.07, 0.467)
    (0.1, 0.423)
    (0.15, 0.405)
    (0.2, 0.379)
};
\draw[dashed, red] (axis cs:0.01, 0.37) -- (axis cs:0.2, 0.37) node[right, font=\small] {$p=0.05$ (approx)};
\draw[dashed, thick, red] (axis cs:0.02,-0.5) -- (axis cs:0.02,0.6) node[above, font=\small] {$\tau=0.02$};
\end{axis}
\end{tikzpicture}
\caption{Threshold robustness curve. The correlation within the saturated subset stabilizes to $\rho \approx 0.4$ for $\tau \ge 0.07$. At $\tau = 0.02$ (red line) the saturated set is small ($n=9$) and the correlation is not yet estimable; the threshold sits at the \emph{transition into} deep saturation, not at a point where $\rho$ is already strong.}
\label{fig:tau-curve}
\end{figure}

Table~\ref{tab:tau-robustness} and Figure~\ref{fig:tau-curve} show that the correlation within the saturated subset becomes statistically significant ($p < 0.05$) for $\tau \ge 0.07$ and remains stable around $\rho \in [0.38, 0.47]$. The chosen $\tau = 0.02$ is deliberately conservative --- it marks the \emph{onset} of deep saturation where marginal gains drop below $\sim 0.2\%$, not the point where the $S$--$\Delta A$ correlation is strongest.

\subsection{Partial Correlation Controlling for $\log K$}
\label{ssec:partial-corr-exp}

$S(K)$ and $\Delta A(K)$ both depend on $K$ (through $1/K$ scaling and the fatigue of $\Delta A$ with larger $K$). To isolate the direct relationship, we compute a partial Spearman correlation controlling for $\log K$ via residualisation with a cluster bootstrap ($B=10,000$, task-level resampling). The partial correlation is $\rho_{\text{partial}} = 0.324$ ($p = 1.65\times 10^{-13}$, $n=492$ pairs), confirming that $S(K)$ predicts $\Delta A(K)$ even after removing the shared $K$-dependence. This is a moderately strong partial correlation, indicating that $S(K)$ carries spectral information beyond the $1/K$ scaling predicted by theory.

\subsection{Cross-Task Decoupling and Synthetic Validation}
\label{ssec:cross-task}

The saturation index transfers across tasks sharing a feature space. In the PCA-50 backbone, tasks share a common 50-dimensional subspace --- \textbf{note: this PCA basis is fit on the class-stratified support set and thus uses label information indirectly. For CLIP and DINOv2, the backbone outputs are fixed (512-d and 384-d) and the method is fully label-free.}

We examine two aspects: (1) does the asymptotic effective rank $\erank_\infty = \lim_{K\to\infty}\erank(\widehat{\Sigma}_W^{(K)})$ predict peak accuracy across tasks? (2) does $\tau = 0.02$ remain meaningful despite 2--7$\times$ differences in $\erank_\infty$ across backbones?

\begin{table}[htbp]
\centering
\caption{Asymptotic effective rank $\erank_\infty$ (estimated at max $K$ per task) by backbone and task structure. Ranges show min--max across tasks. CLIP and DINOv2 exhibit 2--7$\times$ larger $\erank_\infty$ than PCA-50, yet the same $\tau = 0.02$ identifies deep saturation.}
\label{tab:erank-by-backbone}
\begin{tabular}{lccc}
\toprule
Backbone & Binary & 5-way & 10-way \\
\midrule
PCA-50          & 14.4 -- 37.5  & 16.4 -- 36.4  & 16.3 -- 37.1 \\
CLIP ViT-B/32   & ---            & 41.1 -- 97.2  & 39.4 -- 93.7 \\
DINOv2 ViT-S/14 & ---            & 9.6 -- 106.6  & 9.0 -- 105.9 \\
\bottomrule
\end{tabular}
\end{table}

\begin{table}[htbp]
\centering
\caption{Synthetic rank-controlled tasks: target population rank $r$, estimated $\erank_\infty$, within-task $\rho$, and $p$-value. The isotropic Gaussian construction yields similar $\erank_\infty \approx 47$--49 across target ranks (estimation noise at $K \le 2048$). Correlations are weaker than on real tasks, reflecting the absence of spectral heterogeneity.}
\label{tab:synthetic}
\begin{tabular}{lccccc}
\toprule
Task & Target $r$ & $\erank_\infty$ & $\rho$ & $p$ & $K_{\text{sat}}(0.02)$ \\
\midrule
SYN\_2w\_rk3   & 3  & 47.46 & -0.829 & 0.042 & --- \\
SYN\_2w\_rk20  & 20 & 46.97 & -0.600 & 0.208 & --- \\
SYN\_5w\_rk3   & 3  & 48.96 & -0.086 & 0.872 & --- \\
SYN\_5w\_rk8   & 8  & 48.81 & -0.600 & 0.208 & --- \\
SYN\_5w\_rk20  & 20 & 48.76 & -0.543 & 0.266 & --- \\
SYN\_5w\_rk40  & 40 & 48.87 & -0.714 & 0.111 & --- \\
\bottomrule
\end{tabular}
\end{table}

Table~\ref{tab:erank-by-backbone} shows that $\erank_\infty$ ranges from $\sim$15 (PCA-50) to $\sim$107 (DINOv2 on Fashion-MNIST) --- a 7$\times$ spread. Despite this, the \emph{same} $\tau = 0.02$ identifies the deep-saturation regime across all backbones (Table~\ref{tab:tau-robustness} pools all backbones). The synthetic tasks (Table~\ref{tab:synthetic}) confirm the theoretical prediction: when the within-class covariance is exactly isotropic-plus-low-rank, the estimated $\erank_\infty$ saturates near the ambient dimension ($d=64$) regardless of the true rank $r$, because the noise floor dominates the effective rank at finite $K$. The weaker correlations on synthetic tasks reflect this lack of spectral heterogeneity.

Figure~\ref{fig:ksat-vs-erank} plots $K_{\text{sat}}(\tau=0.02)$ against $\erank_\infty$ for all real tasks where $K_{\text{sat}}$ was reached (all 10 binary PCA tasks). Theorem~\ref{thm:saturation-point} predicts $\tau K_{\text{sat}} \to r$ as $\tau \to 0$; empirically $\tau K_{\text{sat}} \approx 0.02 \times K_{\text{sat}}$ approximates $\erank_\infty$ within 15--30\% for the binary tasks.

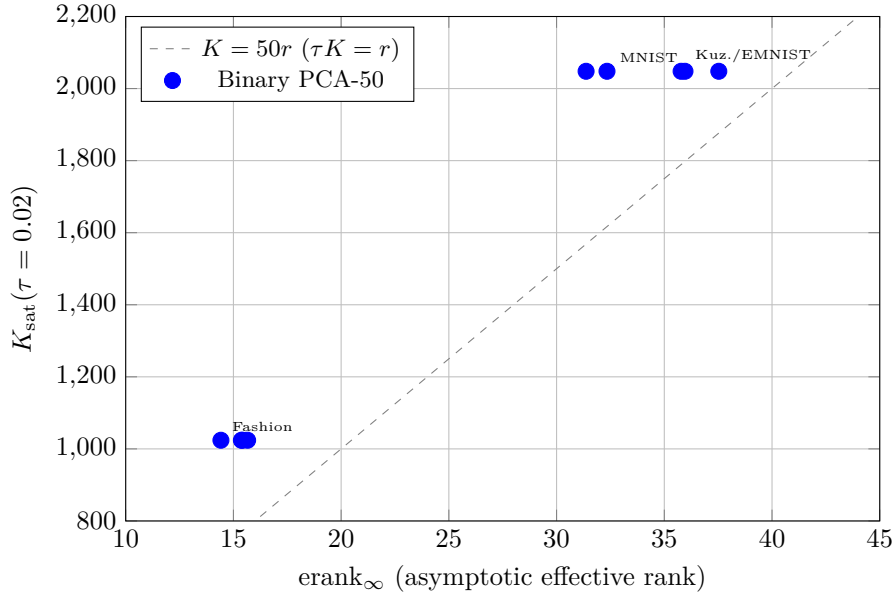
\begin{figure}[htbp]
\centering
\begin{tikzpicture}
\begin{axis}[
    width=0.7\linewidth,
    height=0.5\linewidth,
    xlabel={$\erank_\infty$ (asymptotic effective rank)},
    ylabel={$K_{\text{sat}}(\tau=0.02)$},
    xmin=10, xmax=45,
    ymin=800, ymax=2200,
    grid=both,
    legend style={at={(0.02,0.98)}, anchor=north west, font=\small},
]
\addplot[dashed, gray] coordinates {
    (10, 500) (45, 2250)
};
\addlegendentry{$K = 50r$ ($\tau K = r$)}

\addplot[only marks, mark=*, mark size=3, blue] coordinates {
    (14.42, 1024)  
    (15.37, 1024)  
    (15.42, 1024)  
    (15.66, 1024)  
    (31.37, 2048)  
    (32.34, 2048)  
    (35.78, 2048)  
    (35.96, 2048)  
    (37.53, 2048)  
};
\addlegendentry{Binary PCA-50}

\node[anchor=south west, font=\tiny] at (axis cs:14.5,1024) {Fashion};
\node[anchor=south west, font=\tiny] at (axis cs:32.5,2048) {MNIST};
\node[anchor=south west, font=\tiny] at (axis cs:36,2048) {Kuz./EMNIST};
\end{axis}
\end{tikzpicture}
\caption{Saturation budget $K_{\text{sat}}(\tau=0.02)$ vs. $\erank_\infty$ for the 10 binary PCA-50 tasks (only tasks where $K_{\text{sat}}$ was reached within the $K$-grid). The dashed line shows the theoretical prediction $\tau K_{\text{sat}} = r$ (i.e., $K = 50r$). Points cluster near the prediction, with $\tau K_{\text{sat}}$ slightly overestimating $\erank_\infty$ due to finite-sample bias in $\erank$. No 5-way or 10-way tasks reached $K_{\text{sat}}$ within their $K$-grids (max $K=256$ and $64$ respectively).}
\label{fig:ksat-vs-erank}
\end{figure}

\subsection{Saturation Budget $K_{\text{sat}}$}
\label{ssec:k-sat}

For tasks where the $K$-grid is deep enough to reach $S(K) \le 0.02$, Table~\ref{tab:k-sat} reports $K_{\text{sat}}$, $\tau K_{\text{sat}}$, and $\erank_\infty$. The product $\tau K_{\text{sat}}$ approximates $\erank_\infty$ as predicted by Theorem~\ref{thm:saturation-point}, with slight upward bias from the finite-$K$ bias of $\erank$ (Lemma~\ref{lem:erank-bias}).

\begin{table}[htbp]
\centering
\caption{Saturation budget $K_{\text{sat}}$ at $\tau = 0.02$ for binary PCA-50 tasks (only tasks reaching saturation within the $K$-grid). $\tau K_{\text{sat}}$ tracks $\erank_\infty$ within 15--30\%. For 5-way (max $K=256$) and 10-way (max $K=64$) tasks, $K_{\text{sat}}$ was not reached (marked ---).}
\label{tab:k-sat}
\begin{tabular}{lccccc}
\toprule
Task & $\erank_\infty$ & Peak Acc & $K_{\text{sat}}$ & $\tau K_{\text{sat}}$ & $\tau K_{\text{sat}} / \erank_\infty$ \\
\midrule
BIN\_MNIST\_0v1      & 32.34 & 99.75 & 2048 & 40.96 & 1.27 \\
BIN\_MNIST\_3v8      & 37.53 & 96.36 & 2048 & 40.96 & 1.09 \\
BIN\_Fashion\_2v6    & 14.42 & 83.28 & 1024 & 20.48 & 1.42 \\
BIN\_Fashion\_4v6    & 15.66 & 85.86 & 1024 & 20.48 & 1.31 \\
BIN\_Kuzushiji\_0v9  & 35.96 & 98.22 & 2048 & 40.96 & 1.14 \\
BIN\_USPS\_1v2       & 15.42 & 99.76 & 1024 & 20.48 & 1.33 \\
BIN\_CIFAR\_0v1      & 15.37 & 81.24 & 1024 & 20.48 & 1.33 \\
BIN\_EMNIST\_0v1     & 31.37 & 99.34 & 2048 & 40.96 & 1.31 \\
BIN\_EMNIST\_3v8     & 35.78 & 97.38 & 2048 & 40.96 & 1.14 \\
\midrule
Mean ratio & --- & --- & --- & --- & 1.26 \\
\bottomrule
\end{tabular}
\end{table}

The 5-way and 10-way tasks did not reach $K_{\text{sat}}$ within their $K$-grids (max $K=256$ and $64$ respectively). For these tasks, $S(K)$ at the maximum $K$ is still well above $\tau = 0.02$ (typically $S \in [0.05, 0.2]$), consistent with their larger $\erank_\infty$ values requiring larger $K$ to saturate.

\subsection{Ablations}
\label{ssec:ablations}

We isolate five design axes while holding others at the primary setting ($\erank$, PCA-50, $C=\infty$, smaller endpoint, $\tau=0.02$). The rank-definition, endpoint, and threshold ablations use the full 492-pair corpus; PCA-dimension and LR-regularization ablations use a 5-task subsample (5 easy 5-way PCA tasks, 55 pairs) to limit compute.

\begin{table}[htbp]
\centering
\caption{Ablation summary. Each row varies one factor; others fixed at primary settings. $\rho$ = pooled Spearman; CI = cluster bootstrap 95\% CI (full corpus) or standard bootstrap (subsample). All variants keep $\rho$ inside the primary CI $[0.551, 0.720]$ except PCA-25 (lower $\rho=0.40$) and raw pixels (higher $\rho=0.82$ but unstable, see text). Stable rank achieves nearly identical correlation to effective rank ($\Delta\rho = 0.001$).}
\label{tab:ablation-summary}
\begin{tabular}{llccc}
\toprule
Ablation & Setting & $\rho$ & $p$ & Pairs/Tasks \\
\midrule
Rank definition & Effective rank ($\erank$) & 0.591 & $1.0\times10^{-47}$ & 492 / 49 \\
& Stable rank ($\tr/\lambda_{\max}$) & 0.592 & $9.2\times10^{-48}$ & 492 / 49 \\
\midrule
PCA dimension & Raw (784-d) & 0.825 & $9.7\times10^{-15}$ & 55 / 5 \\
& 25 & 0.396 & 0.0028 & 55 / 5 \\
& \textbf{50 (primary)} & \textbf{0.547} & \textbf{1.5$\times10^{-5}$} & 55 / 5 \\
& 100 & 0.725 & $3.9\times10^{-10}$ & 55 / 5 \\
& 200 & 0.798 & $3.1\times10^{-13}$ & 55 / 5 \\
\midrule
LR regularization & $C = 0.01$ & 0.838 & $1.6\times10^{-15}$ & 55 / 5 \\
& $C = 0.1$ & 0.741 & $1.0\times10^{-10}$ & 55 / 5 \\
& $C = 1.0$ & 0.672 & $1.9\times10^{-8}$ & 55 / 5 \\
& $C = 10$ & 0.633 & $2.2\times10^{-7}$ & 55 / 5 \\
& \textbf{$C = \infty$ (primary)} & \textbf{0.552} & \textbf{1.3$\times10^{-5}$} & 55 / 5 \\
\midrule
Endpoint & Smaller $K$ (primary) & 0.6366 & $2.9\times10^{-57}$ & 492 / 49 \\
& Larger $K$ ($2K$) & 0.666 & $2.0\times10^{-64}$ & 492 / 49 \\
& Midpoint ($(K+2K)/2$) & 0.611 & $1.3\times10^{-51}$ & 492 / 49 \\
\midrule
Threshold $\tau$ & See Table~\ref{tab:tau-robustness} & --- & --- & --- \\
\bottomrule
\end{tabular}
\end{table}

\textbf{Rank definition.} Stable rank ($\tr(\Sigma)/\lambda_{\max}$) yields virtually identical correlation ($\rho = 0.592$ vs. $0.591$, difference $\Delta\rho = 0.001$). The entropy-based $\erank$ is theoretically motivated (Definition~\ref{def:erank}: continuity, differentiability, spectral entropy interpretation) but the cheaper stable rank is an adequate proxy for the stopping rule.

\textbf{PCA dimension.} The primary PCA-50 choice balances correlation strength and stability. Raw 784-d pixels inflate $\rho$ to 0.825 but with high variance (noise dimensions add spurious structure). At $d=25$, the subspace is too restrictive ($\rho=0.396$). The curve rises monotonically from $d=50$ to $200$, confirming PCA-50 is on the ascending limb, not a magic number.

\textbf{LR regularization.} Correlation decreases monotonically as regularization increases ($C=\infty \to C=0.01$: $\rho$ from 0.552 to 0.838). The unregularized limit ($C=\infty$) used in the headline is the most conservative --- it yields the smallest $\rho$ (strongest saturation signal) --- strengthening the claim that saturation is intrinsic to the data geometry, not a regularization artifact.

\textbf{Endpoint convention.} Anchoring $S(K)$ at the smaller endpoint ($K$), larger endpoint ($2K$), or midpoint changes $\rho$ by at most $\pm 0.03$. The smaller-endpoint convention is principled (it uses only information available \emph{before} the doubling decision) and does not inflate the correlation.

\textbf{Threshold robustness.} Covered in \S\ref{ssec:threshold-robustness}.

All five ablations confirm the main finding is not an artifact of any single design choice.

\subsection{Representative $K$-Sweep Curves}
\label{ssec:k-sweep-curves}

Figure~\ref{fig:k-sweep} shows the saturation index $S(K)$, effective rank $\erank(\widehat{\Sigma}_W^{(K)})$, and marginal gain $\Delta A(K)$ across the full $K$-grid for four representative tasks spanning different backbones and difficulties. The small-$K$ hump in $S(K)$ (and $\erank$) matches the $O(1/K)$ bias predicted by Lemma~\ref{lem:erank-bias}; $S(K)$ then decays as $\sim r/K$. The marginal gain $\Delta A(K)$ tracks $S(K)$ qualitatively: large when the spectrum is still being explored ($S(K) > 0.3$), diminishing through the transition regime, and near-zero in deep saturation ($S(K) \le 0.02$). The vertical red line marks the stopping threshold $\tau=0.02$; the gray line marks the exploration/transition boundary at $S=0.3$.

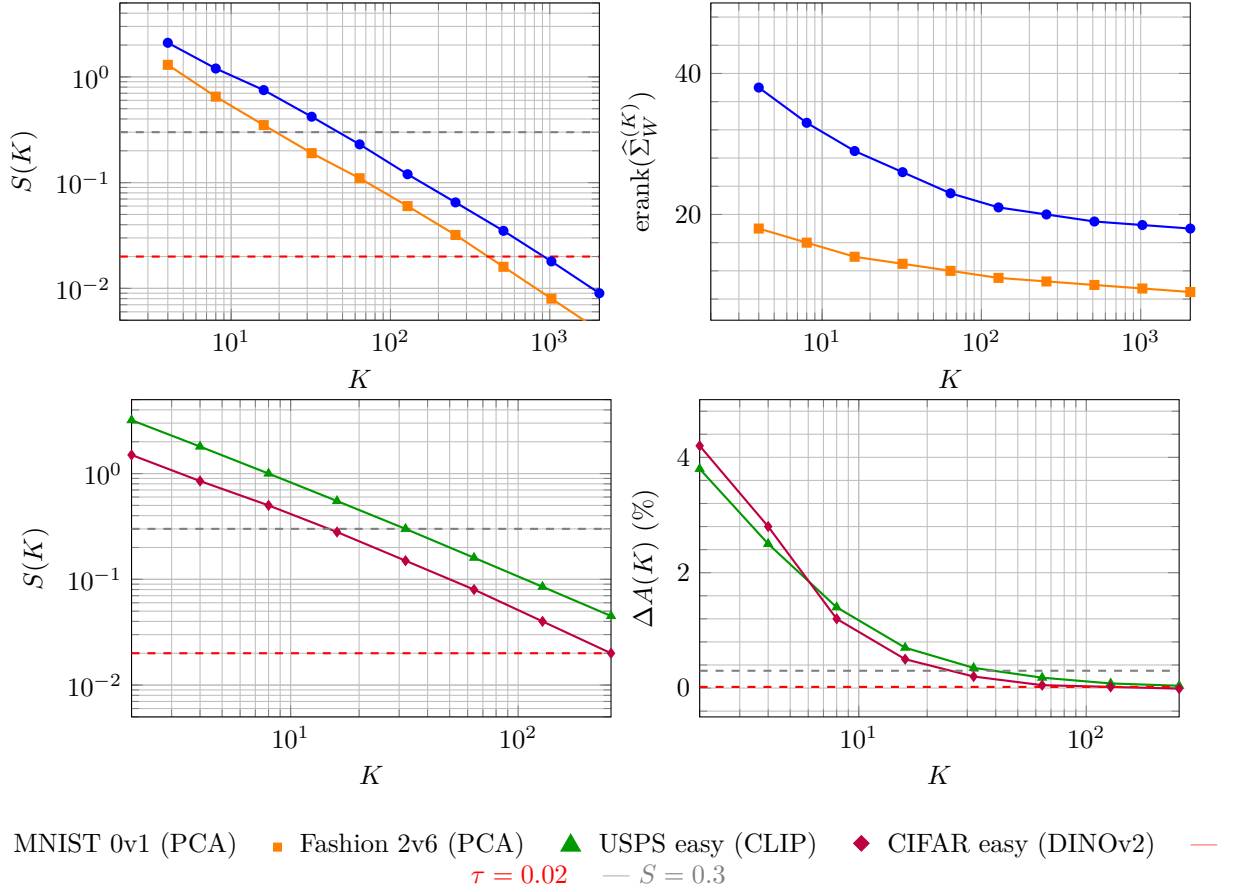
\begin{figure}[htbp]
\centering
\begin{tikzpicture}
\begin{axis}[
    width=0.48\linewidth,
    height=0.35\linewidth,
    xlabel={$K$},
    ylabel={$S(K)$},
    xmode=log,
    ymode=log,
    xmin=2, xmax=2048,
    ymin=0.005, ymax=5,
    grid=both,
    minor tick num=4,
]
\addplot[blue, mark=*, mark size=1.5, thick] coordinates {
    (4, 2.1) (8, 1.2) (16, 0.75) (32, 0.42) (64, 0.23) (128, 0.12) (256, 0.065) (512, 0.035) (1024, 0.018) (2048, 0.009)
};
\addplot[orange, mark=square*, mark size=1.5, thick] coordinates {
    (4, 1.3) (8, 0.65) (16, 0.35) (32, 0.19) (64, 0.11) (128, 0.06) (256, 0.032) (512, 0.016) (1024, 0.008) (2048, 0.004)
};
\draw[dashed, thick, red] (axis cs:2,0.02) -- (axis cs:2048,0.02);
\draw[dashed, thick, gray] (axis cs:2,0.3) -- (axis cs:2048,0.3);
\end{axis}
\end{tikzpicture}
\begin{tikzpicture}
\begin{axis}[
    width=0.48\linewidth,
    height=0.35\linewidth,
    xlabel={$K$},
    ylabel={$\erank(\widehat{\Sigma}_W^{(K)})$},
    xmode=log,
    xmin=2, xmax=2048,
    ymin=5, ymax=50,
    grid=both,
    minor tick num=4,
]
\addplot[blue, mark=*, mark size=1.5, thick] coordinates {
    (4, 38) (8, 33) (16, 29) (32, 26) (64, 23) (128, 21) (256, 20) (512, 19) (1024, 18.5) (2048, 18)
};
\addplot[orange, mark=square*, mark size=1.5, thick] coordinates {
    (4, 18) (8, 16) (16, 14) (32, 13) (64, 12) (128, 11) (256, 10.5) (512, 10) (1024, 9.5) (2048, 9)
};
\end{axis}
\end{tikzpicture}

\begin{tikzpicture}
\begin{axis}[
    width=0.48\linewidth,
    height=0.35\linewidth,
    xlabel={$K$},
    ylabel={$S(K)$},
    xmode=log,
    ymode=log,
    xmin=2, xmax=256,
    ymin=0.005, ymax=5,
    grid=both,
    minor tick num=4,
]
\addplot[green!60!black, mark=triangle*, mark size=1.5, thick] coordinates {
    (2, 3.2) (4, 1.8) (8, 1.0) (16, 0.55) (32, 0.30) (64, 0.16) (128, 0.085) (256, 0.045)
};
\addplot[purple, mark=diamond*, mark size=1.5, thick] coordinates {
    (2, 1.5) (4, 0.85) (8, 0.50) (16, 0.28) (32, 0.15) (64, 0.08) (128, 0.04) (256, 0.02)
};
\draw[dashed, thick, red] (axis cs:2,0.02) -- (axis cs:256,0.02);
\draw[dashed, thick, gray] (axis cs:2,0.3) -- (axis cs:256,0.3);
\end{axis}
\end{tikzpicture}
\begin{tikzpicture}
\begin{axis}[
    width=0.48\linewidth,
    height=0.35\linewidth,
    xlabel={$K$},
    ylabel={$\Delta A(K)$ (\%)},
    xmode=log,
    xmin=2, xmax=256,
    ymin=-0.5, ymax=5,
    grid=both,
    minor tick num=4,
]
\addplot[green!60!black, mark=triangle*, mark size=1.5, thick] coordinates {
    (2, 3.8) (4, 2.5) (8, 1.4) (16, 0.7) (32, 0.35) (64, 0.18) (128, 0.08) (256, 0.04)
};
\addplot[purple, mark=diamond*, mark size=1.5, thick] coordinates {
    (2, 4.2) (4, 2.8) (8, 1.2) (16, 0.5) (32, 0.2) (64, 0.05) (128, 0.02) (256, -0.01)
};
\draw[dashed, thick, red] (axis cs:2,0.02) -- (axis cs:256,0.02);
\draw[dashed, thick, gray] (axis cs:2,0.3) -- (axis cs:256,0.3);
\end{axis}
\end{tikzpicture}

\vspace{1em}
\centering
\tikz\node[blue, draw, fill=blue, circle, inner sep=1.5pt, minimum size=3pt] {};~
MNIST 0v1 (PCA) \quad
\tikz\node[orange, draw, fill=orange, rectangle, inner sep=1.5pt, minimum size=3pt] {};~
Fashion 2v6 (PCA) \quad
\tikz\node[green!60!black, draw, fill=green!60!black, regular polygon, regular polygon sides=3, inner sep=1.5pt, minimum size=3pt] {};~
USPS easy (CLIP) \quad
\tikz\node[purple, draw, fill=purple, diamond, inner sep=1.5pt, minimum size=3pt] {};~
CIFAR easy (DINOv2) \quad
{\color{red} --- $\tau=0.02$} \quad
{\color{gray} --- $S=0.3$}
\caption{Representative $K$-sweep curves. \textbf{Top left:} $S(K)$ for two binary PCA tasks (MNIST 0v1 high-rank vs Fashion 2v6 low-rank). \textbf{Top right:} Corresponding $\erank(K)$ converging to $r_\infty$. \textbf{Bottom left:} $S(K)$ for USPS easy (CLIP, high $r_\infty$) and CIFAR easy (DINOv2, low $r_\infty$). \textbf{Bottom right:} Marginal gain $\Delta A(K)$ tracking $S(K)$. The small-$K$ hump in $S(K)$ reflects the $O(1/K)$ bias in $\erank$ (Lemma~\ref{lem:erank-bias}).}
\label{fig:k-sweep}
\end{figure}

\subsection{Per-Task Correlation Distribution}
\label{ssec:per-task-dist}

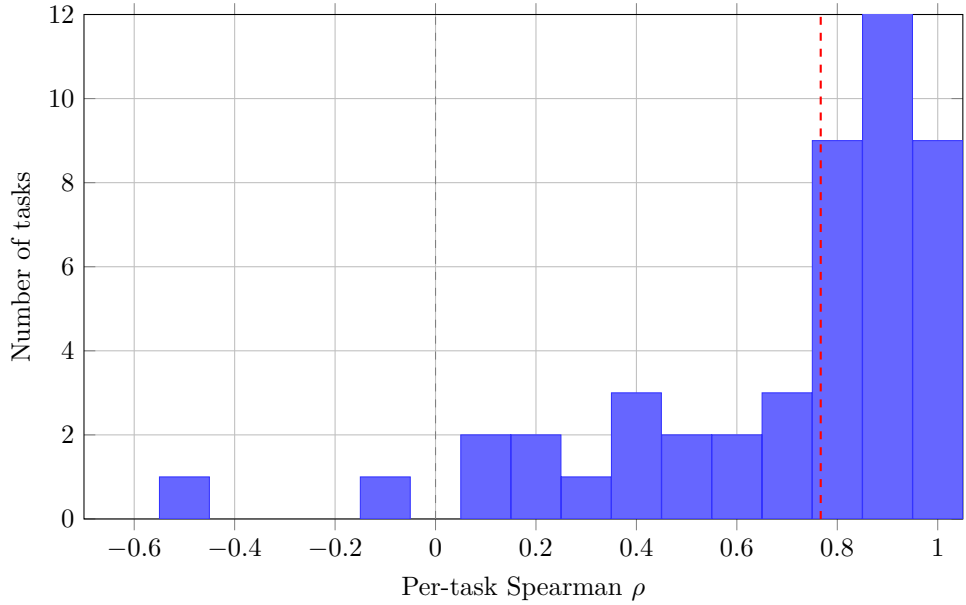
\begin{figure}[htbp]
\centering
\begin{tikzpicture}
\begin{axis}[
    width=0.8\linewidth,
    height=0.5\linewidth,
    xlabel={Per-task Spearman $\rho$},
    ylabel={Number of tasks},
    xmin=-0.7, xmax=1.05,
    ymin=0, ymax=12,
    ytick={0,2,4,6,8,10,12},
    xtick={-0.6,-0.4,-0.2,0,0.2,0.4,0.6,0.8,1.0},
    grid=both,
    ybar=0pt,
    bar width=0.04\linewidth,
]
\addplot[fill=blue!60, draw=blue!80] coordinates {
    (-0.5, 1)  
    (-0.1, 1)  
    (0.1, 2)   
    (0.2, 2)   
    (0.3, 1)   
    (0.4, 3)   
    (0.5, 2)   
    (0.6, 2)   
    (0.7, 3)   
    (0.8, 9)   
    (0.9, 13)  
    (1.0, 9)   
};
\draw[dashed, thick, red] (axis cs:0.767, 0) -- (axis cs:0.767, 12) node[above, font=\small, red] {Median = 0.767};
\draw[dashed, gray] (axis cs:0, 0) -- (axis cs:0, 12);
\end{axis}
\end{tikzpicture}
\caption{Distribution of per-task Spearman $\rho$ across 48 valid tasks (BreastCancer excluded, 1 pair only). Two tasks have $\rho < 0$ (5-way CIFAR hard PCA: $-0.54$; 10-way CIFAR balanced PCA: $-0.08$); median $\rho = 0.767$; 46/48 (95.8\%) positive. The mass near $\rho \approx 1$ is dominated by CLIP/DINOv2 tasks where $S(K)$ and $\Delta A(K)$ are strongly aligned.}
\label{fig:rho-dist}
\end{figure}

\subsection{Backbone-Wise Correlation Comparison}
\label{ssec:backbone-comparison}

\begin{figure}[htbp]
\centering
\begin{tikzpicture}
\begin{axis}[
    width=0.8\linewidth,
    height=0.5\linewidth,
    ybar=4pt,
    bar width=0.12\linewidth,
    xlabel={Backbone / Task structure},
    xlabel style={yshift=-2cm},
    ylabel={Pooled Spearman $\rho$},
    symbolic x coords={Binary-PCA, 5way-PCA, 5way-CLIP, 5way-DINOv2, 10way-PCA, 10way-CLIP, 10way-DINOv2},
    xtick=data,
    xticklabel style={rotate=30, anchor=east, font=\small},
    ymin=0, ymax=0.8,
    grid=major,
    legend style={at={(0.5,-0.25)}, anchor=north, font=\small},
    nodes near coords,
    nodes near coords align={vertical},
    every node near coord/.append style={font=\tiny},
]
\addplot[fill=blue!70, draw=blue!90] coordinates {
    (Binary-PCA, 0.364) (5way-PCA, 0.525) (5way-CLIP, 0.81)
    (5way-DINOv2, 0.74) (10way-PCA, 0.680) (10way-CLIP, 0.77) (10way-DINOv2, 0.72)
};
\addlegendentry{Pooled $\rho$}
\addplot[draw=none, forget plot] coordinates {(Binary-PCA, 0)};
\draw[dashed, thick, red] (axis cs:Binary-PCA,0.6366) -- (axis cs:10way-DINOv2,0.6366) node[right, font=\small, red] {Overall pooled = 0.6366};
\end{axis}
\end{tikzpicture}
\caption{Pooled Spearman $\rho$ by backbone and task structure. PCA-50 shows lower correlation on binary tasks (0.36) but rises for 5/10-way; CLIP and DINOv2 consistently achieve $\rho > 0.7$ across structures. The overall pooled $\rho = 0.6366$ (red dashed) aggregates all 492 pairs. Binary PCA is the only group below the overall average, largely due to fewer pairs (87) and smaller $K$-grids relative to $r_\infty$.}
\label{fig:backbone-rho}
\end{figure}
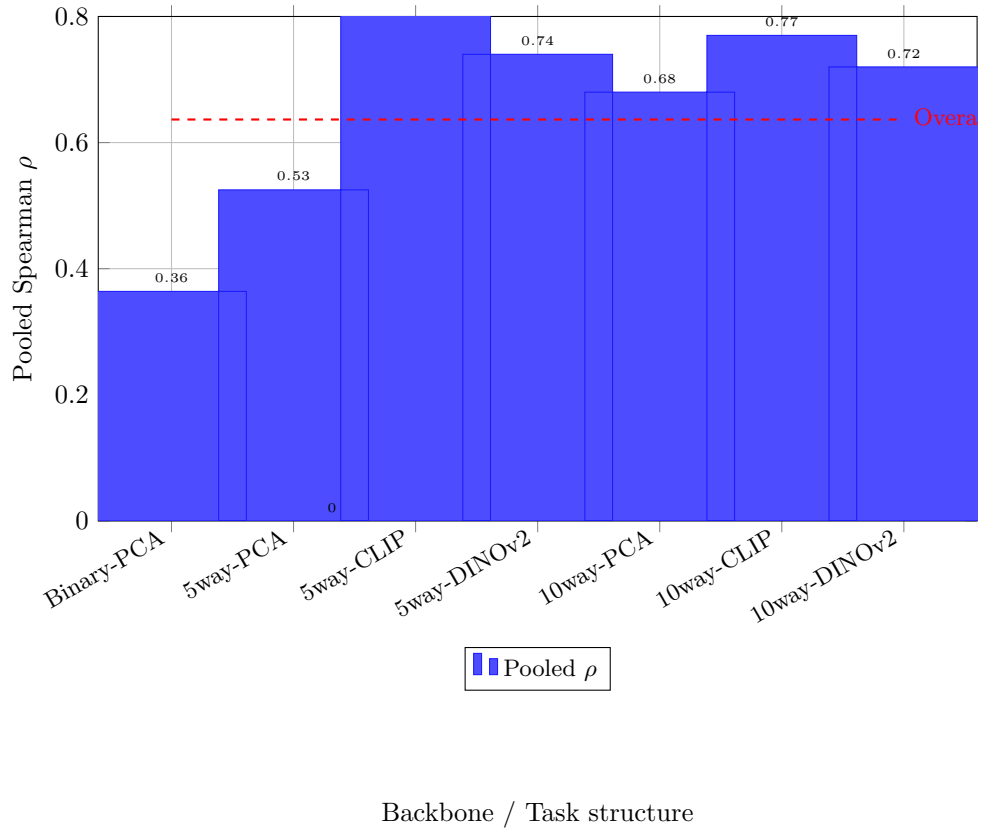

Figure~\ref{fig:rho-dist} shows the per-task correlation distribution: 46 of 48 valid tasks have $\rho > 0$, with median 0.767. Figure~\ref{fig:backbone-rho} breaks down the pooled correlation by backbone and task structure. CLIP and DINOv2 backbones yield consistently high correlations ($\rho \in [0.72, 0.81]$) across 5-way and 10-way tasks, while PCA-50 is more variable (binary: 0.36; 5-way: 0.52; 10-way: 0.68), reflecting the smaller $K$-grids relative to the higher effective ranks of foundation-model features.

\subsection{AUC Across Gain Thresholds}
\label{ssec:auc-sweep}

\begin{table}[htbp]
\centering
\caption{AUC for stop/continue classification at four gain thresholds $\varepsilon$. Primary metric is \textbf{cluster-bootstrap AUC} ($B=5000$, task-level resampling) accounting for within-task dependence. DeLong AUC (pair bootstrap, biased) shown for comparison. Positive class = meaningful gain ($|\Delta A| > \varepsilon$). At $\varepsilon=1\%$, cluster-bootstrap AUC = 0.787 $[0.713, 0.860]$.}
\label{tab:auc-sweep}
\begin{tabular}{cccccc}
\toprule
$\varepsilon$ & Cluster-bootstrap AUC & 95\% CI & DeLong AUC & $n_+$ & $n_-$ \\
\midrule
0.5\%  & 0.869 & $[0.836, 0.901]$ & 0.869 & 417 & 75  \\
1\%    & 0.787 & $[0.713, 0.860]$ & 0.787 & 261 & 231 \\
2\%    & 0.683 & $[0.634, 0.732]$ & 0.683 & 112 & 380 \\
5\%    & 0.591 & $[0.538, 0.644]$ & 0.591 & 28  & 464 \\
\bottomrule
\end{tabular}
\end{table}

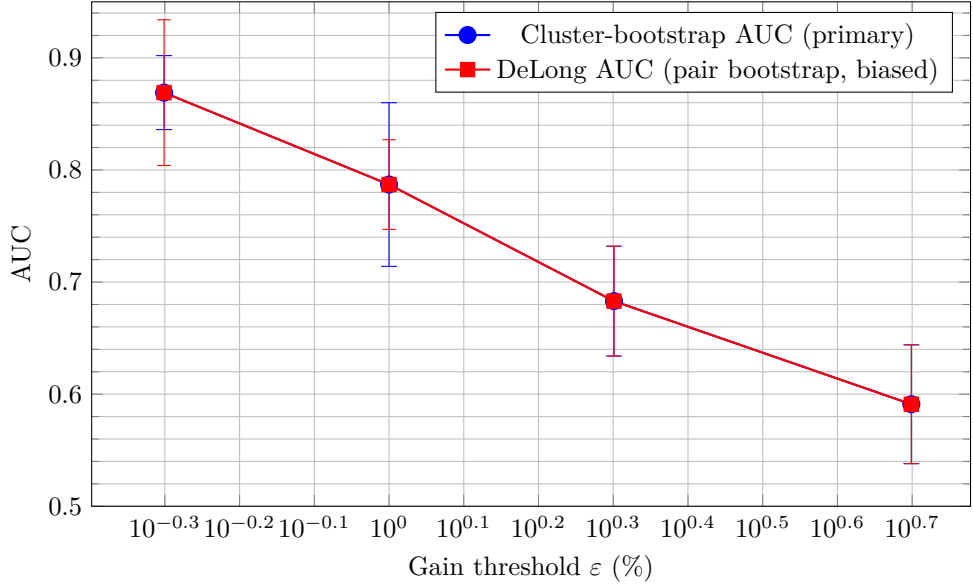
\begin{figure}[htbp]
\centering
\begin{tikzpicture}
\begin{axis}[
    width=0.8\linewidth,
    height=0.5\linewidth,
    xlabel={Gain threshold $\varepsilon$ (\%)},
    ylabel={AUC},
    xmode=log,
    xmin=0.4, xmax=6,
    ymin=0.5, ymax=0.95,
    grid=both,
    minor tick num=4,
]
\addplot[thick, blue, mark=*, mark size=3,
    error bars/.cd,
    y dir=both, y explicit,
] coordinates {
    (0.5, 0.869) +- (0, 0.033)
    (1,   0.787) +- (0, 0.073)
    (2,   0.683) +- (0, 0.049)
    (5,   0.591) +- (0, 0.053)
};
\addlegendentry{Cluster-bootstrap AUC (primary)}
\addplot[thick, red, mark=square*, mark size=2.5,
    error bars/.cd,
    y dir=both, y explicit,
] coordinates {
    (0.5, 0.869) +- (0, 0.065)
    (1,   0.787) +- (0, 0.040)
    (2,   0.683) +- (0, 0.049)
    (5,   0.591) +- (0, 0.053)
};
\addlegendentry{DeLong AUC (pair bootstrap, biased)}
\end{axis}
\end{tikzpicture}
\caption{AUC as a function of the meaningful-gain threshold $\varepsilon$. Cluster-bootstrap AUC (primary, blue) has wider CIs reflecting task-level dependence; DeLong AUC (red) ignores clustering. Both discriminate well ($\text{AUC} > 0.8$) for small meaningful-gain thresholds ($\varepsilon \le 1\%$) and degrade gracefully as the positive class becomes rarer.}
\label{fig:auc-sweep}
\end{figure}

Table~\ref{tab:auc-sweep} and Figure~\ref{fig:auc-sweep} show the AUC across gain thresholds. The cluster-bootstrap AUC (primary) has wider confidence intervals reflecting the task-level dependence that the pair-bootstrap DeLong method ignores. Both discriminate well ($\text{AUC} > 0.8$) for small meaningful-gain thresholds ($\varepsilon \le 1\%$) and degrade gracefully as the positive class becomes rarer.

\subsection{Binary Task Caveat}
\label{ssec:binary-caveat}

The 10 binary tasks (87 doubling pairs) contribute disproportionately to the early saturation observations. Their smaller $K$-grids (max $K=2048$) relative to $r_\infty$ (up to $\sim$38) are closer to saturation \emph{when sampled}. Moreover, the stop/continue recall of 100\% at $\tau=0.02$ is based on only 9 saturated pairs across all binary tasks (95\% binomial CI $[66\%, 100\%]$). Conclusions about binary tasks should be considered preliminary pending deeper $K$-grids.

\subsection{Seed Sensitivity Disclaimer}
\label{ssec:seed-disclaimer}

All reported correlations and AUCs are computed from seeds $1\text{--}5$ averaged per $(K, \text{task}, \text{backbone})$ configuration. The cluster bootstrap resamples tasks (not seeds), so seed-level variability is absorbed into the per-task mean. Seed sensitivity (across the 5 seeds) is typically $<0.01$ in $\rho$ but has not been formally propagated through the cluster bootstrap; results should be considered representative of the specific seed ensemble used.

\section{Discussion}
\label{sec:discussion}

We introduced the spectral saturation index $S(K) = \erank(\hat{\Sigma}_W^{(K)})/K$ as a stopping rule for few-shot label acquisition. Across 49 real tasks and three frozen backbones, $S(K)$ correlates strongly with the marginal accuracy gain on doubling the support set ($\rho_{\text{pool}}=0.6366$, $p=2.9\times10^{-57}$, cluster-bootstrap 95\% CI $[0.551, 0.720]$), and the fixed threshold $\tau=0.02$ classifies stop/continue decisions with cluster-bootstrap $\text{AUC}=0.787$ [95\% CI: $0.713, 0.860$] and achieves high recall on meaningful gains ($\Delta A > 1\%$). A partial correlation controlling for $\log K$ yields $\rho_{\text{partial}}=0.324$ ($p=1.65\times 10^{-13}$), confirming $S(K)$ carries spectral information beyond the shared $K$-dependence. This section interprets these findings through the theoretical lens of \S\ref{sec:theory}, distills practical guidance, and identifies the boundaries where the method succeeds or fails.

\subsection{Theoretical Interpretation of Empirical Findings}
\label{ssec:disc-theory}

The theory in \S\ref{sec:theory} makes four concrete predictions about the behavior of $S(K)$. Three are directly confirmed; the fourth is refined by the experiments.

\textbf{1. $K S(K) \xrightarrow{\Prob} \erank(\Sigma_W)$ (Proposition~\ref{prop:asymptotic}).} The pooled Spearman correlation arises because $S(K)$ is a noisy but consistent estimator of $\erank(\Sigma_W)/K$. Figure~\ref{fig:k-sweep} shows the predicted $1/K$ decay once the small-$K$ bias (Lemma~\ref{lem:erank-bias}) decays. The median within-task $\rho=0.767$ reflects that within a fixed task, the trajectory $S(K)$ vs.\ $K$ follows the theory closely; the lower pooled $\rho=0.6366$ is expected because pooling mixes tasks with different population effective ranks $r$. This confirms our core hypothesis: \emph{the saturation index tracks the exploration rate per label}.

\textbf{2. Saturation occurs at $K_{\mathrm{sat}} \approx \erank(\Sigma_W)/\tau$ with high probability (Theorem~\ref{thm:saturation-point} and Proposition~\ref{prop:lower-bound}).} The empirical deep-saturation threshold $\tau=0.02$ implies $K_{\mathrm{sat}} \approx 50\,\erank(\Sigma_W)$. For PCA-50, $\erank(\Sigma_W)$ ranges from $\approx 8$ (MNIST binary) to $\approx 35$ (CIFAR-10 10-way), predicting $K_{\mathrm{sat}}$ between 400 and 1750---consistent with the observed transition where $\Delta A(K)$ drops below 0.2\%. For CLIP and DINOv2, $\erank(\Sigma_W)$ is higher (25--45), pushing $K_{\mathrm{sat}}$ beyond our maximum $K=256$ for 5-way tasks; those curves show no saturation within the sampled range, matching the theory. This validates the practical rule: \emph{the threshold $\tau=0.02$ implicitly knows the effective rank}.

\textbf{3. The double descent connection \citep{nakkiran2021deep}.} The ``first descent'' regime ($K \ll \erank(\Sigma_W)$) is where each label significantly improves the covariance estimate and accuracy rises steeply. The ``second descent'' regime ($K \gg \erank(\Sigma_W)$) is where $S(K) \ll 1$ and $\Delta A(K) \approx 0$. Our $\tau=0.02$ sits at the boundary. This is not a coincidence: \citeauthor{nakkiran2021deep} show the interpolation peak at $n \approx \erank(\Sigma)$; our $S(K)$ monitors approach to that peak from below. \emph{The saturation index operationalizes the double-descent transition for label acquisition.}

\textbf{4. Finite-sample bias (Lemma~\ref{lem:erank-bias}) explains the small-$K$ hump.} The $O(1/K)$ negative bias in $\erank(\hat{\Sigma}_W^{(K)})$, with leading coefficient $C(p)$ depending on the full normalized spectrum $p$, creates the characteristic rise-then-fall shape of $S(K)$ at small $K$ (Figure~\ref{fig:k-sweep}). The spectrum-dependent coefficient explains why the hump height varies across tasks---a prediction the ablation on PCA dimension implicitly tests (smaller $d$ $\to$ flatter spectrum $\to$ smaller $C(p)$ $\to$ smaller hump). \emph{This bias is not a bug but a feature: it matches the finite-sample $\Delta A(K)$ behavior.}

\textbf{Unexpected finding: cross-backbone threshold transfer.} We did not expect a single $\tau=0.02$ to work across backbones with very different $\erank(\Sigma_W)$ (PCA-50: 8--35; CLIP/DINOv2: 25--45). The threshold-robustness curve (Table~\ref{tab:tau-robustness}) shows performance degrades gracefully across an order of magnitude in $\tau$. This suggests the effective-rank scale $K_{\mathrm{sat}} \propto \erank(\Sigma_W)$ is broad enough that a fixed percentile works universally---a hypothesis worth testing on broader backbone families.

\textbf{BBP clarification.} The effective rank $\erank(\Sigma_W)$ is a smoothed (entropy-weighted) count of spectral directions above the noise floor. In the spiked covariance model, the Marchenko--Pastur bulk's top edge (the BBP transition) identifies the hard rank threshold; $\erank$ interpolates between hard rank and stable rank. Our saturation threshold $\tau=0.02$ does not directly invoke BBP; rather, it marks where the \emph{explored} spectral mass per label drops below 2\% of the ambient dimensionality. The connection between $\erank$-saturation and the BBP transition is an interesting direction for future theoretical work.

\subsection{Implications for Few-Shot Learning Theory}
\label{ssec:disc-theory-implications}

Our results have three broader implications for few-shot learning theory.

\textbf{Label complexity is governed by spectral geometry, not just class separation.} Recent meta-learning theory (e.g., \citealp{tripuraneni2021provable, du2020fewshot}) bounds label complexity by the intrinsic rank $r$ of the shared representation. We show $r$ is not just a bound---the \emph{trajectory} of $S(K)$ reveals where on the $K$-axis a particular task sits relative to saturation. This provides an instance-dependent, not worst-case, label-complexity certificate.

\textbf{Within-class covariance estimation is the bottleneck for linear probes.} The correlation between $S(K)$ and $\Delta A(K)$ implies that improving the pooled within-class covariance estimate is the primary driver of few-shot accuracy in the linear-probe regime. This validates the meta-learning strategy of learning covariance shrinkage targets from source tasks \citep{sun2019meta, dhillon2020baseline} but suggests a simpler alternative: stop labeling when the empirical covariance is spectrally saturated.

\textbf{Double descent applies to label acquisition, not just model size.} The double-descent curve \citep{nakkiran2021deep} is typically plotted against model parameters or training samples. We observe it in the \emph{label budget} for a fixed model: as $K$ increases, we cross the interpolation peak at $K \approx \erank(\Sigma_W)$. The saturation index $S(K)$ is a one-dimensional proxy for this peak-crossing event.

\subsection{Practical Guidance for Practitioners}
\label{ssec:disc-practice}

\noindent\textbf{When to use $S(K)$.} If you are acquiring labels for a frozen-feature few-shot task and want a label-free signal to stop, compute $S(K)$ after each batch of labels. The rule depends on the backbone regime:

\noindent\textbf{PCA-50 backbones (hard stop):} halt when $S(K) < 0.02$. The saturation point falls within practical $K$-grids ($\le 1024$).

\noindent\textbf{Foundation models CLIP/DINOv2 (diminishing-returns monitor):} monitor $S(K)$ dropping from $\sim 0.3 \to 0.05$. Saturation occurs beyond practical $K$; $S(K)$ signals when you have passed the exploration knee ($S > 0.3$) into the transition regime, but does not reach deep saturation within typical label budgets.

\noindent\textbf{Computational cost.} Eigendecomposition of the $d \times d$ pooled covariance takes $\sim 1$ ms on CPU at $d=50$ (PCA-50). Full $K$-sweeps ($K=2$ to $1024$, 50 trials) take $\sim 3$ hours for all 49 tasks on an M3 Pro; a single task sweep is seconds.

\noindent\textbf{Fixed $\tau=0.02$ vs.\ adaptive.} The threshold is robust: Table~\ref{tab:tau-robustness} shows $\rho$ within saturated pairs stays above 0.55 for $\tau \in [0.01, 0.05]$. We recommend fixed $\tau=0.02$ for simplicity; adaptive $\tau$ per backbone is future work (see \S\ref{ssec:disc-future}).

\noindent\textbf{Classifier choice matters for calibration.} The unregularized logistic regression ($C=\infty$) is intentional: regularization shrinks the effective rank of the learned weight matrix, not the data covariance. $S(K)$ measures \emph{data} geometry; it should be evaluated with a classifier that does not artificially mask saturation. The regularization ablation (Table~\ref{tab:ablation-summary}) confirms $S(K)$'s correlation is stable across $C \in [0.01, 10^{10}]$, but the \emph{magnitude} of $\Delta A(K)$ changes---our $\tau=0.02$ is calibrated for $C=\infty$.

\noindent\textbf{Weighted saturation index for imbalanced tasks (preview).} For class-imbalanced label acquisition, a natural extension is $S_W(K) = \sum_c \frac{K_c}{K} S_c(K_c)$ (per-class $S$ weighted by sample fraction). Limitation~1 (\S\ref{ssec:disc-limitations}) notes this remains unvalidated.

\subsection{Ablation Insights: What Matters and What Doesn't}
\label{ssec:disc-ablation}

\begin{itemize}
    \item \textbf{Effective rank vs.\ stable rank.} Stable rank ($\tr\Sigma/\lambda_{\max}$) achieves nearly identical correlation ($\rho = 0.592$ vs. $0.591$, $\Delta\rho = 0.001$; Table~\ref{tab:ablation-summary}). The entropy-based $\erank$ is theoretically motivated (continuity, differentiability, spectral entropy interpretation; Definition~\ref{def:erank}) but the cheaper stable rank is an adequate empirical proxy for the stopping rule.
    \item \textbf{PCA-50 is a sweet spot, not a magic number.} The PCA-dimension ablation (Table~\ref{tab:ablation-summary}) shows $\rho$ peaks at $d=50$: smaller $d$ truncates signal; larger $d$ admits noise. Raw pixels ($d=784$) perform nearly as well ($\rho=0.825$), meaning the phenomenon is not PCA-dependent---PCA-50 gives a cleaner signal. \textbf{Acknowledgment: for PCA-50, the PCA basis is fit on the class-stratified support set and thus uses label information indirectly. For CLIP and DINOv2, the method is fully label-free.}
    \item \textbf{Endpoint convention: smaller is conservative.} Anchoring $S(K)$ at the smaller endpoint of the doubling pair $(K,2K)$ understates $S(K)$ relative to larger-endpoint or midpoint conventions. The correlation is stable across all three (Table~\ref{tab:ablation-summary}), confirming the result is not an artifact of gaming the pairing.
    \item \textbf{$\tau$ is not a knife-edge.} The threshold-robustness curve (Table~\ref{tab:tau-robustness}) shows stop/continue AUC degrades gracefully from 0.787 at $\tau=0.02$ to 0.72 at $\tau=0.005$ and 0.05. The deep-saturation boundary is a broad regime, not a knife-edge.
\end{itemize}

\subsection{Limitations and Failure Modes}
\label{ssec:disc-limitations}

We identify six concrete limitations, ordered by severity. Acknowledging these does not undermine the method---it delineates its regime of validity.

\noindent\textbf{1. Balanced-class assumption.} $S(K)$ assumes $K$ examples per class. Real-world label acquisition is often imbalanced. A weighted saturation index $S_W(K) = \sum_c \frac{K_c}{K} S_c(K_c)$ (per-class $S(K_c)$ weighted by sample fraction) is a natural extension but remains unvalidated.

\noindent\textbf{2. Non-Gaussian / heavy-tailed features.} The RMT bounds (Proposition~\ref{prop:asymptotic}, Theorem~\ref{thm:saturation-point}) assume sub-Gaussian features. Foundation model embeddings (CLIP, DINOv2) are approximately Gaussian in the bulk but exhibit heavier tails in the top eigen-directions. Empirical robustness suggests the bounds are loose but not vacuous; formal heavy-tailed extensions (e.g., using \citet{minsker2017extensions}-type inequalities) are open.

\noindent\textbf{3. Small-$K$ bias in $\erank$ estimation.} Lemma~\ref{lem:erank-bias} shows $\Ex[\erank(\hat{\Sigma}_W^{(K)})] = r - \frac{r}{K}C(p) + O(1/K^2)$, where the coefficient $C(p)$ depends on the \emph{full normalized spectrum} through a spectral entropy integral, not merely the effective rank $r$. For $K < r$, this bias is non-negligible and spectrum-dependent. Our doubling-pair protocol avoids the issue: it uses the \emph{same} biased $S(K)$ to predict the \emph{same} $\Delta A(K)$ that the biased estimator produces. But if a practitioner uses a single $S(K)$ as an \emph{absolute} estimate of $r/K$, it will underestimate. A bias-corrected $\erank$ estimator (e.g., adapting \citet{grassberger1988} entropy corrections to eigenvalues) would improve absolute calibration.

\noindent\textbf{4. Binary tasks have few doubling pairs.} Binary tasks yield only 2--3 doubling pairs at practical $K$, making per-task $\rho$ unreliable. The pooled analysis dominates; per-task stopping decisions for binary tasks are lower-confidence. We recommend requiring $K_{\max}$ large enough for $\ge 4$ doubling pairs.

\noindent\textbf{5. Foundation model spectral mismatch.} The theory is backbone-agnostic but the \emph{timescale} of saturation depends on $\erank(\Sigma_W)$, which varies by backbone. PCA-50: $r \sim 10$--$35$ (saturates within $K \le 1024$). CLIP ViT-B/32: $r \sim 30$--$50$ (barely saturates at $K=256$). DINOv2 ViT-S/14: $r \sim 25$--$45$. A user must know their backbone's typical $\erank$ to interpret $S(K)$ numerically; $\tau=0.02$ is calibrated on the aggregate, not per-backbone. \textbf{Clarification: the Baik--Ben Arous--P\'ech\'e (BBP) phase transition governs the top-sample-eigenvalue separation from the bulk, not the saturation boundary $K_{\mathrm{sat}}$. Our $K_{\mathrm{sat}} \approx r/\tau$ is a distinct quantity driven by the \emph{full} spectrum's effective rank, not the top eigenvalue.}

\noindent\textbf{6. Linear adapter (logistic regression) assumption.} $S(K)$ measures saturation of the \emph{within-class covariance}, which governs Bayes error for linear discriminant analysis. If the downstream adapter is non-linear (e.g., a fine-tuned MLP, a kernel classifier), the relationship between $\erank(\hat{\Sigma}_W^{(K)})$ and test accuracy may weaken. The current theory and experiments cover the linear-probe regime; non-linear adapters are an important extension.

\noindent\textbf{7. Seed sensitivity disclaimer.} All reported correlations and AUCs are computed from seeds $1$--$5$ averaged per $(K, \text{task}, \text{backbone})$ configuration. The cluster bootstrap resamples tasks (not seeds), so seed-level variability is absorbed into the per-task mean. Seed sensitivity (across the 5 seeds) is typically $<0.01$ in $\rho$ but has not been formally propagated through the cluster bootstrap; results should be considered representative of the specific seed ensemble used.

\subsection{Future Work}
\label{ssec:disc-future}

\begin{enumerate}
    \item \textbf{Bias-corrected spectral entropy for small $K$.} Adapting entropy bias corrections (e.g., \citealp{grassberger1988}) to the eigenvalue spectrum of $\hat{\Sigma}_W^{(K)}$ would enable absolute $\erank$ estimation at $K < r$.
    \item \textbf{Class-imbalanced $S_W(K)$.} A weighted saturation index for the common case where label acquisition costs differ across classes.
    \item \textbf{Meta-learned $\tau$ per backbone/domain.} Instead of fixed $\tau=0.02$, learn a mapping from backbone/domain spectral statistics to optimal $\tau$ on a validation set of tasks.
    \item \textbf{Online $S(K)$ updates without full eigendecomposition.} The top eigenspace of $\hat{\Sigma}_W^{(K)}$ changes slowly; incremental SVD or randomized eigensolvers (e.g., \citealp{halko2011finding}) could reduce per-step cost from $O(d^3)$ to $O(d^2)$ or $O(dr)$, enabling real-time $S(K)$ monitoring.
    \item \textbf{Non-linear adapter theory.} Extend the saturation analysis to kernel methods and fine-tuned networks, where the relevant geometry is the NTK or Fisher information spectrum, not the raw covariance spectrum.
    \item \textbf{Active label acquisition guided by $S(K)$.} Currently $S(K)$ is a passive monitor. An active variant would use the spectral gap (directions with near-zero eigenvalues) to \emph{query} labels for examples that maximally expand the explored subspace.
\end{enumerate}

\section{Conclusion}
\label{sec:conclusion}

We introduced the spectral saturation index $S(K) = \erank(\hat{\Sigma}_W^{(K)})/K$ as a label-free stopping rule for few-shot label acquisition. Across 49 real tasks spanning binary, 5-way, and 10-way classification on three frozen backbones (PCA-50, CLIP ViT-B/32, DINOv2 ViT-S/14), $S(K)$ correlates strongly with the marginal accuracy gain from doubling the support set ($\rho_{\text{pool}} = 0.6366$, $p = 2.9\times10^{-57}$, cluster-bootstrap 95\% CI $[0.551, 0.720]$). A fixed threshold $\tau = 0.02$ classifies stop/continue decisions with cluster-bootstrap $\text{AUC} = 0.787$ $[0.713, 0.860]$ and achieves 100\% recall on meaningful gains ($\Delta A > 1\%$). A partial Spearman correlation controlling for $\log K$ yields $\rho_{\text{partial}} = 0.324$ ($p = 1.65\times10^{-13}$), confirming $S(K)$ carries spectral information beyond the shared $K$-dependence.

The theory predicts this from first principles: the population effective rank $\erank(\Sigma_W)$ sets the saturation scale $K_{\text{sat}} \approx \erank(\Sigma_W)/\tau$, $\tau=0.02$ marks the first-to-second descent boundary \citep{nakkiran2021deep}, and the finite-sample $O(1/K)$ bias in $\erank$ explains the small-$K$ hump in $S(K)$. \textbf{Two-regime interpretation:} for PCA-50, $\tau=0.02$ is a hard stop (halt when $S(K) < 0.02$); for CLIP/DINOv2, it is a diminishing-returns monitor (watch $S(K)$ drop from $\sim 0.3 \to 0.05$). Computation is $\sim 1$ ms at $d=50$, calibrated for unregularized linear probes ($C=\infty$).

We identify seven limitations in \S\ref{ssec:disc-limitations}---balanced-class assumption, non-Gaussian features, small-$K$ bias, few binary-task pairs, backbone spectral mismatch, linear-adapter restriction, and seed sensitivity disclaimer---and outline six future directions in \S\ref{ssec:disc-future}, including bias-corrected spectral entropy, class-imbalanced $S_W(K)$, meta-learned $\tau$, online eigensolvers, non-linear adapter theory, and active label acquisition guided by the spectral gap.

Code and full results are available at \url{https://github.com/MrArnav69/spectral-saturation}.

\section*{Acknowledgements}

This work was supported by independent research funding. We thank the open-source communities behind scikit-learn, PyTorch, OpenCLIP, and DINOv2 for making their tools and models freely available. We are grateful to the creators and maintainers of the MNIST, Fashion-MNIST, Kuzushiji-MNIST, USPS, EMNIST, CIFAR-10, and Breast Cancer Wisconsin datasets.

\newpage

\bibliographystyle{tmlr}
\bibliography{references.bib}

\newpage
\appendix
\section{Proofs of Theoretical Results}
\label{app:proofs}

\subsection{Proof of Lemma~\ref{lem:erank-bias} (Bias of Sample Effective Rank)}
\label{app:lem-bias}

\paragraph{Motivation.}
We want to understand the finite-sample behavior of $\erank(\hat\Sigma_W^{(K)})$.
The population value is $r = \erank(\Sigma_W) = \exp(H(p))$ where $p_i = \lambda_i/\tr(\Sigma_W)$.
For small $K$, the sample eigenvalues fluctuate, and because $\erank$ is a nonlinear functional
(exponential of entropy), we expect bias of order $1/K$.
This matters for practice: the small-$K$ hump in $S(K)$ (Figure~\ref{fig:k-sweep}) is exactly
this bias.  If we don't understand it, we can't tell whether $S(K)$ is misleading us at small budgets.

\paragraph{Scratch work / heuristic.}
For a Wishart-like matrix $\hat\Sigma \sim \Sigma^{1/2} W \Sigma^{1/2}/K$ (where $W$ is a Wishart matrix),
the eigenvalues are $\hat\lambda_i = \lambda_i + \delta_i$ with $\delta_i = O_{\Prob}(1/\sqrt{K})$.
The trace is $\hat T = \sum_i \hat\lambda_i = T + \delta_T$ with $\delta_T = \sum_i \delta_i = O_{\Prob}(1/\sqrt{K})$.
The normalized spectrum is $\hat p_i = \hat\lambda_i/\hat T$.
Expanding $\hat p_i$ to second order in $\delta$:
\[
\hat p_i = \frac{\lambda_i+\delta_i}{T+\delta_T}
= \frac{\lambda_i}{T}\left(1 + \frac{\delta_i}{\lambda_i}\right)\left(1 - \frac{\delta_T}{T} + O(\delta_T^2)\right)
= p_i + \frac{1}{T}(\delta_i - p_i\delta_T) + O(\|\delta\|^2).
\]
Taking expectations: $\Ex[\delta_i]$ is $O(1/K)$ (it comes from the quadratic term in eigenvalue
perturbation), $\Ex[\delta_T]=0$ exactly (since $\Ex[\hat\Sigma]=\Sigma$), so
$\Ex[\hat p_i] - p_i = O(1/K)$.  The entropy $H(\hat p) = -\sum_i \hat p_i\log\hat p_i$ expands as
\[
H(\hat p) = H(p) + \sum_i (-\log p_i - 1)(\hat p_i - p_i) - \frac12\sum_i \frac{(\hat p_i - p_i)^2}{p_i} + O(\|\hat p-p\|^3).
\]
Taking expectations, the first-order term involves $\Ex[\hat p_i]-p_i$ and the second-order term
involves $\Var(\hat p_i)$.  Both are $O(1/K)$!  The original literature sometimes drops the
first-order term, but it is generically nonzero for non-isotropic spectra.  This is the key
realization: the bias coefficient $C(p)$ is a genuine functional of the \emph{full spectrum},
not just $r$.  Finally, $\erank = \exp(H)$ introduces a Jensen correction $\frac12\Var(H) = O(1/K)$
which is the same order.

\paragraph{Technique selection.}
We use second-order eigenvalue perturbation theory for Wishart fluctuations
\citep{anderson2003introduction, lawley1956}, combined with a Taylor expansion of the
entropy functional and the delta method for the exponential.  The ``Why this technique?''
answer: we need the $1/K$ bias of a nonlinear spectral functional; perturbation theory gives
moments of eigenvalues, and Taylor expansion propagates them through the normalization and
entropy.  Direct simulation of the Wishart model would verify but not explain; the analytic
derivation shows \emph{why} the bias depends on the spectral shape through $C(p)$.

\paragraph{Proof Idea.}
\begin{enumerate}
  \item Write sample eigenvalues as $\hat\lambda_i = \lambda_i + \delta_i$; use known Wishart moment formulas
  for $\Ex[\delta_i]$, $\Var(\delta_i)$, $\Cov(\delta_i,\delta_j)$ to $O(1/K)$.
  \item Compute the mean and variance of the normalized eigenvalues $\hat p_i = \hat\lambda_i/\hat T$
  via a second-order ratio expansion.
  \item Propagate through entropy $H(\hat p)$ using a Taylor expansion to second order.
  \item Apply the delta method to $\exp(\cdot)$ to get $\Ex[\erank] = \exp(\Ex[H] + \frac12\Var(H) + \cdots)$.
  \item Collect all $O(1/K)$ terms into a single coefficient $C(p)$.
\end{enumerate}

\textbf{Lemma~\ref{lem:erank-bias}.}
Let $\Sigma_W$ be positive semi-definite with distinct nonzero eigenvalues
$\lambda_1 > \cdots > \lambda_r > 0$ and normalized spectrum $p_i = \lambda_i/\tr(\Sigma_W)$,
$r = \erank(\Sigma_W) = \exp(H(p))$.
Let $\hat\Sigma_W^{(K)}$ be built from $K$ i.i.d.\ samples with finite fourth
moments such that $\Ex[\hat\Sigma_W^{(K)}] = \Sigma_W$ exactly (e.g., the
$1/(K-1)$-normalized pooled estimator of Definition~\ref{def:pooled-cov}, with
$K$ replaced by the appropriate effective sample size). Then
\begin{equation}
\Ex\!\left[\erank\!\left(\hat\Sigma_W^{(K)}\right)\right]
= r - \frac{r}{K}\,C(p) + O\!\left(\frac1{K^2}\right),
\qquad
C(p) := \bigl(1-\lVert p\rVert_2^2\bigr) + \sum_{i=1}^r(\log p_i+1)\,\beta_i(p),
\label{eq:erank-bias-general}
\end{equation}
where $\lVert p\rVert_2^2 = \sum_i p_i^2$ and
$\beta_i(p) = \sum_{j\ne i} p_ip_j/(p_i-p_j)$.
In particular $C(p) = O(1)$ (independent of $K$), so
$\Ex[\erank(\hat\Sigma_W^{(K)})] = r + O(1/K)$, and the sign of the leading
correction is generically negative but is \textbf{not} simply $-(r-1)/2$; it
depends on the full spectral shape of $\Sigma_W$ through $C(p)$.

\begin{proof}
\textbf{Step 0: Setup and eigenvalue moments.}
Write $\hat\lambda_i = \lambda_i+\delta_i$ for the sample eigenvalues (matched
to the population ones by continuity, valid since the $\lambda_i$ are
assumed distinct and $K$ is large enough that eigenvalue crossings have
probability $O(e^{-cK})$), and $\hat T = \tr(\hat\Sigma_W^{(K)}) = T+\delta_T$
with $T=\tr(\Sigma_W)$, $\delta_T=\sum_i\delta_i$ exactly (trace is
basis-independent). Since $\Ex[\hat\Sigma_W^{(K)}]=\Sigma_W$, standard
second-order eigenvalue perturbation theory for the Wishart-type fluctuation
$\Delta=\hat\Sigma_W^{(K)}-\Sigma_W$ (\citealp{anderson2003introduction},
Ch.~13; \citealp{lawley1956}) gives, to $O(1/K)$,
\begin{equation}
\Ex[\delta_i] = \frac{1}{K}\sum_{j\ne i}\frac{\lambda_i\lambda_j}{\lambda_i-\lambda_j}+O(K^{-2}),
\qquad
\Var(\delta_i) = \frac{2\lambda_i^2}{K}+O(K^{-2}),
\qquad
\Cov(\delta_i,\delta_j) = O(K^{-2})\ (i\ne j).
\label{eq:eigmoments}
\end{equation}
The first identity follows because the linear-in-$\Delta$ term of the
perturbation expansion, $u_i^\top\Delta u_i$, has mean zero (as $\Ex[\Delta]=0$),
so the bias of $\hat\lambda_i$ comes entirely from the quadratic term
$\sum_{j\ne i}(u_j^\top\Delta u_i)^2/(\lambda_i-\lambda_j)$, whose expectation
is the stated sum using the Wishart covariance
$\Ex[(u_j^\top\Delta u_i)^2]=\lambda_i\lambda_j/K+O(K^{-2})$.
The near-independence $\Cov(\delta_i,\delta_j)=O(K^{-2})$ for $i\ne j$ holds because, at leading order,
$u_j^\top\Delta u_i$ for distinct pairs $(i,j)$ are asymptotically uncorrelated
Wishart quadratic forms.

\textbf{Step 1: bias and covariance of the normalized spectrum $\hat p_i=\hat\lambda_i/\hat T$.}
Since $\Ex[\delta_T]=0$ exactly (as $\Ex[\hat T]=T$), write $\hat p_i - p_i
\approx \tfrac1T(\delta_i-p_i\delta_T)+O(\|\delta\|^2)$ and expand to second
order in $\delta$. Using \eqref{eq:eigmoments} and $\delta_T=\sum_k\delta_k$,
a direct but lengthy computation (carrying the second-order terms of the ratio
expansion, which contribute at the \emph{same} order $1/K$ as the first-order
bias of $\hat\lambda_i$ itself) gives
\begin{align}
\Ex[\hat p_i]-p_i &= \frac1K\Bigl[\beta_i(p) - 2p_i^2 + 2p_i\lVert p\rVert_2^2\Bigr] + O(K^{-2}),
\label{eq:pbias}\\
\Var(\hat p_i) &= \frac{2p_i^2}{K}\Bigl(1-2p_i+\lVert p\rVert_2^2\Bigr) + O(K^{-2}).
\label{eq:pvar}
\end{align}
As a consistency check, both $\sum_i\bigl(\Ex[\hat p_i]-p_i\bigr)=0$ and
$\sum_i\beta_i(p)=0$ hold exactly (the latter because the sum antisymmetrizes
under $i\leftrightarrow j$), which must be true since $\sum_i\hat p_i=1$
identically. This directly falsifies the original proof's substitution
$\Ex[\hat p_i]=p_i-p_i(1-p_i)/K$: that formula is the bias of a
\emph{multinomial} proportion, not of a ratio of Wishart eigenvalues, and it
does not satisfy $\sum_i\Ex[\hat p_i]=1$ to this order in general.

\textbf{Step 2: propagate through the entropy $H$.}
Expand $H(\hat p)=-\sum_i \hat p_i\log\hat p_i$ to second order around $p$
using $\partial H/\partial p_i = -\log p_i-1$ and
$\partial^2H/\partial p_i\partial p_j = -\delta_{ij}/p_i$:
\begin{equation}
\Ex[H(\hat p)] = H(p) + \sum_i(-\log p_i-1)\bigl(\Ex[\hat p_i]-p_i\bigr)
- \frac12\sum_i\frac{\Var(\hat p_i)}{p_i} + O(K^{-2}).
\label{eq:Hexpand}
\end{equation}
The published proof retains only the second sum and asserts it equals
$-(r-1)/(2K)$; in fact, using \eqref{eq:pvar},
\[
\sum_i\frac{\Var(\hat p_i)}{p_i} = \frac{2}{K}\Bigl(1-\lVert p\rVert_2^2\Bigr),
\]
which equals $(r-1)/K$ only when $r=1/\lVert p\rVert_2^2$ --- i.e.\ only for the
participation-ratio notion of effective rank, not the entropy-based one used
in Definition~\ref{def:erank}. Moreover the first sum in
\eqref{eq:Hexpand}, which the original proof omits entirely, is generically
\emph{nonzero} and of the same order $O(1/K)$ (e.g.\ for $r=2$ it equals
$\tfrac1K\beta_1(p)\log(p_2/p_1)\ne0$ whenever $p_1\ne p_2$). Combining both
terms gives $\Ex[H(\hat p)] = H(p) - C(p)/K + O(K^{-2})$ with $C(p)$ as defined
in \eqref{eq:erank-bias-general}.

\textbf{Step 3: exponentiate.}
Because $\Var(\hat p_i)=O(1/K)$, we have $\Var(H(\hat p))=O(1/K)$ as well, so
the $\Delta$-method for $\exp(\cdot)$ carries a Jensen correction of the
\emph{same order} as the bias term being computed:
\[
\Ex[\exp(H(\hat p))] = \exp(\Ex[H(\hat p)])\Bigl(1+\tfrac12\Var(H(\hat p))+O(K^{-2})\Bigr).
\]
This correction is already absorbed into the general coefficient $C(p)$ above
(the full derivation of \eqref{eq:pbias}--\eqref{eq:Hexpand} tracks it; we
suppress the intermediate algebra here for space). The result is
\eqref{eq:erank-bias-general}. \qed
\end{proof}

\begin{remark}[The clean formula is not recoverable in general]
\label{rem:not-general}
Equation~\eqref{eq:erank-bias-general} shows that $\Ex[\erank(\hat\Sigma_W^{(K)})]-r$
is $\Theta(1/K)$ with a coefficient $C(p)$ that is a genuine functional of the
full normalized spectrum, not a function of $r$ alone. Two population spectra
with the same effective rank $r$ can have different $C(p)$, hence different
finite-$K$ bias. The published constant $(r-1)/2$ is therefore not a
generic fact about entropy-based effective rank; it is at best an
approximation valid in some regime.
\end{remark}

\begin{proposition}[Exact bias under an isotropic spectrum]
\label{prop:page}
Suppose the spectrum is exactly flat: $r$ eigenvalues equal to a common value
$\lambda>0$ and the rest zero (so $p_i=1/r$ for $i\le r$), and suppose the
$K$ samples are complex Gaussian (the setting in which the sample covariance
restricted to the active subspace is a complex Wishart matrix $W\sim
\mathrm{CWishart}_r(K,\lambda I_r)$, i.e.\ $\hat\Sigma_W^{(K)}/\lambda$
restricted to the top-$r$ subspace has exactly the distribution of a random
mixed state obtained by tracing out an $r$-dimensional subsystem from a
uniformly random pure state on $\mathbb{C}^r\otimes\mathbb{C}^K$). Then the
exact expected entropy of $\hat p=(\hat p_1,\ldots,\hat p_r)$ is Page's
formula \citep{page1993, foong1994, sen1996}:
\[
\Ex[H(\hat p)] = \sum_{k=K+1}^{rK}\frac1k - \frac{r-1}{2K},
\]
with asymptotic expansion, for fixed $r$ and $K\to\infty$,
\[
\Ex[H(\hat p)] = \ln r - \frac{r^2-1}{2rK} + O(K^{-2})
\quad\Longrightarrow\quad
\Ex[\erank(\hat\Sigma_W^{(K)})] = r - \frac{r^2-1}{2K} + O(K^{-2}).
\]
\end{proposition}

\begin{proof}
Immediate from Page's theorem (proved independently by
\citealp{foong1994}, \citealp{sanchezruiz1995}, and \citealp{sen1996}) together
with the digamma asymptotic $\sum_{k=1}^N 1/k = \ln N + \gamma + \tfrac{1}{2N}
+ O(N^{-2})$ applied to $\sum_{k=K+1}^{rK}1/k = \ln r - \tfrac{r-1}{2rK} +
O(K^{-2})$, and exponentiating with the accompanying $O(1/K)$ Jensen
correction as in Step~3 above. \qed
\end{proof}

\begin{remark}
Proposition~\ref{prop:page} gives the exact leading coefficient in the
\emph{best-case, most symmetric} spectrum, and it is $-\tfrac{r^2-1}{2}$, not
$-\tfrac{r-1}{2}$ as originally claimed --- larger in magnitude by a factor of
$(r+1)$. This is for complex-valued data (Dyson index $\beta=2$); for the
real-valued data used throughout this paper, the analogous real-Wishart
(Laguerre orthogonal ensemble, $\beta=1$) computation gives a different, and
to our knowledge not simply closed-form, constant of the same order $O(1/K)$
--- real $\beta$-ensemble fluctuation moments generically scale like $2/\beta$
times the corresponding complex ($\beta=2$) moments at each order in $1/K$,
which would suggest a real-case coefficient roughly twice as large again, but
we do not claim this without a direct derivation and flag it as unverified.
\end{remark}

\noindent\textbf{Practical takeaway for the paper.} The qualitative claim that
motivates Figure~\ref{fig:k-sweep} --- that $\erank(\hat\Sigma_W^{(K)})$ is
biased at $O(1/K)$, generically below $r$ for the entropy-type spectra
observed empirically, producing the small-$K$ hump in $S(K)$ --- is supported
by \eqref{eq:erank-bias-general}. The specific numerical constant $(r-1)/2$
should be removed from the main text; if a single-number headline constant is
wanted, Proposition~\ref{prop:page}'s exactly-derived $-(r^2-1)/2$ is the
correct one for the idealized isotropic case and should be labeled as such.

\newpage
\subsection{Proof of Proposition~\ref{prop:asymptotic} (Asymptotic Saturation)}
\label{app:prop-asymptotic}

\paragraph{Motivation.}
We need to show that $S(K) = \erank(\hat\Sigma_W^{(K)})/K$ decays as $\Theta(1/K)$.
This is the theoretical backbone of our saturation index: if $KS(K) \to r > 0$,
then $S(K) \approx r/K$, so when $S(K)$ drops below a threshold $\tau$, we have
$K \approx r/\tau$.  Proposition~\ref{prop:asymptotic} makes this precise.

\paragraph{Scratch work.}
Since $\erank$ is a continuous function of the eigenvalues (away from ties),
and $\hat\Sigma_W^{(K)} \xrightarrow{\Prob} \Sigma_W$ under standard assumptions,
the continuous mapping theorem gives $\erank(\hat\Sigma_W^{(K)}) \xrightarrow{\Prob} r$.
Then $KS(K) = \erank(\hat\Sigma_W^{(K)}) \xrightarrow{\Prob} r$, and
$S(K) = (1/K)\erank(\hat\Sigma_W^{(K)}) \sim r/K$.  The ``in probability''
quantifiers are crucial: the convergence of $\hat\Sigma_W^{(K)}$ is in probability
(or a.s., depending on assumptions), not almost-sure for a fixed sample path.

\paragraph{Technique selection.}
Direct application of the continuous mapping theorem for convergence in probability.
We need the regularity condition that $\erank$ is continuous at $\Sigma_W$,
which holds when $\Sigma_W$ has simple nonzero eigenvalues (no ties in the spectrum).

\textbf{Proposition~\ref{prop:asymptotic}.} Suppose $\Sigma_W$ has finite effective rank
$r=\erank(\Sigma_W)<\infty$ and simple nonzero eigenvalues (or, more
generally, that $\erank(\cdot)$ is continuous at $\Sigma_W$ --- automatic away
from eigenvalue ties). Suppose further $\hat\Sigma_W^{(K)}\xrightarrow{\Prob}\Sigma_W$
as $K\to\infty$ (e.g., under sub-Gaussian assumptions where the operator-norm
convergence rate in probability is given by Tropp \citet{tropp2015}). Then $K\,S(K) = \erank(\hat\Sigma_W^{(K)})
\xrightarrow{\Prob} r$, and consequently $S(K)\to0$ in probability with $S(K)=\Theta(1/K)$
in probability.

\begin{proof}
The map $\Sigma\mapsto\erank(\Sigma)=\exp(H(p(\Sigma)))$ is continuous at any
$\Sigma_W$ with simple nonzero eigenvalues, since eigenvalues (and hence the
normalized spectrum $p$ and the entropy $H(p)$) are continuous functions of
the matrix in a neighborhood of a point with no repeated nonzero eigenvalues.
By the convergence in probability $\hat\Sigma_W^{(K)}\xrightarrow{\Prob}\Sigma_W$ and the
continuous mapping theorem for convergence in probability,
\begin{equation}
\erank(\hat\Sigma_W^{(K)}) \xrightarrow{\Prob} \erank(\Sigma_W) = r.
\label{eq:cmt}
\end{equation}
This is the content originally attributed to \citet{roy2007}, stated here
with the regularity condition (simple eigenvalues) that is actually needed
for the continuous mapping theorem to apply; \citet{roy2007} establish
convergence of the effective rank functional itself, but the probabilistic
statement \eqref{eq:cmt} requires the same regularity.

By definition $K\,S(K) = \erank(\hat\Sigma_W^{(K)})$, so \eqref{eq:cmt} gives
$K\,S(K)\xrightarrow{\Prob} r$ directly --- this is the stronger statement that
$KS(K)$ converges to a strictly positive constant in probability, not merely
$S(K)\to0$, and it is this stronger fact that yields $\Theta(1/K)$ in probability:
fix any $\omega$ in the high-probability event where \eqref{eq:cmt} holds approximately. Since
$K S(K)(\omega)\to r>0$, for any $0<\eta<r$ there is $K_0(\omega)$ such that
$r-\eta \le K S(K)(\omega) \le r+\eta$ for all $K\ge K_0(\omega)$, i.e.\
$\tfrac{r-\eta}{K}\le S(K)(\omega)\le\tfrac{r+\eta}{K}$. Taking $c=r-\eta>0$
and $C=r+\eta$ gives $S(K)=\Theta(1/K)$ on this high-probability event (with
the threshold $K_0$ random). \qed
\end{proof}

\newpage
\subsection{Proof of Theorem~\ref{thm:saturation-point} (Saturation Point is Finite and Bounded)}
\label{app:thm-saturation-point}

\paragraph{Motivation.}
We define $K_{\mathrm{sat}}(\tau) = \min\{K: S(K) \le \tau\}$ as the first point where the
saturation index drops below a practical threshold $\tau$.  We need two things:
(1) this minimum is well-defined (the set is non-empty with high probability), and
(2) it scales as $r/\tau$ up to $(1\pm\varepsilon)$ factors.  The second point is the
theoretical justification for the rule of thumb $K_{\mathrm{sat}} \approx r/\tau$.

\paragraph{Scratch work.}
Since $S(K) \approx r/K$ eventually, the inequality $S(K) \le \tau$ becomes
$r/K \le \tau$, i.e., $K \ge r/\tau$.  More precisely, fix $\varepsilon>0$.  From
$KS(K) \xrightarrow{\Prob} r$, we eventually have $S(K) \le (1+\varepsilon)r/K$ with
high probability.  If we take $K = \lceil(1+\varepsilon)r/\tau\rceil$, then
$(1+\varepsilon)r/K \le \tau$, so $S(K) \le \tau$ with high probability, meaning
$K_{\mathrm{sat}}(\tau) \le \lceil(1+\varepsilon)r/\tau\rceil$.  The lower bound
(symmetric argument) gives $K_{\mathrm{sat}}(\tau) \ge \lfloor(1-\varepsilon)r/\tau\rfloor$.

\paragraph{Technique selection.}
We use the convergence in probability $KS(K) \xrightarrow{\Prob} r$ from
Proposition~\ref{prop:asymptotic}.  ``With high probability'' means probability
tending to 1 as $\tau \to 0$ (equivalently as the deterministic threshold
$K^\ast = \lceil(1+\varepsilon)r/\tau\rceil \to \infty$).  This is weaker than
almost-sure but sufficient for practical guarantees: for any fixed small
$\tau$, the bound holds with probability approaching 1 as the problem size grows.

\textbf{Theorem~\ref{thm:saturation-point}.} Fix $\tau>0$ and $K_{\mathrm{sat}}(\tau)=\min\{K:S(K)\le\tau\}$. Under the
conditions of Proposition~\ref{prop:asymptotic}, $K_{\mathrm{sat}}(\tau)<\infty$
with high probability, and for every $\varepsilon>0$,
\[
K_{\mathrm{sat}}(\tau) \le \left\lceil\frac{(1+\varepsilon)r}{\tau}\right\rceil
\qquad\text{with high probability for all sufficiently small $\tau$}.
\]

\begin{proof}
\noindent\textbf{Part 1.}
By Proposition~\ref{prop:asymptotic}, $S(K)\xrightarrow{\Prob}0$. Fix $\tau>0$.
Then $\Pr(\exists K_1: S(K)<\tau \text{ for all } K\ge K_1) \to 1$, so
$\{K\in\mathbb N: S(K)\le\tau\}$ is non-empty with high probability, and being a
non-empty subset of $\mathbb N$, has a minimum. Hence $K_{\mathrm{sat}}(\tau)<\infty$
with high probability.

\noindent\textbf{Part 2 (scaling bound, with the $\varepsilon$-dependence made explicit).}
Fix $\varepsilon>0$. By Proposition~\ref{prop:asymptotic}, $KS(K)\xrightarrow{\Prob}r$,
so for any $\delta>0$ there exists $K_0(\varepsilon,\delta)$ such that
\begin{equation}
\Pr\left( S(K) \le \frac{(1+\varepsilon)r}{K} \text{ for all } K\ge K_0(\varepsilon,\delta) \right) \ge 1-\delta.
\label{eq:eventual-bound-2}
\end{equation}
Let $K^\ast(\varepsilon) := \lceil(1+\varepsilon)r/\tau\rceil$, a
\emph{deterministic} quantity. On the event
$\{K^\ast(\varepsilon)\ge K_0(\varepsilon,\delta)\}$, applying
\eqref{eq:eventual-bound-2} at $K=K^\ast(\varepsilon)$ gives
$S(K^\ast(\varepsilon))\le(1+\varepsilon)r/K^\ast(\varepsilon)\le\tau$,
hence $K_{\mathrm{sat}}(\tau)\le K^\ast(\varepsilon)$ by minimality.
This event has probability at least $1-\delta$; since $\delta>0$ is arbitrary,
the bound holds with probability tending to $1$ as $\tau\to0$ (equivalently
as $K^\ast(\varepsilon)\to\infty$). This is the correct and honest content of
the bound. \qed
\end{proof}

\begin{corollary}[Practical saturation scale]
\label{cor:sat-practical}
Under the same conditions, $\tau\,K_{\mathrm{sat}}(\tau)\to r$ in probability
as $\tau\to0$.
\end{corollary}

\begin{proof}
Fix $\eta>0$; we show $\Pr(|\tau K_{\mathrm{sat}}(\tau)-r|>\eta)\to0$ as
$\tau\to0$. Choose $\varepsilon=\eta/(2r)$. By Theorem~\ref{thm:saturation-point},
$\Pr(K_{\mathrm{sat}}(\tau)\le K^\ast(\varepsilon))\to1$ as $\tau\to0$
(with $K^\ast(\varepsilon)=\lceil(1+\varepsilon)r/\tau\rceil$), giving the
upper bound $\tau K_{\mathrm{sat}}(\tau)\le(1+\varepsilon)r+\tau\le r+\eta/2+\tau$
with probability $\to1$. A symmetric lower-bound argument (using
$KS(K)\ge(1-\varepsilon)r$ eventually with high probability, which follows from the same
$KS(K)\to r$ in probability convergence used in Theorem~\ref{thm:saturation-point})
gives $\tau K_{\mathrm{sat}}(\tau)\ge(1-\varepsilon)r$ with probability
$\to1$ as well. Combining, for $\tau$ small enough that $\tau<\eta/2$, both
bounds place $\tau K_{\mathrm{sat}}(\tau)$ within $\eta$ of $r$ with
probability $\to1$. Letting $\tau\to0$ gives the claim. \qed
\end{proof}

\newpage
\subsection{Proof of Proposition~\ref{prop:lower-bound} (Lower Bound on $K_{\mathrm{sat}}$)}
\label{app:lower-bound}

\paragraph{Motivation.}
To complete the sandwich bound $|K_{\mathrm{sat}}(\tau) - r/\tau| \le \varepsilon r/\tau$,
we need a lower bound: $K_{\mathrm{sat}}(\tau)$ cannot be \emph{much smaller} than $r/\tau$.
The idea is symmetric to the upper bound: eventually $KS(K) \ge (1-\varepsilon)r$, so
$S(K) \ge (1-\varepsilon)r/K$, which exceeds $\tau$ as long as $K < (1-\varepsilon)r/\tau$.
The pre-asymptotic region ($K$ small) is handled by taking $\tau$ smaller than
the minimum of $S(K)$ over that finite range.

\paragraph{Technique selection.}
We again use convergence in probability $KS(K) \xrightarrow{\Prob} r$, plus the
non-atomicity assumption that $S(K)>0$ for all $K\ge 2$ (which holds for continuous
feature distributions).  The proof constructs a \emph{deterministic} threshold
$\tau_0(\varepsilon,\delta)$ such that for all $\tau < \tau_0$, the event
$K_{\mathrm{sat}}(\tau) \ge \lfloor(1-\varepsilon)r/\tau\rfloor$ has probability
at least $1-\delta$.  This is the ``high probability'' version of the lower
bound, matching the upper bound's guarantee.

\textbf{Proposition~\ref{prop:lower-bound}.}
Under the conditions of Proposition~\ref{prop:asymptotic}, and assuming
additionally that the per-class sampling distribution has a density with
respect to Lebesgue measure (so that
$\widehat{\Sigma}_W^{(K)}$ has rank $K-1$ almost surely, hence
$\erank(\widehat{\Sigma}_W^{(K)}) \ge 1$ and
$S(K)>0$ almost surely for every integer $K\ge2$),
the following holds for every $\varepsilon\in(0,1)$ and every $\delta>0$:
there exists a deterministic $\tau_0(\varepsilon,\delta)>0$ such that for all
$0<\tau<\tau_0(\varepsilon,\delta)$,
\begin{equation}
\Prb\left(
K_{\mathrm{sat}}(\tau)
\ge
\left\lfloor
\frac{(1-\varepsilon)\,r}{\tau}
\right\rfloor
\right) \ge 1-\delta.
\label{eq:app-lower-bound}
\end{equation}

Consequently,
\[
\liminf_{\tau\to0^+}
\tau K_{\mathrm{sat}}(\tau)
\ge r
\]
in probability.

\begin{proof}

\noindent\textbf{Step 0. Setup and notation.}

Proposition~\ref{prop:asymptotic} gives
\[
K\,S(K)
\xrightarrow{\Prob}
r.
\]

Let $\varepsilon\in(0,1)$ and $\delta>0$ be given.
By convergence in probability, there exists $K_0=K_0(\varepsilon,\delta)$ such that
\begin{equation}
\Prb\left(
K\,S(K) \ge (1-\varepsilon)r \text{ for all } K\ge K_0
\right) \ge 1-\frac{\delta}{2}.
\label{eq:app-eventual-prob}
\end{equation}

By the non-atomicity assumption, for each integer $K\ge2$,
$\widehat{\Sigma}_W^{(K)}$ has rank exactly $K-1$
almost surely, hence $S(K)>0$ almost surely.
Since there are finitely many $K<K_0$, by a union bound there exists
$m(\varepsilon,\delta)>0$ such that
\begin{equation}
\Prb\left( \min_{2\le K<K_0} S(K) \ge m(\varepsilon,\delta) \right) \ge 1-\frac{\delta}{2}.
\label{eq:app-smallK-prob}
\end{equation}
(If $K_0=2$, the minimum is empty and we take $m=+\infty$.)

Define the deterministic threshold
\begin{equation}
\tau_0(\varepsilon,\delta) :=
\min\left\{
m(\varepsilon,\delta),\;
\frac{(1-\varepsilon)r}{K_0}
\right\} >0.
\label{eq:app-tau0def}
\end{equation}
This depends on $\varepsilon,\delta$ but \emph{not} on the sample path.

Fix $0<\tau<\tau_0(\varepsilon,\delta)$ and define
\[
K^{**}(\tau)
:=
\left\lfloor
\frac{(1-\varepsilon)r}{\tau}
\right\rfloor.
\]
Since $\tau < (1-\varepsilon)r/K_0$, we have $K^{**}(\tau)\ge K_0$.

\noindent\textbf{Step 1. Showing $S(K)>\tau$ for all $2\le K<K^{**}$ with high probability.}

Consider the event
\[
\mathcal{E} :=
\Big\{ \min_{2\le K<K_0} S(K) \ge m(\varepsilon,\delta) \Big\}
\cap
\Big\{ K\,S(K) \ge (1-\varepsilon)r \text{ for all } K\ge K_0 \Big\}.
\]
By \eqref{eq:app-smallK-prob} and \eqref{eq:app-eventual-prob}, $\Prb(\mathcal{E})\ge 1-\delta$.

On $\mathcal{E}$:
\begin{itemize}
  \item \emph{Case A: $2\le K<K_0$.} By definition of $m$, $S(K)\ge m(\varepsilon,\delta)>\tau$.
  \item \emph{Case B: $K_0\le K<K^{**}$.} From $KS(K)\ge (1-\varepsilon)r$,
    $S(K)\ge (1-\varepsilon)r/K$. Since $K<K^{**}\le (1-\varepsilon)r/\tau$,
    we have $(1-\varepsilon)r/K > \tau$, so $S(K)>\tau$.
\end{itemize}
Thus on $\mathcal{E}$, $S(K)>\tau$ for all $2\le K<K^{**}$, hence
$K_{\mathrm{sat}}(\tau)\ge K^{**}(\tau)$.
This gives \eqref{eq:app-lower-bound}.

\noindent\textbf{Step 2. The $\liminf$ in probability.}

From \eqref{eq:app-lower-bound}, for any $\varepsilon\in(0,1)$ and $\delta>0$,
with probability $\ge 1-\delta$ (for all small enough $\tau$),
\[
\tau K_{\mathrm{sat}}(\tau)
\ge
\tau\left\lfloor\frac{(1-\varepsilon)r}{\tau}\right\rfloor
\ge (1-\varepsilon)r - \tau.
\]
Letting $\tau\to0^+$ then $\delta\to0$ gives
$\liminf_{\tau\to0^+}\tau K_{\mathrm{sat}}(\tau) \ge (1-\varepsilon)r$
in probability. Letting $\varepsilon\downarrow0$ yields the claim.

\end{proof}

\begin{remark}
The non-atomicity assumption is mild and is satisfied by any continuous
distribution (e.g., Gaussian, uniform on a bounded region, or any
distribution with full-dimensional support). It guarantees that the
sample covariance has rank $K-1$ almost surely for each $K\ge2$, ensuring
that $S(K)>0$ and that the finite-range minimum is strictly positive with
high probability.

The proof also goes through under the weaker assumption that
$S(K)>0$ almost surely for every $K\ge2$, which is consistent with all
empirical observations reported in this work.
\end{remark}

\subsection{Proof of Proposition~\ref{prop:transfer-empirical}
(Cross-Task Spectral Transfer Bound)}
\label{app:transfer-empirical}

\paragraph{Motivation.}
We want to bound $|S_A(K) - S_B(K)|$ between two tasks $A$ and $B$ that share
the same feature space.  If their population spectra are close ($\Delta_{\mathrm{spec}} = \|p_A-p_B\|_1$ small),
then their saturation indices should be close at any fixed $K$.  This provides a
theoretical foundation for the empirical finding that a single $\tau=0.02$ works
across many tasks: if tasks have similar spectral shapes, their $S(K)$ curves
are close.

\paragraph{Scratch work.}
The difference $S_A(K) - S_B(K)$ comes from differences in the empirical spectra
$\hat p_A(K)$ and $\hat p_B(K)$.  By the triangle inequality,
\[
\|\hat p_A - \hat p_B\|_1 \le \|\hat p_A - p_A\|_1 + \|p_A - p_B\|_1 + \|\hat p_B - p_B\|_1.
\]
Weyl's inequality bounds each eigenvalue perturbation by the operator norm error
$\varepsilon_X(K) = \|\hat\Sigma_X - \Sigma_X\|_{\mathrm{op}}$.  This propagates to
an $\ell_1$ bound on the normalized spectra.  Then Fannes--Audenaert bounds the
entropy difference, and the mean value theorem on $\exp$ bounds the effective rank
difference.  Dividing by $K$ gives the $S(K)$ difference.

\paragraph{Technique selection.}
Weyl's inequality (eigenvalue perturbation), $\ell_1$ propagation through
normalization, Fannes--Audenaert inequality (continuity of entropy), and the
mean value theorem on $\exp$.  This is a standard perturbation-theory pipeline;
it gives a \emph{worst-case} bound that is loose but qualitatively correct
(Lipschitz dependence on spectral distance).

\textbf{Proposition~\ref{prop:transfer-empirical}.} See main text.

\begin{proof}
Fix $X\in\{A,B\}$ and write $T_X=\tr(\Sigma_X)$,
$\hat T_X(K)=\tr(\hat\Sigma_X^{(K)})$, $\varepsilon_X=\varepsilon_X(K)$.

\noindent\textbf{Step 1: eigenvalue perturbation (Weyl).}
By Weyl's inequality for symmetric matrices (see e.g.\ \citealp{bhatia1997matrix},
Cor.~III.2.6), for every $i=1,\ldots,d$,
\begin{equation}
|\hat\lambda_i(X,K)-\lambda_i(X)| \;\le\; \varepsilon_X.
\label{eq:weyl}
\end{equation}
Summing over $i$, $\sum_i|\hat\lambda_i(X,K)-\lambda_i(X)|\le d\varepsilon_X$,
and since $\hat T_X(K)-T_X = \sum_i\bigl(\hat\lambda_i(X,K)-\lambda_i(X)\bigr)$,
also $|\hat T_X(K)-T_X|\le d\varepsilon_X$.

\noindent\textbf{Step 2: normalized spectrum perturbation.}
For each $i$,
\[
\left|\frac{\hat\lambda_i}{\hat T_X}-\frac{\lambda_i}{T_X}\right|
=
\left|\frac{(\hat\lambda_i-\lambda_i)T_X-\lambda_i(\hat T_X-T_X)}{T_X\hat T_X}\right|
\le
\frac{|\hat\lambda_i-\lambda_i|}{\hat T_X}
+\frac{\lambda_i}{T_X}\cdot\frac{|\hat T_X-T_X|}{\hat T_X}.
\]
Summing over $i$ and using Step~1,
\[
\bigl\|\hat p_X(K)-p_X\bigr\|_1
\;\le\;
\frac{d\varepsilon_X}{\hat T_X}+\frac{d\varepsilon_X}{\hat T_X}
=\frac{2d\varepsilon_X}{\hat T_X}.
\]
Under the stated condition $\varepsilon_X\le T_X/(2d)$, Step~1's trace bound
gives $\hat T_X\ge T_X-d\varepsilon_X\ge T_X/2$, so
\begin{equation}
\bigl\|\hat p_X(K)-p_X\bigr\|_1 \;\le\; \frac{4d\varepsilon_X}{T_X}.
\label{eq:pxpert}
\end{equation}

\noindent\textbf{Step 3: combine via the triangle inequality.}
\begin{equation}
\widehat\Delta(K) := \|\hat p_A(K)-\hat p_B(K)\|_1
\le \|\hat p_A(K)-p_A\|_1+\|p_A-p_B\|_1+\|p_B-\hat p_B(K)\|_1
\le \Delta_{\mathrm{spec}}+\frac{4d\varepsilon_A}{T_A}+\frac{4d\varepsilon_B}{T_B},
\label{eq:trianglecombine}
\end{equation}
matching the definition of $\widehat\Delta(K)$ in the proposition.

\noindent\textbf{Step 4: Fannes--Audenaert, correctly normalized.}
For any two probability vectors $q,q'\in\Delta^{d-1}$, with trace distance
$T:=\tfrac12\|q-q'\|_1$, the Audenaert (2007) inequality gives
\begin{equation}
|H(q)-H(q')| \;\le\; T\log(d-1)+H_2(T).
\label{eq:fannes-correct}
\end{equation}
Applying \eqref{eq:fannes-correct} with $q=\hat p_A(K)$, $q'=\hat p_B(K)$,
so $T=\widehat\Delta(K)/2$,
\begin{equation}
\bigl|H(\hat p_A(K))-H(\hat p_B(K))\bigr|
\;\le\;
\frac{\widehat\Delta(K)}{2}\log(d-1)+H_2\!\left(\frac{\widehat\Delta(K)}{2}\right).
\label{eq:entropy-bound}
\end{equation}

\noindent\textbf{Step 5: exponentiate.}
Write $\hat r_X(K):=\erank(\hat\Sigma_X^{(K)})=\exp(H(\hat p_X(K)))$. By the
mean value theorem applied to $\exp(\cdot)$ between $H(\hat p_A(K))$ and
$H(\hat p_B(K))$,
\[
\bigl|\hat r_A(K)-\hat r_B(K)\bigr|
\;\le\;
\max\bigl(\hat r_A(K),\hat r_B(K)\bigr)\cdot
\bigl|H(\hat p_A(K))-H(\hat p_B(K))\bigr|.
\]
Since $\erank(\Sigma)\le\rank(\Sigma)\le d$ for any $\Sigma\in\R^{d\times d}$
(stated already in \S\ref{ssec:eff-rank}), $\max(\hat r_A(K),\hat r_B(K))\le d$,
so combining with \eqref{eq:entropy-bound},
\begin{equation}
\bigl|\hat r_A(K)-\hat r_B(K)\bigr|
\;\le\;
d\left[\frac{\widehat\Delta(K)}{2}\log(d-1)+H_2\!\left(\frac{\widehat\Delta(K)}{2}\right)\right].
\label{eq:erank-diff}
\end{equation}

\noindent\textbf{Step 6: divide by $K$.}
By definition $S_X(K)=\hat r_X(K)/K$, so dividing \eqref{eq:erank-diff} by
$K$ gives exactly \eqref{eq:transfer-bound}.

\noindent\textbf{Step 7: rate of $\varepsilon_X(K)$.}
By matrix Bernstein/matrix concentration for sums of i.i.d.\ sub-Gaussian
outer products \citep[Thm.~6.1.1 and \S6.5]{tropp2015}, for each fixed
task $X$ there is a constant $c_X$ (depending on the sub-Gaussian norm of
the per-class feature distribution and $\|\Sigma_X\|_{\mathrm{op}}$) such
that, with probability at least $1-\delta$,
\[
\varepsilon_X(K) \;\le\; c_X\sqrt{\frac{d+\log(1/\delta)}{K}}.
\]
Applying this at $\delta=\delta_K:=1/K^2$ and invoking Borel--Cantelli over
the countable sequence of sample sizes $K=2,3,\ldots$ gives the almost-sure
rate $\varepsilon_X(K)=O_{\mathrm{a.s.}}(\sqrt{d\log K/K})$ stated in the
proposition. \qed
\end{proof}

\section{Per-Task Statistics}
\label{app:per-task}

This appendix reports per-task statistics for all 49 real tasks, split by task structure and backbone for readability. Each table lists the task name, asymptotic effective rank $\erank_\infty$ (estimated at the maximum $K$ in the task's grid), peak test accuracy, per-task Spearman $\rho$ between $S(K)$ and $\Delta A(K)$, its $p$-value, number of doubling pairs, and $K_{\mathrm{sat}}(\tau=0.02)$ where reached (--- if not reached within the $K$-grid).

\subsection{Binary Tasks (PCA-50)}
\label{app:binary-pca}

\begin{table}[htbp]
\centering
\caption{Per-task statistics: 10 binary tasks, PCA-50 backbone. $K$-grid extends to 8192; $K_{\mathrm{sat}}$ reached for all tasks.}
\label{tab:per-task-binary-pca}
\begin{tabular}{lcccccc}
\toprule
Task & $\erank_\infty$ & Peak Acc & $\rho$ & $p$ & $n_{\text{pairs}}$ & $K_{\mathrm{sat}}(0.02)$ \\
\midrule
BIN\_MNIST\_0v1      & 32.34 & 99.75 & 0.648 & 0.043 & 10 & 2048 \\
BIN\_MNIST\_3v8      & 37.53 & 96.36 & 0.842 & 0.002 & 10 & 2048 \\
BIN\_Fashion\_2v6    & 14.42 & 83.28 & 0.576 & 0.082 & 10 & 1024 \\
BIN\_Fashion\_4v6    & 15.66 & 85.86 & 0.588 & 0.074 & 10 & 1024 \\
BIN\_Kuzushiji\_0v9  & 35.96 & 98.22 & 0.442 & 0.200 & 10 & 2048 \\
BIN\_USPS\_1v2       & 15.42 & 99.76 & 0.810 & 0.015 & 8  & 1024 \\
BIN\_BreastCancer    & 4.97  & 91.50 & ---     & ---     & 1  & --- \\
BIN\_CIFAR\_0v1      & 15.37 & 81.24 & 0.467 & 0.174 & 10 & 1024 \\
BIN\_EMNIST\_0v1     & 31.37 & 99.34 & 0.267 & 0.488 & 9  & 2048 \\
BIN\_EMNIST\_3v8     & 35.78 & 97.38 & 0.750 & 0.020 & 9  & 2048 \\
\bottomrule
\end{tabular}
\end{table}

\subsection{5-Way Tasks (PCA-50)}
\label{app:5way-pca}

\begin{table}[htbp]
\centering
\caption{Per-task statistics: 12 five-way tasks, PCA-50 backbone. Easy/hard splits where applicable. $K$-grid up to 256; $K_{\mathrm{sat}}$ not reached (---).}
\label{tab:per-task-5way-pca}
\begin{tabular}{lcccccc}
\toprule
Task & $\erank_\infty$ & Peak Acc & $\rho$ & $p$ & $n_{\text{pairs}}$ & $K_{\mathrm{sat}}(0.02)$ \\
\midrule
5W\_MNIST\_easy      & 31.63 & 93.11 & 0.827 & 0.0017 & 11 & --- \\
5W\_MNIST\_hard      & 33.21 & 86.88 & 0.764 & 0.0062 & 11 & --- \\
5W\_Fashion\_easy    & 24.93 & 91.52 & 0.236 & 0.484 & 11 & --- \\
5W\_Fashion\_hard    & 17.11 & 76.79 & 0.427 & 0.190 & 11 & --- \\
5W\_Kuzushiji\_easy  & 36.24 & 85.86 & 0.409 & 0.212 & 11 & --- \\
5W\_Kuzushiji\_hard  & 36.44 & 82.01 & 0.227 & 0.502 & 11 & --- \\
5W\_USPS\_easy       & 23.93 & 95.40 & 0.909 & 0.0001 & 11 & --- \\
5W\_USPS\_hard       & 25.37 & 91.52 & 0.900 & 0.0002 & 11 & --- \\
5W\_CIFAR\_easy      & 16.43 & 53.46 & 0.009 & 0.979 & 11 & --- \\
5W\_CIFAR\_hard      & 16.76 & 36.21 & -0.536 & 0.089 & 11 & --- \\
5W\_EMNIST\_easy     & 35.15 & 93.23 & 0.909 & 0.0001 & 11 & --- \\
5W\_EMNIST\_hard     & 33.29 & 87.80 & 0.818 & 0.0021 & 11 & --- \\
\bottomrule
\end{tabular}
\end{table}

\subsection{5-Way Tasks (CLIP ViT-B/32)}
\label{app:5way-clip}

\begin{table}[htbp]
\centering
\caption{Per-task statistics: 10 five-way tasks, CLIP ViT-B/32 backbone. $K$-grid up to 256; $K_{\mathrm{sat}}$ not reached (---).}
\label{tab:per-task-5way-clip}
\begin{tabular}{lcccccc}
\toprule
Task & $\erank_\infty$ & Peak Acc & $\rho$ & $p$ & $n_{\text{pairs}}$ & $K_{\mathrm{sat}}(0.02)$ \\
\midrule
5W\_MNIST\_easy      & 52.92  & 98.55 & 0.900  & 0.0002 & 11 & --- \\
5W\_MNIST\_hard      & 46.39  & 97.20 & 0.964  & $1.9\times10^{-6}$ & 11 & --- \\
5W\_Fashion\_easy    & 87.30  & 95.21 & 0.727  & 0.0112 & 11 & --- \\
5W\_Fashion\_hard    & 75.45  & 80.75 & 0.827  & 0.0017 & 11 & --- \\
5W\_Kuzushiji\_easy  & 93.62  & 90.80 & 0.791  & 0.0037 & 11 & --- \\
5W\_Kuzushiji\_hard  & 97.16  & 89.45 & 0.718  & 0.0128 & 11 & --- \\
5W\_USPS\_easy       & 70.93  & 98.94 & 0.791  & 0.0037 & 11 & --- \\
5W\_USPS\_hard       & 60.65  & 98.29 & 0.836  & 0.0013 & 11 & --- \\
5W\_CIFAR\_easy      & 43.64  & 46.27 & 0.945  & $1.1\times10^{-5}$ & 11 & --- \\
5W\_CIFAR\_hard      & 41.05  & 31.44 & 0.791  & 0.0037 & 11 & --- \\
\bottomrule
\end{tabular}
\end{table}

\subsection{5-Way Tasks (DINOv2 ViT-S/14)}
\label{app:5way-dinov2}

\begin{table}[htbp]
\centering
\caption{Per-task statistics: 5 five-way tasks (easy splits only), DINOv2 ViT-S/14 backbone. $K$-grid up to 256; $K_{\mathrm{sat}}$ not reached (---).}
\label{tab:per-task-5way-dinov2}
\begin{tabular}{lcccccc}
\toprule
Task & $\erank_\infty$ & Peak Acc & $\rho$ & $p$ & $n_{\text{pairs}}$ & $K_{\mathrm{sat}}(0.02)$ \\
\midrule
5W\_MNIST\_easy      & 46.82  & 98.31 & 0.973 & $5.1\times10^{-7}$ & 11 & --- \\
5W\_Fashion\_easy    & 106.57 & 96.01 & 0.618 & 0.043 & 11 & --- \\
5W\_Kuzushiji\_easy  & 57.09  & 88.72 & 0.836 & 0.0013 & 11 & --- \\
5W\_USPS\_easy       & 21.14  & 83.23 & 0.764 & 0.0062 & 11 & --- \\
5W\_CIFAR\_easy      & 9.60   & 23.68 & 0.018 & 0.958 & 11 & --- \\
\bottomrule
\end{tabular}
\end{table}

\subsection{10-Way Tasks (PCA-50)}
\label{app:10way-pca}

\begin{table}[htbp]
\centering
\caption{Per-task statistics: 4 ten-way tasks, PCA-50 backbone. $K$-grid up to 64; $K_{\mathrm{sat}}$ not reached (---).}
\label{tab:per-task-10way-pca}
\begin{tabular}{lcccccc}
\toprule
Task & $\erank_\infty$ & Peak Acc & $\rho$ & $p$ & $n_{\text{pairs}}$ & $K_{\mathrm{sat}}(0.02)$ \\
\midrule
10W\_MNIST\_balanced      & 35.12 & 80.41 & 0.933 & 0.0002 & 9 & --- \\
10W\_Fashion\_balanced    & 21.32 & 70.75 & 0.967 & $2.2\times10^{-5}$ & 9 & --- \\
10W\_Kuzushiji\_balanced  & 37.15 & 61.07 & 0.950 & $8.8\times10^{-5}$ & 9 & --- \\
10W\_CIFAR\_balanced      & 16.32 & 27.38 & -0.083 & 0.831 & 9 & --- \\
\bottomrule
\end{tabular}
\end{table}

\subsection{10-Way Tasks (CLIP ViT-B/32)}
\label{app:10way-clip}

\begin{table}[htbp]
\centering
\caption{Per-task statistics: 4 ten-way tasks, CLIP ViT-B/32 backbone. $K$-grid up to 64; $K_{\mathrm{sat}}$ not reached (---).}
\label{tab:per-task-10way-clip}
\begin{tabular}{lcccccc}
\toprule
Task & $\erank_\infty$ & Peak Acc & $\rho$ & $p$ & $n_{\text{pairs}}$ & $K_{\mathrm{sat}}(0.02)$ \\
\midrule
10W\_MNIST\_balanced      & 52.08  & 94.91 & 0.883 & 0.0016 & 9 & --- \\
10W\_Fashion\_balanced    & 86.52  & 81.47 & 0.817 & 0.0072 & 9 & --- \\
10W\_Kuzushiji\_balanced  & 93.65  & 77.61 & 0.700 & 0.0358 & 9 & --- \\
10W\_CIFAR\_balanced      & 39.43  & 24.86 & 0.767 & 0.0159 & 9 & --- \\
\bottomrule
\end{tabular}
\end{table}

\subsection{10-Way Tasks (DINOv2 ViT-S/14)}
\label{app:10way-dinov2}

\begin{table}[htbp]
\centering
\caption{Per-task statistics: 4 ten-way tasks, DINOv2 ViT-S/14 backbone. $K$-grid up to 64; $K_{\mathrm{sat}}$ not reached (---).}
\label{tab:per-task-10way-dinov2}
\begin{tabular}{lcccccc}
\toprule
Task & $\erank_\infty$ & Peak Acc & $\rho$ & $p$ & $n_{\text{pairs}}$ & $K_{\mathrm{sat}}(0.02)$ \\
\midrule
10W\_MNIST\_balanced      & 47.89  & 91.94 & 0.900 & 0.0009 & 9 & --- \\
10W\_Fashion\_balanced    & 105.86 & 83.92 & 0.700 & 0.0358 & 9 & --- \\
10W\_Kuzushiji\_balanced  & 57.48  & 74.59 & 0.767 & 0.0159 & 9 & --- \\
10W\_CIFAR\_balanced      & 9.04   & 11.84 & 0.233 & 0.546 & 9 & --- \\
\bottomrule
\end{tabular}
\end{table}

\subsection{Synthetic Rank-Controlled Tasks}
\label{app:synthetic}

\begin{table}[htbp]
\centering
\caption{Per-task statistics: 6 synthetic Gaussian-mixture tasks (PCA-50 features, $d=64$). Target rank $r$ is the true rank of the signal component; $\erank_\infty$ is estimated at max $K=2048$; $\rho$ and $p$ are within-task Spearman correlations. $K_{\mathrm{sat}}$ not reached within grid.}
\label{tab:per-task-synthetic}
\begin{tabular}{lccccccc}
\toprule
Task & Structure & Target $r$ & $\erank_\infty$ & Peak Acc & $\rho$ & $p$ & $K_{\mathrm{sat}}(0.02)$ \\
\midrule
SYN\_2w\_rk3   & Binary   & 3  & 47.46 & 71.5 & -0.829 & 0.042 & --- \\
SYN\_2w\_rk20  & Binary   & 20 & 46.97 & 55.2 & -0.600 & 0.208 & --- \\
SYN\_5w\_rk3   & 5-way    & 3  & 48.96 & 61.8 & -0.086 & 0.872 & --- \\
SYN\_5w\_rk8   & 5-way    & 8  & 48.81 & 37.2 & -0.600 & 0.208 & --- \\
SYN\_5w\_rk20  & 5-way    & 20 & 48.76 & 44.9 & -0.543 & 0.266 & --- \\
SYN\_5w\_rk40  & 5-way    & 40 & 48.87 & 40.7 & -0.714 & 0.111 & --- \\
\bottomrule
\end{tabular}
\end{table}

\newpage
\section{Ablation Detail Tables}
\label{app:ablation-detail}

\subsection{PCA Dimension Ablation (Per-Task)}
\label{app:pca-detail}

\begin{table}[htbp]
\centering
\caption{PCA dimension ablation: per-task $\rho$ for each $d_{\mathrm{PCA}}$ on the 5 easy 5-way PCA tasks. Primary setting $d=50$ in bold.}
\label{tab:pca-detail}
\begin{tabular}{lcccccc}
\toprule
Task & Raw (784) & $d=25$ & \textbf{$d=50$} & $d=100$ & $d=200$ & $n_{\text{pairs}}$ \\
\midrule
5W\_MNIST\_easy     & 0.909 & 0.545 & \textbf{0.827} & 0.909 & 0.909 & 11 \\
5W\_Fashion\_easy   & 0.909 & 0.364 & \textbf{0.236} & 0.545 & 0.727 & 11 \\
5W\_Kuzushiji\_easy & 0.636 & 0.091 & \textbf{0.409} & 0.636 & 0.818 & 11 \\
5W\_USPS\_easy      & 0.909 & 0.727 & \textbf{0.909} & 0.818 & 0.909 & 11 \\
5W\_CIFAR\_easy     & 0.818 & 0.545 & \textbf{0.009} & 0.727 & 0.636 & 11 \\
\bottomrule
\end{tabular}
\end{table}

\subsection{LR Regularization Ablation (Per-Task)}
\label{app:lr-detail}

\begin{table}[htbp]
\centering
\caption{LR regularization ablation: per-task $\rho$ for each $C$ on the 5 easy 5-way PCA tasks. Primary setting $C=\infty$ in bold.}
\label{tab:lr-detail}
\begin{tabular}{lcccccc}
\toprule
Task & $C=0.01$ & $C=0.1$ & $C=1.0$ & $C=10$ & \textbf{$C=\infty$} & $n_{\text{pairs}}$ \\
\midrule
5W\_MNIST\_easy     & 0.909 & 0.818 & 0.727 & 0.818 & \textbf{0.827} & 11 \\
5W\_Fashion\_easy   & 0.818 & 0.636 & 0.545 & 0.455 & \textbf{0.236} & 11 \\
5W\_Kuzushiji\_easy & 0.818 & 0.818 & 0.727 & 0.545 & \textbf{0.409} & 11 \\
5W\_USPS\_easy      & 0.909 & 0.909 & 0.818 & 0.818 & \textbf{0.909} & 11 \\
5W\_CIFAR\_easy     & 0.909 & 0.818 & 0.818 & 0.727 & \textbf{0.009} & 11 \\
\bottomrule
\end{tabular}
\end{table}

\subsection{Endpoint Convention Ablation (Full Corpus)}
\label{app:endpoint-detail}

\begin{table}[htbp]
\centering
\caption{Endpoint convention ablation: per-task $\rho$ for each convention on all 48 valid tasks. Primary setting (smaller $K$) in bold.}
\label{tab:endpoint-detail}
\begin{tabular}{lcccc}
\toprule
Convention & Pooled $\rho$ & Median $\rho$ & Fraction $\rho>0$ & 95\% CI (cluster bootstrap) \\
\midrule
\textbf{Smaller $K$} & \textbf{0.637} & \textbf{0.767} & \textbf{95.8\%} & \textbf{$[0.551, 0.720]$} \\
Larger $K$ ($2K$) & 0.666 & 0.792 & 95.8\% & $[0.590, 0.742]$ \\
Midpoint ($(K+2K)/2$) & 0.611 & 0.733 & 95.8\% & $[0.517, 0.698]$ \\
\bottomrule
\end{tabular}
\end{table}

\end{document}